\documentclass{article}

\usepackage[preprint]{neurips_2026}

\usepackage[utf8]{inputenc}
\usepackage[T1]{fontenc}
\usepackage{hyperref}
\usepackage{url}
\usepackage{booktabs}
\usepackage{amsfonts}
\usepackage{nicefrac}
\usepackage{microtype}
\usepackage{xcolor}

\usepackage{amsmath,amssymb,amsthm}
\usepackage{bm}
\usepackage{algorithm}
\usepackage{algorithmic}
\usepackage{graphicx}
\usepackage{comment}
\usepackage[capitalize,noabbrev]{cleveref}
\usepackage{mathtools}
\usepackage{bbm}
\usepackage{natbib}

\theoremstyle{plain}
\newtheorem{theorem}{Theorem}[section]
\newtheorem{proposition}[theorem]{Proposition}
\newtheorem{lemma}[theorem]{Lemma}
\newtheorem{corollary}[theorem]{Corollary}
\theoremstyle{definition}
\newtheorem{definition}[theorem]{Definition}
\theoremstyle{remark}

\newtheorem{assumption}[theorem]{Assumption}

\DeclareMathOperator*{\Var}{Var}

\title{Denoised Conformal Alignment for Reliable Selection of Conditional Average Treatment Effect Predictions}

\author{%
  Xinyun Lu\thanks{Work done during undergrad at SUFE. Currently starting MS at SJTU.} \\
  Shanghai Jiao Tong University \\
  Shanghai University of Finance and Economics \\
  \texttt{luxy@stu.sufe.edu.cn} \\
  \And
  Haoang Chi \\
  National University of Defense Technology \\
  Changsha, China \\
  \And
  Zhiheng Zhang\thanks{Correspondence to: Zhiheng Zhang \texttt{<zhangzhiheng@mail.shufe.edu.cn>}.} \\
  School of Statistics and Data Science \\
  Shanghai University of Finance and Economics \\
  Shanghai 200433, P.R. China \\
  Institute of Big Data Research \\
  Shanghai University of Finance and Economics \\
  Shanghai 200433, P.R. China \\
  \texttt{zhangzhiheng@mail.shufe.edu.cn}
}

\begin{document}

\maketitle

\begin{abstract}
In selective deployment, practitioners act only on a model-chosen subset of individuals based on predicted conditional average treatment effects, but marginal conformal guarantees need not control reliability on that selected subset. We study reliable selection for black-box CATE predictors: selecting candidates whose CATE errors are below a tolerance while controlling the false discovery rate (FDR). Since CATE errors are unobservable, we construct doubly robust proxy errors from pseudo-outcomes; however, naive proxies can lose power under heteroskedasticity because variance overwhelms the reliability signal. We propose Denoised Conformal Alignment, which subtracts an estimated conditional variance component and combines conformal calibration with Benjamini--Hochberg selection. Our analysis shows that validity is governed by stability of proxy/oracle threshold labels, rather than pointwise perfection of the variance estimator. Experiments show substantially improved power while maintaining FDR control across challenging settings.

\end{abstract}

\section{Introduction}
\label{sec:intro}

A defining feature of modern data-driven decision-making is \emph{selective deployment}:
models are not applied uniformly to all individuals, but are instead used to trigger actions only for a small, model-chosen subset.
Hospitals may escalate care for patients with the largest predicted risks or gains; policymakers may target interventions to communities with the greatest projected impact \citep{athey2017state}; and online platforms may offer incentives only to users predicted to respond strongly, while flagging those predicted to react adversely.
In such settings, decisions hinge on \emph{extreme individualized effect estimates}, not on population averages.
As a result, errors are no longer evenly distributed: they are concentrated precisely on the individuals for whom actions are taken.
This makes selective deployment simultaneously powerful and perilous.
A central question therefore arises:
\emph{when can we trust a black-box conditional average treatment effect (CATE) prediction enough to act on it?}

\begin{figure}[t]
\centering
\includegraphics[width=\columnwidth]{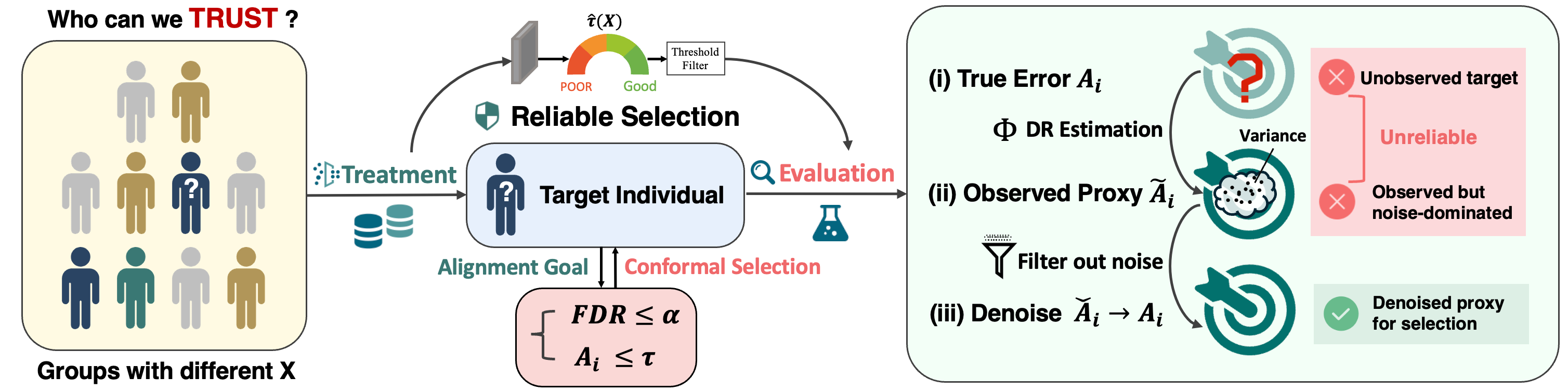}
\caption{\textbf{Conceptual illustration.}
\textbf{(Left) \emph{Reliable selection.}}
Given covariates $X$, our goal is to identify individuals or groups whose CATE predictions are reliable, rather than relying on average effects or marginal guarantees.
\textbf{(Right) \emph{Denoised error proxies.}}
Since the true CATE prediction error $A_i$ is counterfactual and unobservable, we construct a noisy observable proxy $\widetilde A_i$ from doubly robust pseudo-outcomes, and denoise it to obtain $\check A_i$ by removing variance-induced noise.
These scores support selective inference with controlled risk via conformal calibration.}
\label{fig:teaser}
\end{figure}

Individualized treatment effect estimation lies at the core of modern causal inference and personalized decision-making.
Prior work has developed increasingly accurate CATE predictors through meta-learners \citep{kunzel2019metalearners, kennedy2020towards}, representation learning \citep{shalit2017estimating}, and forest-based methods \citep{wager2018estimation, nie2021quasi}, typically evaluated by population-level metrics such as PEHE.
However, \emph{high average accuracy is not a deployment guarantee}. Selective deployment requires reliability \emph{conditional on the model’s own ranking}, precisely where predictions are most extreme.
Conformal prediction (CP) offers distribution-free uncertainty quantification with marginal coverage guarantees \citep{vovk2005algorithmic, angelopoulos2021gentle}, and recent work has extended CP to causal inference to produce valid counterfactual and ITE intervals under standard identification assumptions \citep{lei2021conformal, chernozhukov2021exact}.
Yet these guarantees are fundamentally \emph{marginal}:
they hold on average over the population, not on the post-selection subset where decisions are taken.

This mismatch is consequential.
Once actions are triggered only for top-ranked individuals, the relevant target distribution is no longer the population, but the \emph{post-selection subpopulation} induced by the model itself.
As a result, uncertainty statements that are valid marginally can become systematically miscalibrated on the very cases that matter most \citep{geifman2017selective}.
Such post-selection failures are not unique to causal inference:
\citet{jin2023selection} show that selection alone can turn nominal guarantees into uncontrolled errors even in purely predictive settings.
These observations point to a fundamental limitation of coverage-based guarantees for deployment.

Motivated by this gap, we argue that selective deployment calls for a different primitive:
not a guarantee about \emph{coverage}, but about \emph{which cases are safe to trust}.
We formalize this deployment-risk objective as \emph{reliable selection}:
from a pool of candidates, select a subset whose CATE predictions are accurate up to a user-specified tolerance, while controlling the expected fraction of inaccurate selections.
This criterion directly aligns with deployment risk and naturally connects to recent conformal calibration ideas for selective use, including Conformal Alignment \citep{gui2024conformal}.
Operationally, we adopt false discovery rate (FDR) control \citep{benjamini1995controlling} as the notion of safety:
\emph{among the individuals we choose to act on, the expected proportion whose CATE predictions violate the tolerance is at most $\alpha$.}

Turning reliable selection into a valid procedure is deceptively difficult in causal settings.
Because one potential outcome is missing per unit, unit-level ITEs are not point-identifiable from observational data \citep{rubin1974estimating, pearl2009causality}.
We therefore target the identifiable CATE prediction error $A_i = |\hat\tau(X_i)-\tau(X_i)|$, which remains unavailable for calibration because $\tau(X_i)$ is unknown, precluding conformalized selection methods that require observed error labels \citep{candes2023conformalized, gui2024conformal, jin2023selection}.
A natural workaround is to use doubly robust (DR) pseudo-outcomes as observable surrogates \citep{robins1994estimation, bang2005doubly}.
However, this strategy encounters a deeper, structural failure:
in heterogeneous and heteroskedastic settings, DR-based proxy errors can be dominated by variance-driven fluctuations rather than genuine model miscalibration.
As a consequence, the ranking signal required for reliable selection is washed out.
Selection procedures then prioritize \emph{noise} over accuracy, causing severe power collapse precisely in the regimes where selective deployment is most needed.

This paper introduces \emph{Denoised Conformal Alignment}, a framework that makes reliable CATE prediction selection feasible for black-box estimators.
The key idea is to explicitly separate model error from variance-induced noise in DR surrogates by constructing a \emph{denoised proxy error}.
Using consistent nuisance estimation \citep{chernozhukov2018double}, we estimate and subtract the conditional variance component, recovering a ranking signal that is otherwise obscured.
We then calibrate these scores using conformal alignment and apply FDR control to obtain a selected subset that is both \emph{safe} (few false discoveries) and \emph{useful} (nontrivial power) under standard identification assumptions \citep{rosenbaum1983central}.
Beyond the algorithm, we identify a fundamental signal-to-noise barrier for post-selection guarantees:
without denoising, reliable selection can be structurally impossible.
Denoising restores separability while remaining conservative under imperfect variance estimation, a phenomenon we confirm empirically in challenging heteroskedastic regimes.

Our main contributions are summarized as follows:\\
\noindent (i) We propose {Denoised Conformal Alignment}, a deployment-first framework for CATE prediction selection that resolves a fundamental signal-to-noise bottleneck by constructing variance-subtracted proxy errors under heteroskedasticity.\\
\noindent (ii) We establish rigorous asymptotic FDR control for selective deployment of CATE predictions without observable error labels, extending conformalized selection to the counterfactual domain via consistent nuisance estimation under standard identification assumptions and sample splitting.\\
\noindent (iii) We demonstrate empirically that denoising is \emph{structurally necessary} to prevent power collapse, recovering substantial selection yields in high-noise regimes where naive proxy-based baselines fail.

\section{Preliminaries and Problem Formulation}
\label{sec:problem}

In this section, we formalize the problem of reliable CATE selection under the potential outcomes framework \citep{rubin1974estimating} and define the statistical objectives of FDR control and power maximization. We observe $\mathcal{D}_{\text{obs}}=\{(X_i,T_i,Y_i)\}_{i=1}^n$ from a super-population $\mathcal{P}$, where $X_i\in\mathcal X\subseteq\mathbb R^d$, $T_i\in\{0,1\}$, and $Y_i\in\mathbb R$ denote covariates, treatment, and observed outcome.
Following the Neyman--Rubin model, potential outcomes $Y_i(0)$ and $Y_i(1)$ exist, with $Y_i=T_iY_i(1)+(1-T_i)Y_i(0)$. The individual treatment effect is $\tau_i\coloneqq Y_i(1)-Y_i(0)$, and the CATE is $\tau(x)\coloneqq \mathbb E[\tau_i\mid X=x]$.

We treat $\hat\tau:\mathcal X\to\mathbb R$ as a fixed black-box CATE predictor from any existing pipeline, e.g., Causal Forests \citep{wager2018estimation} or meta-learners \citep{kunzel2019metalearners}; our goal is to assess its reliability, not retrain it. To link data to causal quantities,
we make standard identification assumptions:
\begin{assumption}[Identification]
\label{assum:identifiability}
For any $x \in \mathcal{X}$: (1) Unconfoundedness: $\{Y(0), Y(1)\} \perp T \mid X$; and (2) Overlap: $0 < \eta \le \mathbb{P}(T=1 \mid X=x) \le 1 - \eta$.
\end{assumption}

Under Assumption~\ref{assum:identifiability}, causal effects are identifiable at the conditional mean level $\tau(x)$, while the \emph{individual} effect $\tau_i$ remains inherently unobservable due to the missing counterfactual.
Accordingly, we assess reliability by the still-unobservable CATE prediction error
$A_i \coloneqq |\hat{\tau}(X_i)-\tau(X_i)|$.

\textbf{Problem Formulation: Reliable Group Selection} Given fresh test covariates $\{X_{n+j}\}_{j=1}^m$, our goal is to select an index set $\mathcal S\subseteq\{1,\dots,m\}$ whose model predictions are trustworthy.

The tolerance $c$ is specified by the deployment context: smaller values enforce stricter reliability and usually reduce yield, while larger values admit more candidates under a looser reliability requirement.
Given tolerance $c>0$, we cast reliability as multiple testing: for each candidate $j\in\{1,\dots,m\}$,
\begin{equation}
    H_{0,j}: A_{n+j} \ge c \quad \text{versus} \quad H_{1,j}: A_{n+j} < c.
\end{equation}
A selection corresponds to rejecting the null hypothesis $H_{0,j}$. Let $\mathcal{H}_0$ and $\mathcal{H}_1$ denote the sets of indices corresponding to true nulls (unreliable units) and true alternatives (reliable units), respectively:
\begin{equation}
\begin{aligned}
    \mathcal{H}_0 = \{j \in \{1, \dots, m\} : A_{n+j} \ge c\},  \quad \mathcal{H}_1 = \{j \in \{1, \dots, m\} : A_{n+j} < c\}.
    \end{aligned}
\end{equation}
We aim to design a selection procedure controlling False Discovery Rate while maximizing Power.
\begin{definition}[FDR and Power]
$
    \mathrm{FDR}(\mathcal{S}) \coloneqq
    \mathbb{E}\!\left[ \frac{|\mathcal{S} \cap \mathcal{H}_0|}{|\mathcal{S}| \lor 1} \right],
    \quad
    \mathrm{Power}(\mathcal{S}) \coloneqq
    \mathbb{E}\!\left[ \frac{|\mathcal{S} \cap \mathcal{H}_1|}{|\mathcal{H}_1| \lor 1} \right].
$
FDR is the expected fraction of selected unreliable units; Power is the expected fraction of reliable units selected.
\end{definition}


\textbf{Objective.} Ensure $\mathrm{FDR}(\mathcal S)\le\alpha$ for $\alpha\in(0,1)$ while maximizing $\mathrm{Power}(\mathcal S)$\footnote{within score-threshold rules induced by a learned alignment score.}.

The core challenge lies in the unobservability of $A_{n+j}$. This prevents the direct application of standard conformal selection or alignment methods, which require ground-truth error labels for calibration \citep{jin2023selection, gui2024conformal}. This motivates the need for a proxy-based strategy tailored to the causal setting, which we develop in the next section.

\section{Methodology: Denoised Conformal Alignment (DCA)}
\label{sec:method}

We present \emph{Denoised Conformal Alignment}  for reliable selection when test outcomes are unobserved. DCA bridges the gap between the unobservable target $A_i$ and covariates using an outcome-observed reference set $\mathcal D$, split into $\mathcal D_{\mathrm{tr1}}$ for base/nuisance estimation, $\mathcal D_{\mathrm{tr2}}$ for learning the denoised error signal, and $\mathcal D_{\mathrm{cal}}$ for conformal calibration. Throughout, $\hat\tau_i\coloneqq\hat\tau(X_i)$ denotes CATE prediction.

\subsection{Why Naive DR Proxies Lose Power}
\label{subsec:noise_challenge}

\begin{figure}[t]
\centering
\includegraphics[width=\textwidth]{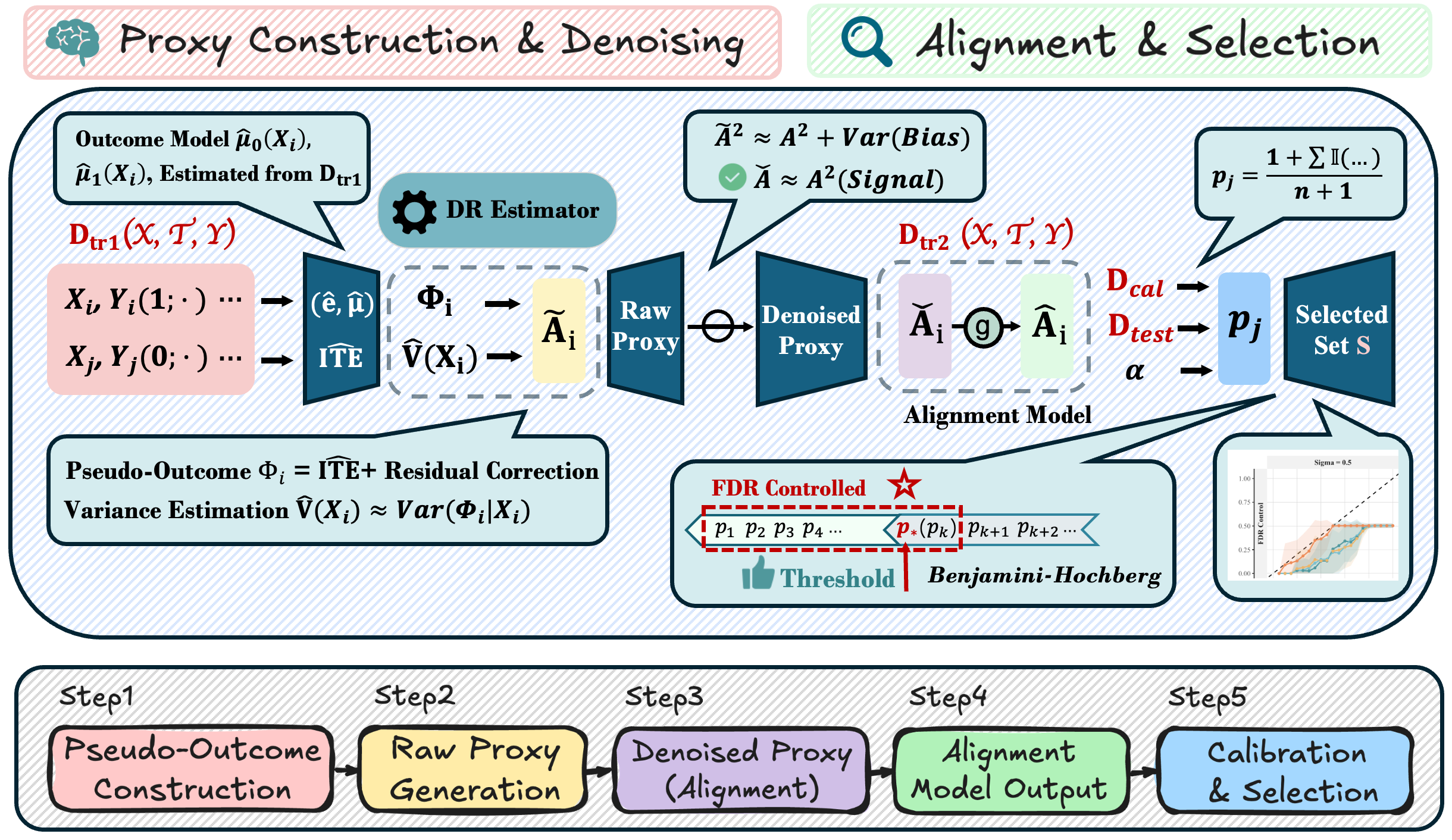}
\caption{\textbf{End-to-end pipeline for denoised proxy construction and selective inference.}
\textbf{(Left) \emph{Initialization and proxy construction.}}
Using $D_{\mathrm{tr}1}$, we estimate the base CATE predictor $\hat\tau(X)$ and nuisance components; DR pseudo-outcomes $\phi$ form the raw proxy $\widetilde A$.
\textbf{(Right) \emph{Denoising, alignment, and selection.}}
Variance subtraction denoises $\widetilde A$ into $\check A$, which trains an alignment model on $D_{\mathrm{tr}2}$ and is calibrated on $D_{\mathrm{cal}}$ to rank candidates and select a subset with controlled FDR.}
\label{fig:pipeline}
\end{figure}
To assess reliability on reference data, we proxy the CATE prediction error $A_i \coloneqq |\hat{\tau}_i-\tau(X_i)|$ using the doubly robust (DR) pseudo-outcome
\begin{equation}
\label{eq:dr_pseudo_outcome}
\phi_i :=
\hat{\mu}_1(X_i) - \hat{\mu}_0(X_i)
+ \frac{T_i(Y_i - \hat{\mu}_1(X_i))}{\hat{e}(X_i)}
- \frac{(1-T_i)(Y_i - \hat{\mu}_0(X_i))}{1 - \hat{e}(X_i)}.
\end{equation}
With correctly specified nuisance functions, and asymptotically under consistent sample-split nuisance estimation, the DR pseudo-outcome is conditionally centered at the CATE: $\mathbb{E}[\phi_i \mid X_i] = \tau(X_i)$ up to vanishing nuisance error.
However, pointwise centering does not ensure proxy reliability. For intuition, consider the corresponding bias-variance decomposition of the squared proxy error:

\begin{equation}
\label{eq:decomposition_proof}
\begin{aligned}
\mathbb{E}[(\hat{\tau}_i - \phi_i)^2 \mid X_i]
&= \mathbb{E}[(\hat{\tau}_i - \tau(X_i) + \tau(X_i) - \phi_i)^2 \mid X_i] \\
&= \underbrace{(\hat{\tau}_i - \tau(X_i))^2}_{A_i^2 \text{ (Signal)}} 
 + \underbrace{\mathbb{E}[(\tau(X_i) - \phi_i)^2 \mid X_i]}_{\mathrm{Var}(\phi_i \mid X_i) \text{ (Noise)}} \\
&= A_i^2 + \mathrm{Var}(\phi_i \mid X_i).
\end{aligned}
\end{equation}

It contains signal $A_i^2$ and noise $\mathrm{Var}(\phi_i\mid X_i)$. The cross-term vanishes because $\mathbb{E}[\tau(X_i) - \phi_i \mid X_i] = 0$. Thus, the raw squared proxy error is true error plus irreducible variance.
In heteroskedastic regimes, $\mathrm{Var}(\phi_i\mid X_i)$ can dominate $A_i^2$, causing selection to rank environmental variance rather than model accuracy and motivating conditional-variance subtraction to recover the ranking signal.

\subsection{Stage 1: Denoised Proxy Construction}
\label{subsec:stage1}

To decouple signal $A_i^2$ from noise, we train an auxiliary variance model on $\mathcal{D}_{\text{tr1}}$ to estimate $\widehat{V}(X_i) \approx \text{Var}(\phi_i \mid X_i)$ by regressing $(\phi_i - \hat{\tau}_i)^2$ on $X_i$. Under nuisance consistency, subtracting $\widehat{V}(X_i)$ yields a proxy that targets $A_i^2$ in conditional expectation.
Let the raw proxy error be $\tilde{A}_i = |\hat{\tau}_i - \phi_i|$. We introduce a \textbf{conservative denoising coefficient} $\rho \in [0, 1]$ and construct the denoised proxy $\check{A}_i$ as:
\begin{equation}
\label{eq:rho_denoising}
    \check{A}_i(\rho) = \sqrt{\max\left\{ \tilde{A}_i^2 - \rho \cdot \widehat{V}(X_i), \ 0 \right\}}.
\end{equation}
Here, $\rho$ controls the \textbf{safety--power trade-off}: larger $\rho$ removes more variance-induced noise, while smaller $\rho$ buffers variance-estimation error and threshold-label flips near $c$.
Thus, full subtraction is not automatically optimal for downstream selection when $\widehat V$ is estimated; the relevant criterion is stability of threshold labels near the deployment boundary.
In practice, we choose $\rho$ from a prespecified grid using only an outcome-observed validation slice, without test-outcome peeking; Appendix~\ref{app:rho_selection} reports sensitivity.


\textbf{Alignment via Prediction.} Since $\check A_i$ depends on observed outcomes, it is unavailable for test candidates. We therefore train an \textbf{alignment predictor} $g(\cdot)$ on $\mathcal D_{\text{tr2}}$ to map covariates to denoised scores, $\hat A=g(X)$, which serve as observable selection scores.

\subsection{Stage 2: Conformal Calibration and Selection}
\label{subsec:stage2}

In Stage~2, Conformal Alignment tests $H_{0,j}:A_{n+j}\ge c$ using predicted scores $\hat A_i=g(X_i)$ on $\mathcal D_{\text{cal}}\cup\mathcal D_{\text{test}}$ and calibration proxy labels $\check A_i$.
For each test point, we compute the left-tail conformal $p$-value by ranking $\hat A_{n+j}$ against calibration units labeled unreliable ($\check A_i\ge c$):
\begin{equation}\label{def_pj}
    p_j = \frac{1 + \sum_{i \in \mathcal{D}_{\text{cal}}} \mathbbm{1}\{\check{A}_i \ge c, \ \hat{A}_i \le \hat{A}_{n+j}\}}{1 + |\mathcal{D}_{\text{cal}}|}.
\end{equation}
Intuitively, since smaller $\hat{A}$ indicates higher predicted reliability, if $\hat{A}_{n+j}$ is smaller than the scores of most unreliable calibration units, the count in the numerator is small, yielding a small p-value that favors selection.

Finally, we apply the Benjamini--Hochberg (BH) procedure to $\{p_j\}_{j=1}^m$ at level $\alpha$, returning the selected set $\mathcal S$.
Combined with the $p$-value validity established in Section~\ref{sec:theory}, this yields asymptotic control of the false discovery rate (FDR) as defined in Section~\ref{sec:problem}.
In words, BH finds the largest rank $\hat k$ such that the $\hat k$-th smallest $p$-value is below $(\hat k/m)\alpha$, and selects exactly those $\hat k$ candidates.



\begin{algorithm}[t]
   \caption{DCA with sample splitting}
   \label{alg:causal_fdr}
\begin{algorithmic}[1]
   \REQUIRE Reference data $\mathcal{D}$; Candidate set $\mathcal{D}_{\text{test}}$; Learners $\mathcal{G}_{\tau}, \mathcal{G}_g$; Tolerance $c$; Target FDR $\alpha$.
   \ENSURE Selected set $\mathcal{S} \subseteq \{1, \dots, m\}$. Randomly split $\mathcal{D}$ into $\mathcal{D}_{\text{tr1}},\mathcal{D}_{\text{tr2}},\mathcal{D}_{\text{cal}}$.

\textbf{Stage 1: Proxy construction and denoising}
   \STATE Train $\hat{\mu}_0,\hat{\mu}_1$ (thus $\hat\tau=\hat\mu_1-\hat\mu_0$), propensity $\hat e$, and variance model $\hat V$ on $\mathcal{D}_{\text{tr1}}$.
   \STATE For each $i\in \mathcal{D}_{\text{tr2}}\cup\mathcal{D}_{\text{cal}}$, compute DR pseudo-outcome $\phi_i$ and denoised proxy:\\
   \STATE \quad $\check{A}_i \leftarrow \sqrt{\max\left\{ (\hat\tau(X_i)-\phi_i)^2 - \rho \cdot \widehat{V}(X_i), \ 0 \right\}}$

   \STATE Train alignment predictor $g$ on $\mathcal{D}_{\text{tr2}}$:  $g \leftarrow \mathcal{G}_g(\{(X_i,\check{A}_i)\}_{i\in\mathcal{D}_{\text{tr2}}})$.

\textbf{Stage 2: Conformal calibration and BH selection}
   \STATE Compute scores $\hat A_i \leftarrow g(X_i)$ for $i\in \mathcal{D}_{\text{cal}}\cup\mathcal{D}_{\text{test}}$.
   \STATE Compute $p_j \leftarrow \frac{1+\sum_{i\in\mathcal{D}_{\text{cal}}}\mathbbm{1}\{\check{A}_i \ge c,\ \hat A_i \le \hat A_{n+j}\}}{1+|\mathcal{D}_{\text{cal}}|}$ for $j = 1,2,...m$.
   \STATE $\mathcal{S} \leftarrow \mathrm{BH}(p_1,\dots,p_m,\alpha)$.
\end{algorithmic}
\small \textit{Note: Under covariate shift, replace the unweighted conformal counts by importance-weighted counts (Eq.~\eqref{eq:weighted_pvalue}).}
\end{algorithm}

\subsection{Extension: Covariate Shift with Weighted $p$-values}
\label{subsec:cov_shift}

Stage~2 uses split-conformal calibration and therefore relies on calibration and test candidates being comparable in covariate distribution. In some deployments, the outcome-observed reference sample and the candidate pool may differ in their marginal covariate distribution. We handle this case through an importance-weighted variant of the conformal $p$-values.
This extension is for covariate shift only: the marginal distribution of $X$ may change, while the conditional causal mechanism and the reliability target given $X$ remain stable.
The DCA pipeline is otherwise unchanged: proxy construction, denoising, alignment-score learning, and BH selection are the same; only the calibration counts are reweighted.

Let $P_{\mathrm{src}}$ and $P_{\mathrm{tgt}}$ denote source and target covariate distributions, and let
$w(x)=dP_{\mathrm{tgt}}(x)/dP_{\mathrm{src}}(x)$
be the density ratio. We replace the left-tail conformal $p$-value in~\eqref{def_pj} by

\begin{equation}
\label{eq:weighted_pvalue}
p^{(w)}_j=
\frac{w(X_{n+j})+\sum_{i\in\mathcal D_{\mathrm{cal}}}w(X_i)\,
\mathbbm{1}\{\check A_i\ge c,\ \hat A_i\le \hat A_{n+j}\}}
{w(X_{n+j})+\sum_{i\in\mathcal D_{\mathrm{cal}}}w(X_i)}.
\end{equation}

Intuitively, this reweights calibration units so that the rank comparison approximates one drawn from the target covariate distribution rather than the reference distribution.

In practice, $w$ can be estimated from unlabeled source/target covariates and clipped or normalized for stability. We use this weighted variant only in the covariate-shift experiments. The weighted version should be read as an asymptotic drop-in extension in the main text; the precise shift assumptions, regularity conditions, and finite-sample weighted-selection discussion are deferred to Appendix~\ref{app:setting3}, so that the main text focuses on the core DCA mechanism: denoising noisy causal proxies, controlling post-selection risk, and addressing sample splitting, variance estimation, and baseline robustness.

\section{Theoretical Analysis}
\label{sec:theory}

We seek a selection set $\mathcal{S}\subseteq\{1,\ldots,m\}$ such that
$
\mathrm{FDR}(\mathcal{S}) \le \alpha
~\text{while maximizing}~
\mathrm{Power}(\mathcal{S})
$
under score-threshold rules induced by a learned alignment score.
With tolerance $c>0$ labeling unreliability in \eqref{def_pj}, reliability is based on
$A_i := |\hat\tau(X_i)-\tau(X_i)|$:
a unit is good if $A_i<c$ and bad otherwise.
Our theories are formalized as follows:

(i) \textit{Oracle finite-sample validity.} If the true errors $A_i$ were observable, conformal alignment + BH yields \emph{finite-sample} FDR control (Lemma~\ref{lemma:oracle}), recovering \citet{gui2024conformal} as a special case. 

(ii) \textit{Finite-sample robustness to proxy error.} When $A_i$ is replaced by a proxy label, we quantify how much validity can degrade in finite samples via an explicit \emph{perturbation bound} (Proposition~\ref{prop:finite_sample_approx_fdr}). 

(iii) \textit{Main theorem (causal setting).} Under identification and nuisance consistency, our denoised proxy $\check A_i$ makes the perturbation vanish, yielding \emph{asymptotic} FDR control for DCA (Theorem~\ref{thm:main}).

(iv) \textit{Power and the signal-to-noise barrier.} We formalize why naive DR proxies can be uninformative for selection in heteroskedastic regimes, leading to power collapse, and why denoising restores separability (Proposition~\ref{prop:snr_barrier}) and induces optimal power (Proposition~\ref{prop:power_denoised} and Proposition~\ref{prop:optimal_threshold_power}).

\textbf{Assumptions.} Our results rely on (i) causal identification for $\tau(x)$, (ii) split-conformal exchangeability to yield rank-based $p$-values, and (iii) nuisance consistency to ensure $\check A_i$ approximates $A_i$.

\begin{assumption}[Split-conformal exchangeability]
\label{assum:exchangeability}
Conditional on the data used to fit the scoring rule $g$ (and the base CATE predictor $\hat\tau$), the calibration and test covariates are exchangeable:
$\{X_i\}_{i\in\mathcal{D}_{\mathrm{cal}}}\cup\{X_{n+j}\}_{j=1}^m$ are i.i.d.\ from the same distribution.
Equivalently, conditional on the trained $g$, the pairs $\{(\hat A_i,A_i)\}_{i\in\mathcal{D}_{\mathrm{cal}}\cup\mathcal{D}_{\mathrm{test}}}$ are exchangeable.
\end{assumption}

\begin{assumption}[Nuisance consistency]
\label{assum:nuisance}
Let $\hat{\mu}_t, \hat{e}, \hat{V}$ be nuisance and variance estimators trained on an independent split $\mathcal{D}_{\mathrm{tr1}}$ (via sample splitting to avoid leakage).
As $|\mathcal{D}_{\mathrm{tr1}}| \to \infty$, they are consistent in mean square error:
$
\|\hat{\mu}_t - \mu_t\|_2 = o_p(1), \quad
\|\hat{e} - e\|_2 = o_p(1), \quad
\|\hat{V} - \mathrm{Var}(\phi \mid X)\|_2 = o_p(1).
$
\end{assumption}

\begin{assumption}[No mass at the tolerance boundary]
\label{assum:continuity}
The distribution of $A_i$ is continuous at $c$, i.e., $\mathbb{P}(A_i=c)=0$.
\end{assumption}

Assumption~\ref{assum:exchangeability} is standard for split conformal inference; sample splitting makes it plausible by training $g$ without calibration or test candidates.
Assumptions~\ref{assum:nuisance}--\ref{assum:continuity} should be interpreted through boundary stability rather than pointwise accuracy of the variance estimator.
The FDR guarantee does not require $\widehat V$ to be exactly correct; concretely, variance misspecification matters only when it flips proxy/oracle threshold labels near $c$.
Appendix~\ref{app:boundary_stability} gives a margin-based interpretation.

\textbf{Covariate shift (drop-in weighted calibration).}
Assumption~\ref{assum:exchangeability} can fail under deployment shift, i.e., when $P_{\mathrm{src}}(X)\neq P_{\mathrm{tgt}}(X)$ while the conditional causal mechanism given $X$ remains stable.
A drop-in fix is to replace the unweighted conformal $p$-values by the importance-weighted version in Eq.~\eqref{eq:weighted_pvalue}, with density ratio $w(x)=dP_{\mathrm{tgt}}/dP_{\mathrm{src}}(x)$.
This yields the appropriate weighted ranking comparison under covariate shift. Under the weighted exchangeability and regularity conditions stated in Appendix~\ref{app:setting3}, the oracle weighted $p$-values are valid, proxy-to-oracle deviation is controlled by weighted calibration mislabeling, and the same asymptotic FDR argument carries over.
Formal statements, finite-sample weighted-selection details, and proofs are deferred to Appendix~\ref{app:setting3}.

\subsection{Finite Sample Oracle FDR control}
\label{sec:theory_oracle}

We first analyze an oracle setting where the true errors $A_i$ are observable on the calibration set.
In this case, we can compute \emph{oracle p-values}:
\begin{equation}
\label{eq:oracle_pvalues}
p_j^*
=
\frac{1 + \sum_{i \in \mathcal{D}_{\mathrm{cal}}}
\mathbbm{1}\{A_i \ge c,\ \hat A_i \le \hat A_{n+j}\}}
{1 + |\mathcal{D}_{\mathrm{cal}}|}.
\end{equation}

\begin{lemma}[Oracle FDR control]
\label{lemma:oracle}
Under Assumption~\ref{assum:exchangeability}, applying BH at level $\alpha$ to the oracle $p$-values $\{p_j^*\}_{j=1}^m$ in Eq.~\eqref{eq:oracle_pvalues} yields a selected set $\mathcal S^*$ satisfying $\mathrm{FDR}(\mathcal S^*)\le \alpha$.
\end{lemma}

If we somehow knew which calibration individuals are bad (error $\ge c$),
then conformal alignment produces valid $p$-values and BH ensures that among the selected individuals, the expected fraction of truly bad ones is at most $\alpha$.
This is exactly the deployment safety promise we seek. It is not a guarantee about prediction intervals, and it does not assert that $\hat\tau$ is accurate for everyone.

\subsection{Finite-sample robustness to proxy labels}
\label{sec:theory_finite_sample}

In causal problems $A_i$ is unobservable, so we replace oracle labels by proxy labels from $\check A_i$ and use the proxy p-values in Eq.~\eqref{def_pj}.
The only way this can go wrong is if proxy labels $\mathbbm{1}\{\check A_i\ge c\}$ disagree with the (unobserved) oracle labels $\mathbbm{1}\{A_i\ge c\}$. Define the calibration mislabeling rate
\begin{equation}
\label{eq:mislabel_rate}
\widehat{\Delta}_{\mathrm{cal}}
\;:=\;
\frac{1}{|\mathcal{D}_{\mathrm{cal}}|}
\sum_{i\in\mathcal{D}_{\mathrm{cal}}}
\left|\mathbbm{1}\{\check A_i\ge c\}-\mathbbm{1}\{A_i\ge c\}\right|.
\end{equation}

\begin{proposition}[Finite-sample approximate validity via mislabeling]
\label{prop:finite_sample_approx_fdr}
For every test index $j$,
$\big|\tilde p_j - p_j^*\big| \le \widehat{\Delta}_{\mathrm{cal}}.$
In particular, for any $u\in[0,1]$ and any null $j$ (i.e., $A_{n+j}\ge c$),
\begin{equation}
\label{eq:approx_superuniform}
\mathbb{P}\!\left(\tilde p_j \le u \,\middle|\, g \right)
\;\le\;
u + \mathbb{E}\!\left[\widehat{\Delta}_{\mathrm{cal}} \,\middle|\, g\right].
\end{equation}
Consequently, BH applied to $\{\tilde p_j\}_{j=1}^m$ controls FDR up to an additive perturbation that vanishes whenever $\widehat{\Delta}_{\mathrm{cal}}$ is small, i.e., $FDR(\mathcal{S}) \leq \alpha + m\mathbb{E}\!\left[\widehat{\Delta}_{\mathrm{cal}} \,\middle|\, g\right].$
\end{proposition}

This bound isolates the \emph{only} statistical bottleneck introduced by counterfactual unobservability:
we do not need perfect proxy values, only \emph{stable classification} around the tolerance threshold $c$.
It also addresses a common (and important) skepticism:
\emph{``Isn’t DR unbiasedness already enough to justify proxy-based selection?''}
No---unbiasedness is a statement about conditional means, while post-selection validity depends on \emph{boundary stability} of $\mathbbm{1}\{\check A_i\ge c\}$.
Heteroskedastic variance can cause frequent label flips even when DR is unbiased, inflating $\widehat{\Delta}_{\mathrm{cal}}$ and collapsing usefulness.

Moreover, the additive inflation term in Proposition~\ref{prop:finite_sample_approx_fdr} is governed by the calibration error
$\widehat\Delta_{\mathrm{cal}}$.
Under standard empirical process arguments, when the calibration set has size
$|\mathcal D_{\mathrm{cal}}|=n_{\mathrm{cal}}$, we typically have
$
\widehat\Delta_{\mathrm{cal}}
=
O_p\!\left(\sqrt{\frac{\log m}{n_{\mathrm{cal}}}}\right),
$
uniformly over $m$ hypotheses.
Consequently,
$
m\,\mathbb E[\widehat\Delta_{\mathrm{cal}}\mid g]
=
O\!\left(m\sqrt{\frac{\log m}{n_{\mathrm{cal}}}}\right).
$
In particular, the FDR inflation term vanishes provided that
$n_{\mathrm{cal}}\gg m^2\log m$.
This condition characterizes the trade-off between the number of hypotheses and the
calibration sample size required for asymptotically exact FDR control.

\subsection{Asymptotic FDR control with Denoised Proxies}
\label{sec:theory_asymptotic}

We now show that our denoised proxy construction makes the mislabeling perturbation vanish, recovering oracle validity asymptotically.

\begin{theorem}[Asymptotic FDR control]
\label{thm:main}
Suppose Assumptions~\ref{assum:identifiability}--\ref{assum:continuity} hold.
Let $\mathcal{S}$ be the set selected by applying BH at level $\alpha$ to the proxy p-values $\{\tilde{p}_j\}_{j=1}^m$.
Then, as the sample sizes of the reference data (used to estimate nuisances/variance and to train $g$) and the test size $m$ grow,
$
\limsup_{n,m \to \infty} \mathrm{FDR}(\mathcal{S}) \le \alpha.
$
\end{theorem}

Theorem~\ref{thm:main} is a \emph{safety} result: it guarantees FDR control, not nontrivial power.
A method that selects nobody trivially satisfies FDR$\le\alpha$.
The next subsection explains why denoising is the key ingredient that prevents this vacuity in heteroskedastic regimes.

\subsection{Power and the Signal-to-noise Barrier}
\label{sec:theory_snr}

Why can naive proxy-based selection become useless?
The short answer is that reliability is a \emph{ranking} problem: to achieve power while controlling FDR, the score must separate good units ($A<c$) from bad ones ($A\ge c$).
Naive DR proxies can destroy exactly this separability. Recall from Section~\ref{sec:method} that the squared naive proxy error satisfies, conditionally on $X$,
$
\mathbb{E}\!\left[(\hat\tau(X)-\phi)^2 \mid X\right] = A(X)^2 + V(X),
$
where $V(X):=\mathrm{Var}(\phi\mid X)$ captures irreducible, heteroskedastic noise.

\begin{proposition}[Signal-to-noise barrier for proxy ranking]
\label{prop:snr_barrier}
Consider the stylized model
$
\tilde A^2 = A^2 + V(X) + \varepsilon,
$
where $\varepsilon$ is mean-zero noise independent of $(A^2,V(X))$, and suppose $A^2$ is weakly related to $V(X)$.
Then the correlation between the naive proxy $\tilde A^2$ and the target signal $A^2$ satisfies
$
\mathrm{Corr}(\tilde A^2, A^2)
=
{\mathrm{Var}(A^2)}/{\sqrt{\mathrm{Var}(A^2)\big(\mathrm{Var}(A^2)+\mathrm{Var}(V(X))+\mathrm{Var}(\varepsilon)\big)}}.
$
In particular, in the high-noise heteroskedastic regime where $\mathrm{Var}(V(X))+\mathrm{Var}(\varepsilon)\gg \mathrm{Var}(A^2)$,
$
\mathrm{Corr}(\tilde A^2, A^2)\to 0,
$
so any selection rule based primarily on ranking by $\tilde A$ becomes asymptotically uninformative for distinguishing $A<c$ from $A\ge c$, leading to vanishing power at fixed FDR level.
\end{proposition}

Proposition~\ref{prop:snr_barrier} formalizes the failure mode we highlight in the introduction:
even if DR is unbiased, heteroskedastic variance can overwhelm the reliability signal.
Selection then prioritizes variance rather than model accuracy, forcing BH to be conservative and collapsing power.

\textbf{How denoising fixes it, and why $\rho$ matters.}
DCA replaces $\tilde A^2$ by a variance-subtracted proxy
$
\check A^2(\rho) = \big(\tilde A^2 - \rho \hat V(X)\big)_+.
$
Under perfect variance knowledge (oracle $\hat V=V$) and ignoring truncation for intuition,
$
\mathbb{E}[\check A^2(\rho)\mid X] \approx A^2 + (1-\rho)V(X),
$
so increasing $\rho$ directly increases the signal-to-noise ratio of the ranking score, improving separability and power.
At the same time, $\rho<1$ provides robustness to variance estimation error:
it reduces the risk of over-subtraction (which could misclassify truly bad units as good), thereby protecting the FDR guarantee.
This is precisely the safety--usefulness trade-off controlled by $\rho$ in practice.


\textbf{Asymptotic Power after Denoising} We now characterize the \emph{asymptotic power} of Denoised Conformal Alignment (DCA) and make precise how denoising restores nontrivial selection power under FDR control.
Our result parallels the power analysis of \citet{gui2024conformal}, while addressing the additional difficulty that, in causal settings, reliability labels are counterfactual and must be inferred through noisy proxies. Recall that the alignment score $\hat A = g(X)$ is trained to predict the denoised proxy $\check A$, and that selection is based on the conformalized ranking of $\hat A$.
Let
$
H(t) := \mathbb{P}(A \ge c,\ g(X) \le t),
\qquad
t(\alpha) := \sup\Big\{t:\ \frac{H(t)}{\mathbb{P}(g(X)\le t)} \le \alpha \Big\},
$
where $A = |\hat\tau(X)-\tau(X)|$ denotes the oracle CATE prediction error.
The threshold $t(\alpha)$ is the population-level cutoff implicitly targeted by the BH procedure under conformal calibration.

\begin{proposition}[Asymptotic power of denoised conformal alignment]
\label{prop:power_denoised}
Under the conditions of Theorem~\ref{thm:main}, assume further that

    (i) $(X_i, A_i)_{i\ge 1}$ are i.i.d.;
    (ii) the alignment score $g(X)$ has a continuous distribution;
    (iii) there exists $\varepsilon>0$ such that
    $
    \frac{H(t)}{\mathbb{P}(g(X)\le t)} < \alpha
    \quad\text{for all } t\in[t(\alpha)-\varepsilon,\ t(\alpha)].
    $ Then, as $|\mathcal D_{\mathrm{cal}}|,\, m \to \infty$, the power of DCA converges to
$
\lim \mathrm{Power}
=
\mathbb{P}\!\left(g(X)\le t(\alpha)\ \middle|\ A < c\right).
$

\end{proposition}

Proposition~\ref{prop:power_denoised} shows that, asymptotically, DCA selects all units whose alignment score $g(X)$ is no larger than the cutoff $t(\alpha)$, and that the resulting power is exactly the mass of truly reliable units ($A<c$) lying to the left of this cutoff.

Compared to the power analysis of \citet{gui2024conformal}, the present result is more delicate: the event $A<c$ is counterfactual and unobservable, and $g(X)$ is trained on proxy labels.
Denoising is used to stabilize threshold labels near $c$ and align the learned score with the true reliability ordering, rather than to require pointwise perfect recovery of $A$.
Without denoising, heteroskedastic noise in DR pseudo-outcomes can dominate the proxy signal, making the conditional distributions of $g(X)\mid A\ge c$ and $g(X)\mid A<c$ overlap and driving $t(\alpha)$ to the extreme tail, so limiting power degenerates to zero even though FDR control still holds\footnote{Our score-threshold formulation is equivalent to \citet{gui2024conformal} by monotonicity of the score transform.}.
We next formalize optimality of power.

\begin{proposition}[Optimal power among score-threshold rules]
\label{prop:optimal_threshold_power}
Consider selection rules induced by a fixed alignment score $g(X)$:
$\mathcal S_t:=\{j:g(X_{n+j})\le t\}$, $t\in\mathbb R$.
Define $H(t):=\mathbb P(A\ge c,g(X)\le t)$ and $R(t):=\mathbb P(g(X)\le t)$, and let
$t(\alpha):=\sup\Big\{t:\frac{H(t)}{R(t)}\le\alpha\Big\}$.
Then, among all threshold rules satisfying the population-level constraint $H(t)/R(t)\le\alpha$,
$
\mathcal S_{t(\alpha)}
\in
\arg\max_{t:\,H(t)/R(t)\le\alpha}
\mathbb P\!\left(g(X)\le t\mid A<c\right).
$
\end{proposition}

Proposition~\ref{prop:optimal_threshold_power} formalizes a natural but nontrivial optimality principle for selective deployment under score-threshold rules $\mathcal S_t=\{j:g(X_{n+j})\le t\}$.
Increasing $t$ selects more units and increases power, but also admits more unreliable units and raises the population-level error rate.
The cutoff $t(\alpha)$ is precisely the largest threshold satisfying the population-level constraint: increasing it further would violate the constraint, while decreasing it would unnecessarily discard reliable units.
Thus, $\mathcal S_{t(\alpha)}$ exhausts the admissible population-level error budget and achieves the largest fraction of truly reliable units among admissible threshold rules.

\section{Experiments}
\label{sec:experiments}

We evaluate Denoised Conformal Alignment (DCA) for \emph{deployment-first} reliable CATE selection: selecting candidates with trustworthy CATE predictions while controlling post-selection FDR at level $\alpha$ (Section~\ref{sec:problem}). We report realized FDR on the selected set and power (selection yield) as $\alpha$ varies.

Experiments are designed to support four claims from Section~\ref{sec:theory}: \textbf{(G1)}\emph{ \textbf{Safety:}} stable FDR control across data regimes; \textbf{(G2)}\emph{ \textbf{Usefulness under proxy contamination:}} naive DR-based proxies can lose power when heteroskedasticity or heavy tails inflate proxy noise (Proposition~\ref{prop:snr_barrier}), while variance-aware denoising restores informative ranking; \textbf{(G3)}\emph{ \textbf{Robustness to covariate shift:}} importance-weighted conformal $p$-values provide a drop-in calibration modification when $P_{\rm src}(X)\neq P_{\rm tgt}(X)$ under stable conditional causal mechanisms; \textbf{(G4)}\emph{ \textbf{External validity:}} DCA remains effective on semi-synthetic benchmarks with realistic covariates and ground-truth CATE, and produces meaningful qualitative diagnostics on real data.

The two-stage pipeline uses disjoint sample splitting: $\mathcal D_{\mathrm{tr1}}$ fits nuisances and the conditional variance model, $\mathcal D_{\mathrm{tr2}}$ trains the alignment predictor $g$, and $\mathcal D_{\mathrm{cal}}$ calibrates conformal $p$-values for selecting from an unlabeled candidate pool $\mathcal D_{\mathrm{test}}$. This validity-first split preserves split-conformal exchangeability and prevents outcome leakage; Appendix~\ref{app:exp_details} reports full sample-size and cross-fitting checks, showing that gains are not driven by a favorable split and outlining data-reuse variants. Unless stated otherwise, we use Random Forests for $(\hat\mu_t,\hat V)$, logistic regression for $\hat e$, and Random Forest regressor for $g$.
Full setup, robustness checks (split/cross-fitting, denoising strength, variance misspecification, tolerance, and alternative proxies), and IHDP/NLSM/NSW details are in Appendix~\ref{app:exp_details}. 
The tolerance $c$ is an application-level acceptability threshold (Section~\ref{sec:problem}). Denoising strength $\rho$ controls variance subtraction in Eq.~\eqref{eq:rho_denoising} (theory uses $\rho\in[0,1]$); we include $\rho>1$ only as a stress test. Representative $\rho$ values are chosen by a fixed rule on an outcome-observed validation slice, with no test-outcome peeking, since proxy contamination varies by regime.

\textbf{Baselines.} We compare proxy- and uncertainty-based scores: \emph{Naive (plug-in)} uses the raw DR proxy $\tilde A=|\hat\tau-\phi|$, \emph{Ensemble Std} uses predictive disagreement, and \emph{CQR width} uses conformalized quantile regression interval width \citep{romano2019conformalized}. All methods use the same conformal alignment and BH selection; only the scalar score differs (smaller scores indicate higher reliability).

\begin{figure}[t]
    \centering
    \includegraphics[width=\textwidth]{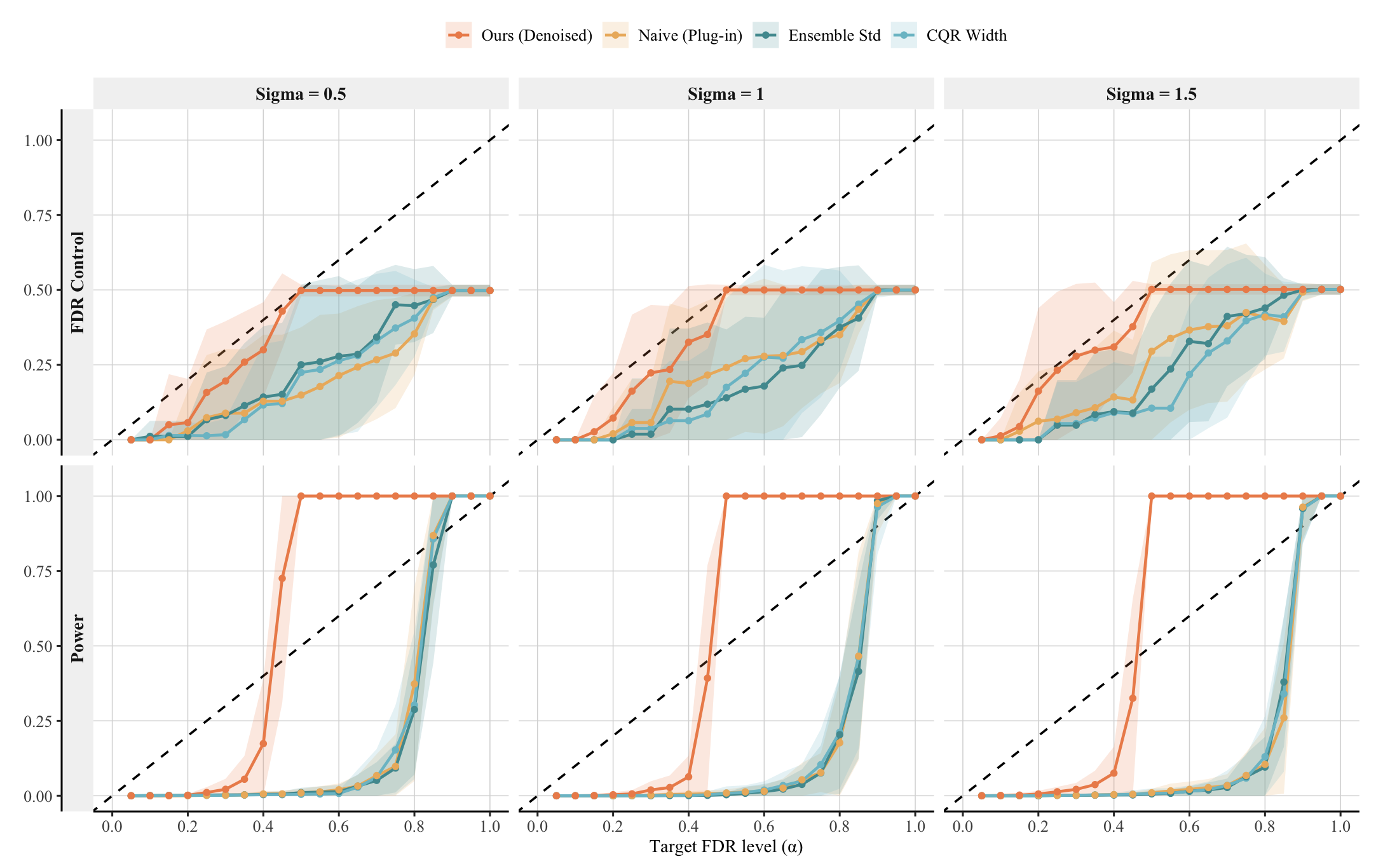}
    \caption{\textbf{Gaussian outcomes with heteroskedastic proxy noise.}
Realized FDR (top) and power (bottom) versus target $\alpha$ for $\sigma\in\{0.5,1.0,1.5\}$ with $\rho=0.65$.
    Across all $\sigma$, DCA stays close to or below the nominal FDR line while achieving much higher power, indicating that variance-aware denoising mitigates the heteroskedastic proxy-noise barrier. More results are in Appendix~\ref{app:exp_details}.}
    \label{fig:synth_noshift}
\end{figure}

\textbf{Heteroskedastic proxy noise.} We simulate a nonlinear causal model with heteroskedastic outcomes and vary $\sigma$. Figure~\ref{fig:synth_noshift} shows a proxy-noise barrier: proxy-based baselines have lower power at moderate $\alpha$, improving only at larger $\alpha$, consistent with distorted ranking signal (Section~\ref{sec:theory_snr}). In contrast, \textbf{DCA achieves higher power at comparable realized FDR}, supporting (G1--G2).

\begin{figure}[htbp]
    \centering
    \includegraphics[width=0.9\textwidth]{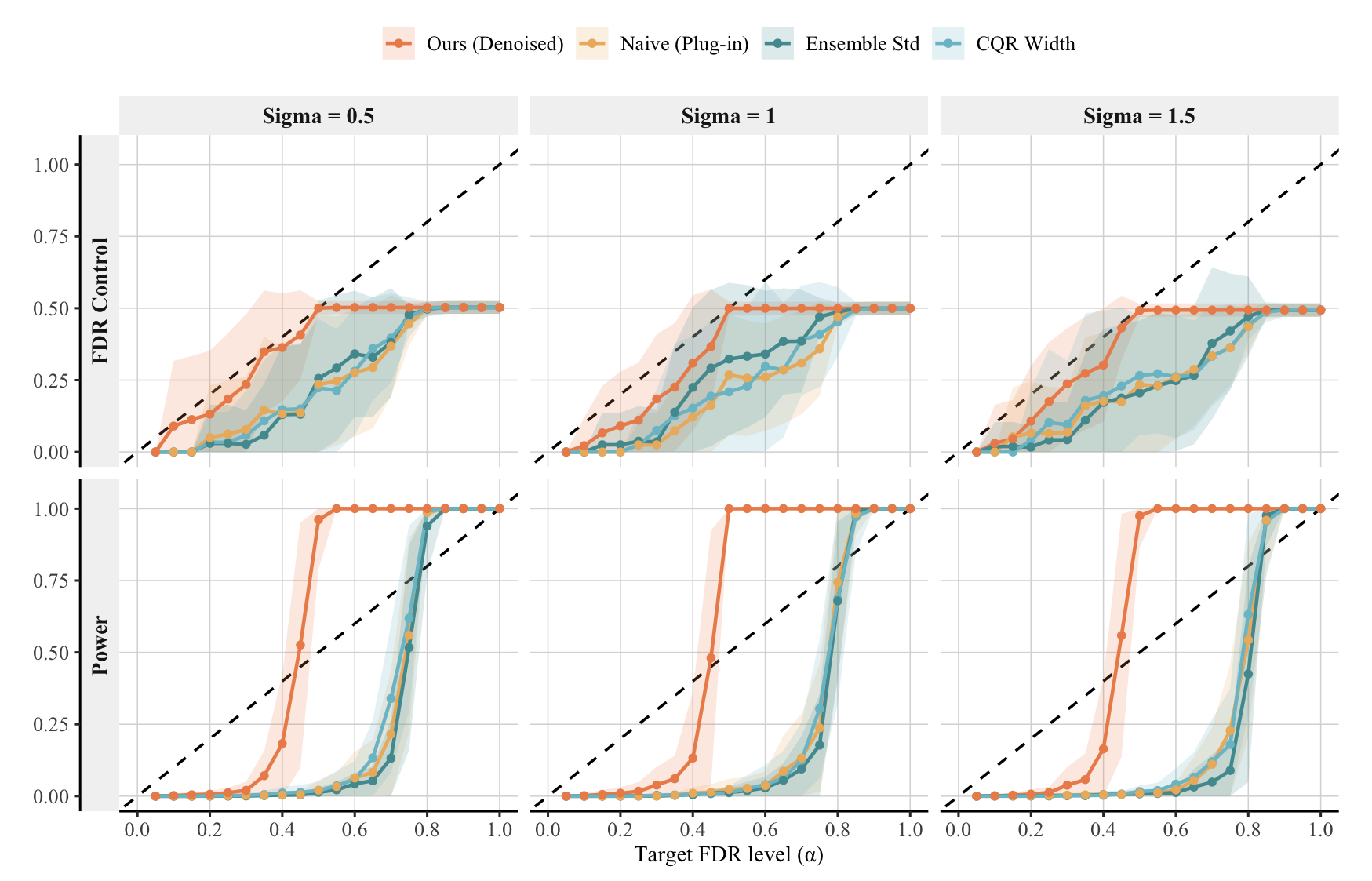}
    \caption{\textbf{Hard overlap with heavy-tailed outcomes.}
    Realized FDR (top) and power (bottom) versus target $\alpha$ for $\sigma\in\{0.5,1.0,1.5\}$ with a conservative denoising strength $\rho=0.15$.
    Even when inverse-propensity factors and heavy-tailed residuals make raw DR proxies unstable, DCA recovers nontrivial selection yield while keeping realized FDR controlled or conservative.}
    \label{fig:main_setting2_heavytail}
\end{figure}

\textbf{Hard overlap + heavy tails.} This setting compounds two failure modes for proxy-based selection: near-extreme propensity scores amplify DR corrections, and heavy-tailed residuals create occasional proxy spikes. 
Figure~\ref{fig:main_setting2_heavytail} shows that \textbf{DCA remains useful in this difficult regime, but prefers a smaller denoising strength}.
This is consistent with our boundary-stability view: when $\widehat V$ is harder to estimate, modest subtraction can reduce variance-driven ranking errors without over-collapsing proxy labels near the tolerance boundary.

\begin{figure}[t]
    \centering
    \begin{minipage}[t]{0.49\textwidth}
        \centering
        \includegraphics[width=\linewidth]{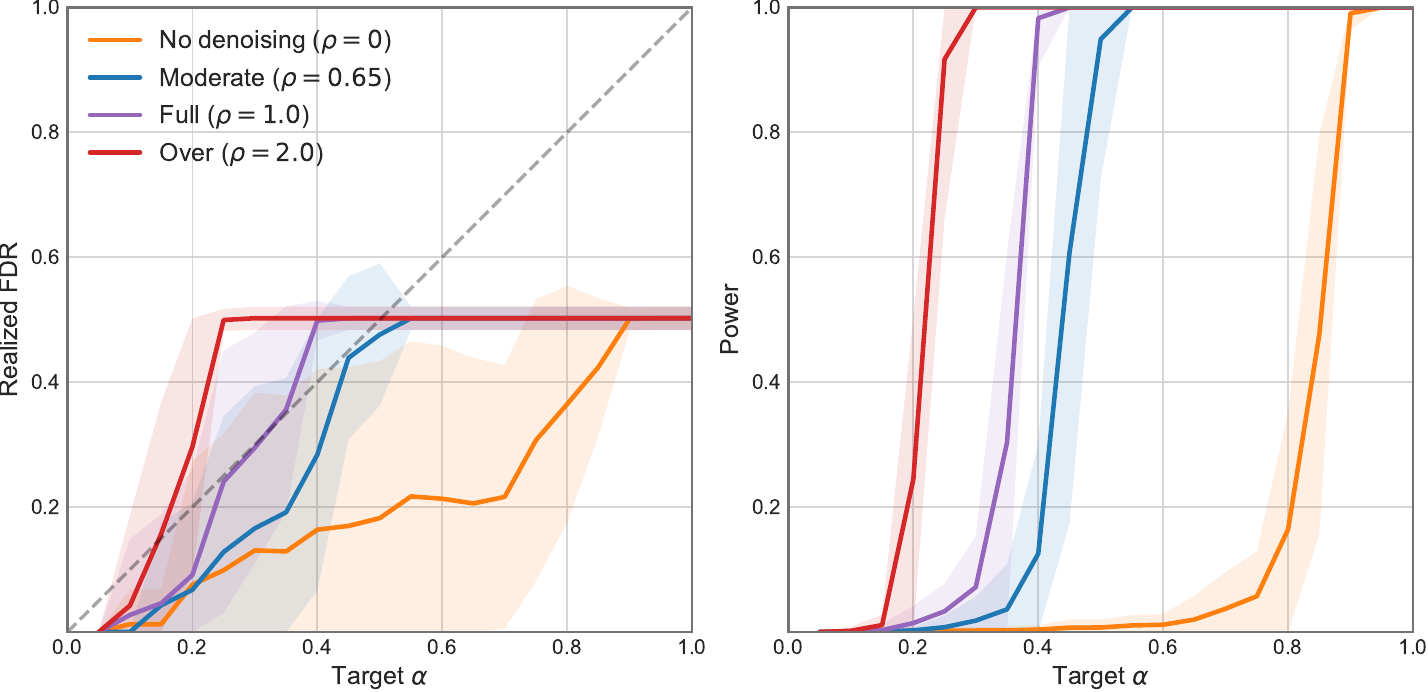}
        \vspace{-0.5em}
        \centerline{\small \textbf{(a)} Denoising strength}
    \end{minipage}
    \hfill
    \begin{minipage}[t]{0.49\textwidth}
        \centering
        \includegraphics[width=\linewidth]{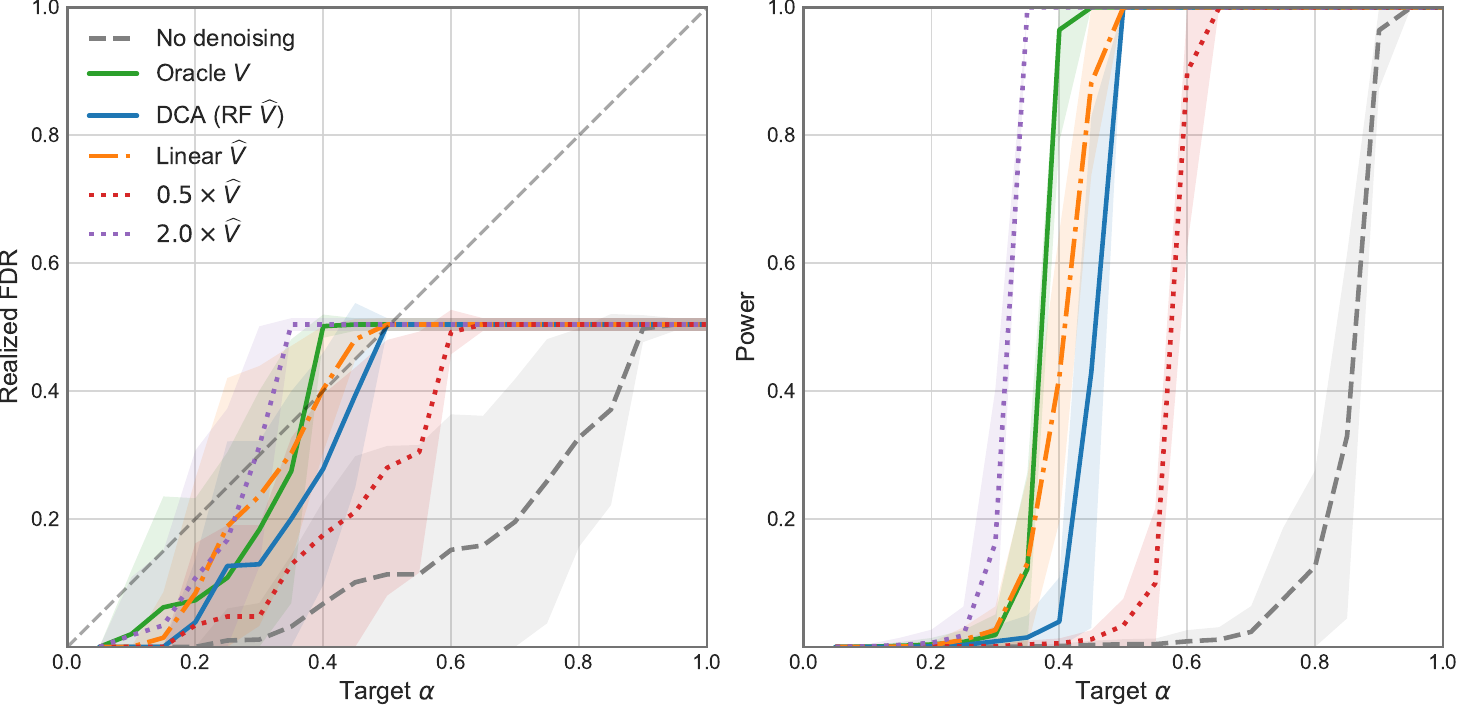}
        \vspace{-0.5em}
        \centerline{\small \textbf{(b)} Variance misspecification}
    \end{minipage}
    \caption{\textbf{Denoising ablations and robustness to imperfect variance estimates.}
    Panel (a) compares no denoising, moderate subtraction, full subtraction, and over-subtraction.
    Panel (b) perturbs the variance model through well-specified, misspecified, under-estimated, and over-estimated variants.
    DCA's gain does not rely on perfect recovery of $\widehat V$: imperfect variance estimates mainly affect the procedure through boundary-label stability near $c$, leading to graceful degradation rather than catastrophic loss of calibration.}
    \label{fig:main_denoising_ablation_misspec}
\end{figure}

\textbf{Denoising strength and variance misspecification.} Figure~\ref{fig:main_denoising_ablation_misspec} directly tests the role of the variance model. Moderate subtraction improves the safety--power trade-off relative to $\rho=0$, while overly aggressive subtraction can reduce reliability when many proxy scores are truncated or when $\widehat V$ is noisy. Under misspecified variance models, DCA degrades smoothly: exact conditional-variance recovery is a sufficient route to asymptotic exactness, but the finite-sample object that matters for FDR is the induced proxy/oracle threshold-label agreement near $c$.

\begin{figure}[t]
    \centering
    \includegraphics[width=\textwidth]{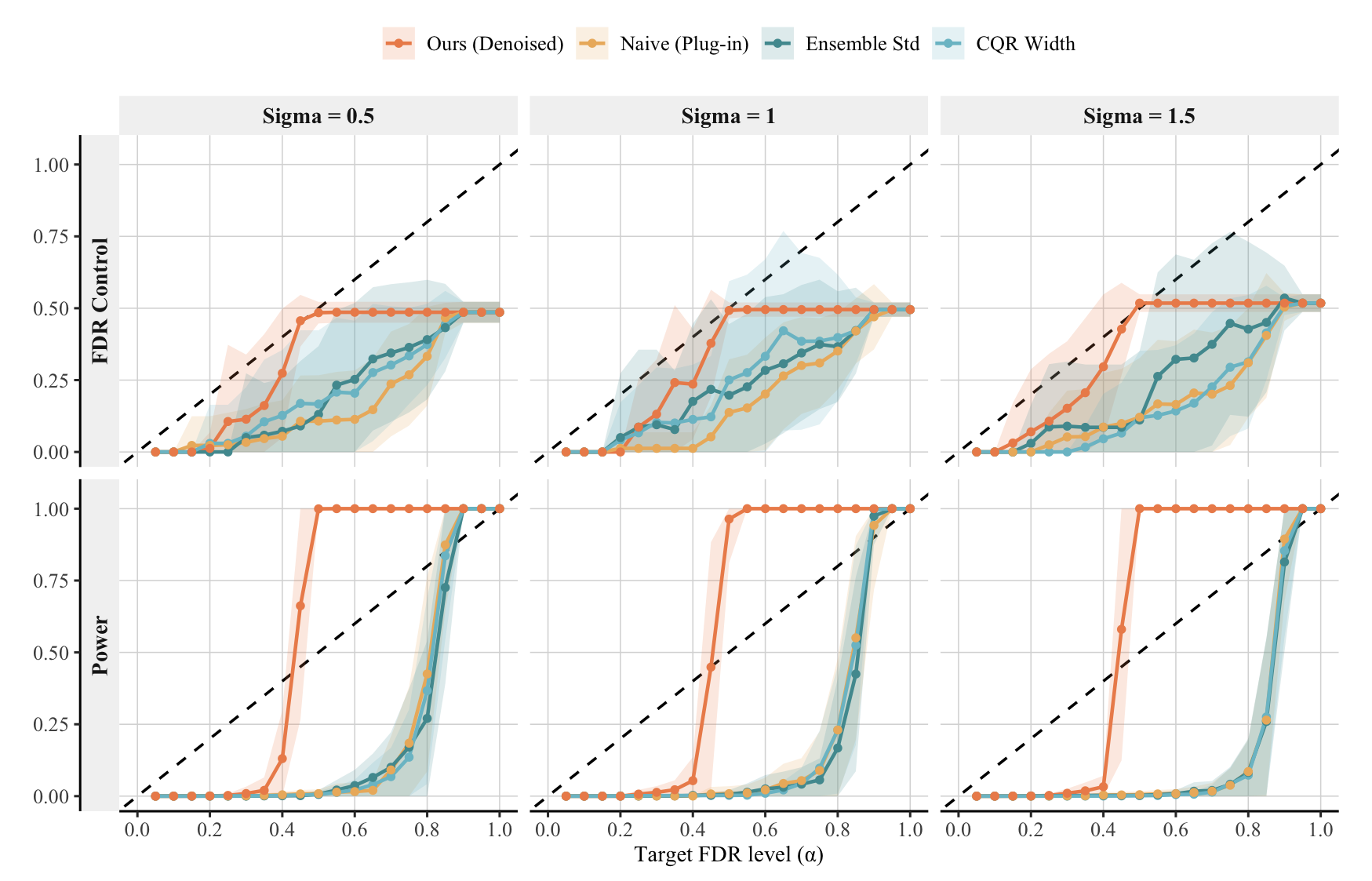}
    \caption{\textbf{Covariate shift with weighted conformal calibration.}
    Under $P_{\mathrm{src}}(X)\neq P_{\mathrm{tgt}}(X)$ with invariant conditional causal mechanism, importance-weighted conformal $p$-values provide a drop-in calibration modification.
    DCA maintains higher selection yield than proxy and uncertainty baselines while the weighted procedure keeps realized FDR stable in this representative shift experiment.}
    \label{fig:main_setting3_shift}
\end{figure}

\textbf{Covariate shift (weighted $p$-values).} We keep the conditional causal mechanism fixed but shift the candidate covariate distribution ($P_{\rm src}(X)\neq P_{\rm tgt}(X)$). Using unlabeled covariates, we compute density ratios and use weighted conformal $p$-values (Eq.~\eqref{eq:weighted_pvalue}). Figure~\ref{fig:main_setting3_shift} shows that the weighted variant remains empirically well-behaved under representative shifts: \textbf{DCA retains its power advantage without sacrificing realized FDR stability}. Full shift sensitivity and the weighted-validity discussion are deferred to Appendix~\ref{app:setting3}.

\textbf{Semi-synthetic and real-data benchmarks.} We evaluate on IHDP and NLSM, where realistic covariate/outcome structures come with simulator-provided ground-truth CATE, and on NSW as a qualitative real-data deployment analysis where true CATE errors are unobserved. Across IHDP and NLSM, \textbf{DCA consistently attains higher selection yield than proxy-based baselines at comparable realized FDR}, suggesting that denoised proxies remain informative beyond fully synthetic designs. On NSW, we report selected-subpopulation and policy-relevant diagnostics rather than oracle FDR/power. Additional comparisons of selection paradigms and downstream policy value are also reported in Appendix~\ref{app:robustness_selection_value}; they show that heuristic Top-$K$ rules can select many units but do not provide the same post-selection error control as conformal/BH procedures.

\section{Discussion and Limitations}
\label{sec:discussion}

Denoised Conformal Alignment (DCA) is a \emph{deployment-first} wrapper: it targets \emph{who to act on} with controlled post-selection risk, rather than uniformly improving CATE estimation. A CATE model can have good average accuracy while still being unreliable on the model-selected tail where actions are taken. DCA addresses this post-selection question by combining denoised causal proxies with conformal calibration and FDR control; variance subtraction prevents DR proxy noise from dominating the ranking in high-noise regimes.

The reliability target is the CATE prediction error $|\hat\tau(X)-\tau(X)|$, not the realized individual treatment effect error. This distinction is necessary because unit-level treatment effects are not point-identifiable under the standard potential-outcomes setup without stronger assumptions. Our validity also relies on stable proxy labeling around the tolerance boundary $c$ and on exchangeability between calibration and candidate pools. Weak nuisance estimation, heavy tails, poor overlap, or overly aggressive denoising can destabilize proxy/oracle threshold agreement, reducing power and affecting finite-sample behavior.

The covariate-shift extension assumes that the conditional causal mechanism remains stable given $X$ and that importance weights are estimable and not too extreme; mechanism shift is outside the current theory. Empirically, semi-synthetic benchmarks provide oracle CATE errors, while purely real data such as NSW can only support qualitative deployment diagnostics. Future work includes more robust proxy construction, sharper overlap and weight conditions, mechanism-shift extensions, principled data reuse, and overlap-aware multi-treatment DCA.


\clearpage

\bibliography{neurips2026}
\bibliographystyle{plainnat}

\appendix

\section{Related Work}
\label{sec:related}

We position our work at the intersection of heterogeneous treatment effect estimation, conformal prediction, and selective inference, bridging the gap between marginal uncertainty quantification and reliable post-selection decision-making.

\textbf{Heterogeneous Treatment Effect Estimation.}
Estimating ITEs from observational data is central to personalized decision-making. Classical approaches rely on matching or reweighting \citep{rosenbaum1983central}, while modern methods leverage machine learning to handle high-dimensional covariates. Prominent examples include tree-based methods like Causal Forests \citep{wager2018estimation}, meta-learners (T/S/X-learner) \citep{kunzel2019metalearners}, and representation learning frameworks \citep{shalit2017estimating}. While these methods optimize for estimation accuracy (e.g., minimizing PEHE), they typically yield point estimates without finite-sample reliability guarantees. Our framework is model-agnostic and wraps around these estimators to operationalize them for safe selective deployment.

\textbf{Conformal Prediction in Causal Inference.}
Conformal prediction (CP) \citep{vovk2005algorithmic} provides distribution-free uncertainty quantification. In the causal domain, recent works have extended CP to construct valid counterfactual intervals \citep{lei2021conformal} and robust predictive inference under misspecification \citep{chernozhukov2021exact}. However, these guarantees are \emph{marginal}: they ensure validity on average over the population but do not guarantee reliability conditional on the selection event. As shown by \citet{geifman2017selective}, selection bias can invalidate marginal guarantees, leading to uncontrolled risk in the subset of individuals chosen for intervention. Our work addresses this limitation by shifting the target from marginal coverage to \emph{post-selection} FDR control.

\textbf{Selective Inference and FDR Control.}
Selective prediction allows models to abstain when uncertainty is high, a concept foundational to reliable machine learning \citep{el2010foundations}. In the context of multiple hypothesis testing, controlling the False Discovery Rate (FDR) is the standard for selecting promising candidates \citep{benjamini1995controlling}. Recent advances have unified CP with FDR control: \citet{candes2023conformalized} introduced conformalized selection for outlier detection, and \citet{jin2023selection} and \citet{gui2024conformal} developed the Conformal Alignment framework.
Related work has further studied data-efficient conformal selection through score optimization, full-data-efficient multiple testing, interactive selection, and derandomized e-value aggregation \citep{bai2024optimized, huo2025unified, gui2025acs, bashari2023derandomized}.
Under covariate shift, \citet{jin2023modelfree} develop weighted conformal $p$-values for selective inference; we use this line of work to motivate our weighted extension, while distinguishing our standard FDR target from mFDR-type threshold feasibility notions \citep{sun2007oracle}.
Crucially, these methods assume ground-truth error labels are observable for calibration. A fundamental barrier in causal inference is that the relevant CATE prediction error depends on unobservable counterfactuals. We resolve this by introducing a denoised proxy mechanism that recovers the ranking signal from noisy doubly robust surrogates, extending conformal selection theory to the counterfactual domain.

\section{Additional Proofs for Section~\ref{sec:theory}}
\label{app:proofs}

\subsection{Notation, sigma-fields, and the conformal selection setup}
\label{app:proofs:setup}

\textbf{Indices and sample splitting.}
Write $n:=|\mathcal D_{\mathrm{cal}}|$ for the calibration size and $m$ for the number of test candidates.
For notational simplicity, index calibration units by $i\in[n]:=\{1,\ldots,n\}$ and test units by $n+j$ with $j\in[m]$.
The (random) data used to fit the base CATE predictor $\hat\tau$ and the alignment score $g$ are collected in an external training split
and are assumed \emph{independent} of the calibration/test candidates.
Let $\mathcal F_{\mathrm{tr}}$ denote the sigma-field generated by this training split and the fitted objects $(\hat\tau,g)$.

\textbf{Truth and reliability.}
For each candidate unit $i$, define the (unobserved) CATE error
\[
A_i \;:=\; \big|\hat\tau(X_i)-\tau(X_i)\big|,
\qquad \tau(x):=\mathbb E[Y(1)-Y(0)\mid X=x].
\]
Fix a tolerance $c>0$. A unit is \emph{good} (reliable) if $A_i<c$ and \emph{bad} (unreliable) if $A_i\ge c$.
For each test index $j\in[m]$, define the null indicator
\[
H_j \;:=\; \mathbbm 1\{A_{n+j}\ge c\}.
\]
In the deployment-first interpretation, a \emph{false discovery} is the selection of a bad unit ($H_j=1$).

\textbf{Score and p-values.}
Let $\hat A_i:=g(X_i)$ be the learned alignment score; in this paper we use the convention that
\emph{smaller} $\hat A_i$ indicates \emph{more} reliable (smaller)  CATE prediction error.
(Everything below is invariant to monotone reparameterizations; if one prefers ``larger is better'',
replace $g$ by $-g$ and reverse inequalities.)

We adopt the denominator $n+1$ (as in \citet{jin2023selection} and \citet{gui2024conformal}), and we assume almost surely no ties:
\begin{equation}
\label{eq:no_ties_assump}
\mathbb P\big(\hat A_i=\hat A_{n+j}\big)=0,\qquad \forall i\in[n],\ \forall j\in[m].
\end{equation}
Under \eqref{eq:no_ties_assump}, using ``$\le$'' vs ``$<$'' in the indicator is immaterial; for definiteness,
the displayed proofs below use strict comparisons. The randomized tie-safe version, which removes
\eqref{eq:no_ties_assump}, is stated after the p-value definitions.

\textbf{Oracle and proxy p-values (global denominator).}
In the \emph{oracle} setting where the calibration errors $A_i$ are observable, define
\begin{equation}
\label{eq:oracle_pvals_app}
p_j^\ast
\;:=\;
\frac{1+\sum_{i=1}^n \mathbbm 1\{A_i\ge c,\ \hat A_i < \hat A_{n+j}\}}
{n+1},
\qquad j\in[m].
\end{equation}
In the causal setting where $A_i$ is unobserved, we form a denoised proxy $\check A_i$ on the calibration set
and define the proxy p-values
\begin{equation}
\label{eq:proxy_pvals_app}
\tilde p_j
\;:=\;
\frac{1+\sum_{i=1}^n \mathbbm 1\{\check A_i\ge c,\ \hat A_i < \hat A_{n+j}\}}
{n+1},
\qquad j\in[m].
\end{equation}

\paragraph{Randomized tie-breaking (tie-safe implementation).}
\label{app:randomized_ties}
The no-ties assumption in \eqref{eq:no_ties_assump} is used only to keep the main proof notation clean.
It is not a substantive restriction on the method. When truncation at zero or a discrete learner creates ties,
we attach i.i.d.\ auxiliary variables
\[
U_1,\ldots,U_{n+m}\stackrel{\mathrm{i.i.d.}}{\sim}\mathrm{Unif}(0,1),
\]
independent of all data and training randomness, and compare scores lexicographically:
\[
(a,u)\prec_{\mathrm{lex}}(b,v)
\quad\Longleftrightarrow\quad
a<b\ \text{or}\ (a=b\ \text{and}\ u<v).
\]
Write $\preceq_{\mathrm{lex}}$ for the associated non-strict order.
Thus every comparison $\hat A_i<\hat A_{n+j}$ in \eqref{eq:oracle_pvals_app}--\eqref{eq:proxy_pvals_app}
is replaced by
\[
(\hat A_i,U_i)\prec_{\mathrm{lex}}(\hat A_{n+j},U_{n+j}).
\]
The corresponding tie-safe oracle and proxy p-values are
\begin{align}
\label{eq:oracle_pvals_rand_app}
p_{j,\mathrm{rand}}^\ast
&:=
\frac{1+\sum_{i=1}^n \mathbbm 1\{A_i\ge c,\,
(\hat A_i,U_i)\prec_{\mathrm{lex}}(\hat A_{n+j},U_{n+j})\}}
{n+1},\\
\label{eq:proxy_pvals_rand_app}
\tilde p_{j,\mathrm{rand}}
&:=
\frac{1+\sum_{i=1}^n \mathbbm 1\{\check A_i\ge c,\,
(\hat A_i,U_i)\prec_{\mathrm{lex}}(\hat A_{n+j},U_{n+j})\}}
{n+1}.
\end{align}
Under \eqref{eq:no_ties_assump}, these randomized p-values coincide almost surely with
\eqref{eq:oracle_pvals_app}--\eqref{eq:proxy_pvals_app}. With ties, they are the exact-rank version:
deterministically counting all tied calibration scores using ``$\le$'' is conservative, while randomizing within tied blocks
restores the uniform-rank argument. Lemma~\ref{lem:randomized_tie_validity} formalizes this claim after the auxiliary-score representation below.
The experimental tie-handling check in Appendix~\ref{app:robustness_ties}
(Figure~\ref{fig:tie_breaking}) reports that the deterministic, jittered, and smoothed variants give nearly identical curves in our simulations.

\textbf{BH selection, FDP, and FDR.}
Given p-values $q_1,\ldots,q_m\in[0,1]$, the (standard) BH procedure at level $\alpha$ is the step-up rule:
let $q_{(1)}\le\cdots\le q_{(m)}$ be the order statistics and set
\[
\hat k \;:=\; \max\Big\{k\in[m]: q_{(k)}\le \alpha k/m\Big\},
\qquad \hat k:=0\ \text{if the set is empty},
\]
and reject (select) $\mathcal S=\{j: q_j\le \alpha \hat k/m\}$.
Write $\mathrm{BH}(q_{1:m};\alpha)$ for this rejection set.
For any (random) selection set $\mathcal S\subseteq[m]$, define
\[
\mathrm{FDP}(\mathcal S)
\;:=\;
\frac{\sum_{j=1}^m H_j\mathbbm 1\{j\in\mathcal S\}}{|\mathcal S|\vee 1},
\qquad
\mathrm{FDR}(\mathcal S)
\;:=\;
\mathbb E[\mathrm{FDP}(\mathcal S)].
\]
(Our power definition in the main text will be treated in Appendix~\ref{app:proofs:power} after establishing cutoff consistency.)

\textbf{Exchangeability.}
We will use the following conditional exchangeability condition, which is implied by i.i.d.\ calibration/test candidates
conditional on $\mathcal F_{\mathrm{tr}}$ (Assumption~\ref{assum:exchangeability} in the main text).

\begin{assumption}[Conditional exchangeability for conformal selection]
\label{assum:cond_exch_app}
Conditional on $\mathcal F_{\mathrm{tr}}$, the candidate pairs $\{(X_i,A_i)\}_{i=1}^{n+m}$ are i.i.d.
Equivalently, for each $j\in[m]$, the multiset $\{(X_i,A_i)\}_{i=1}^{n}\cup\{(X_{n+j},A_{n+j})\}$ is exchangeable
conditional on $\{X_{n+\ell}\}_{\ell\neq j}$ and $\mathcal F_{\mathrm{tr}}$.
\end{assumption}

\subsection{A key representation: oracle p-values are conformal p-values for an auxiliary score}
\label{app:proofs:aux_score}

The denominator choice $n+1$ in \eqref{eq:oracle_pvals_app} is not arbitrary: it is the canonical normalization of
\citet{jin2023selection} for \emph{selection by prediction} and is exactly the one used in
\citet{gui2024conformal} (Conformal Alignment).
The analysis becomes transparent after rewriting \eqref{eq:oracle_pvals_app} as a standard conformal p-value
for a carefully constructed auxiliary score.

\textbf{Boundedness (w.l.o.g.).}
Fix $\bar M>0$ such that $\hat A=g(X)\in[0,\bar M]$ almost surely.\footnote{
If $g$ is unbounded, one may apply a strictly increasing squashing map (e.g.\ sigmoid) to $g$
without changing rankings and then rescale to $[0,\bar M]$. All proofs remain unchanged since only order comparisons are used.
}
Define the auxiliary map
\begin{equation}
\label{eq:V_def_app}
V(x,a)
\;:=\;
2\bar M\cdot \mathbbm 1\{a<c\} \;+\; g(x),
\qquad
\hat V(x)
\;:=\;
V(x,c)=g(x).
\end{equation}
Note that $V(x,a)$ is \emph{monotone non-increasing} in $a$: if $a_1\ge a_2$, then $V(x,a_1)\le V(x,a_2)$.

For each unit $i$, write $V_i:=V(X_i,A_i)$ and $\hat V_i:=\hat V(X_i)=g(X_i)=\hat A_i$.

\begin{lemma}[Oracle p-values as conformal p-values]
\label{lem:pval_as_conformal}
Under \eqref{eq:no_ties_assump}, for every $j\in[m]$,
\begin{equation}
\label{eq:pval_conformal_form}
p_j^\ast
\;=\;
\frac{1+\sum_{i=1}^n \mathbbm 1\{V_i<\hat V_{n+j}\}}{n+1}.
\end{equation}
Moreover, on the null event $\{A_{n+j}\ge c\}$ we have $V_{n+j}\le \hat V_{n+j}$ and in fact $V_{n+j}=\hat V_{n+j}$
for the particular choice \eqref{eq:V_def_app}.
\end{lemma}

\begin{proof}
Fix $j\in[m]$.
By construction, for any calibration index $i\in[n]$:
\[
V_i<\hat V_{n+j}
\iff
2\bar M\cdot \mathbbm 1\{A_i<c\}+g(X_i) < g(X_{n+j}).
\]
Since $\hat V_{n+j}=g(X_{n+j})\in[0,\bar M]$, if $A_i<c$ then $V_i\ge 2\bar M> \bar M\ge \hat V_{n+j}$, so $V_i<\hat V_{n+j}$ is false.
Thus the event $V_i<\hat V_{n+j}$ can hold only when $A_i\ge c$, in which case $V_i=g(X_i)=\hat A_i$.
Therefore
\[
\mathbbm 1\{V_i<\hat V_{n+j}\}
=
\mathbbm 1\{A_i\ge c,\ \hat A_i<\hat A_{n+j}\},
\]
and substituting into \eqref{eq:oracle_pvals_app} yields \eqref{eq:pval_conformal_form}.

Finally, on $\{A_{n+j}\ge c\}$, the monotonicity in $a$ gives $V(X_{n+j},A_{n+j})\le V(X_{n+j},c)=\hat V_{n+j}$.
For \eqref{eq:V_def_app}, if $A_{n+j}\ge c$ then $\mathbbm 1\{A_{n+j}<c\}=0$, hence $V_{n+j}=g(X_{n+j})=\hat V_{n+j}$.
\end{proof}

\textbf{A notion of validity.}
Because \eqref{eq:oracle_pvals_app} uses a \emph{global} denominator $n+1$, $p_j^\ast$ is not (and need not be) a classical p-value
in the conditional sense $\mathbb P(p_j^\ast\le t\mid A_{n+j}\ge c)\le t$.
Instead, the relevant validity notion for BH-FDR in \citet{jin2023selection} is the \emph{joint} (unconditional) bound
\[
\mathbb P\big(p_j^\ast\le t,\ A_{n+j}\ge c\big)\le t,\qquad \forall t\in[0,1],
\]
which follows from the conformal representation above and monotonicity; see Lemma~\ref{lem:joint_validity_oracle} below.
This is the correct mathematical framework behind Lemma~\ref{lemma:oracle} in the main text and matches
the proof template in Appendix~B.3 of \citet{jin2023selection}.

\begin{lemma}[Joint validity of oracle p-values]
\label{lem:joint_validity_oracle}
Assume Assumption~\ref{assum:cond_exch_app} and \eqref{eq:no_ties_assump}.
Then for every $j\in[m]$ and every $t\in[0,1]$,
\begin{equation}
\label{eq:joint_validity_oracle}
\mathbb P\big(p_j^\ast\le t,\ A_{n+j}\ge c \,\big|\, \mathcal F_{\mathrm{tr}}\big)
\;\le\; t
\qquad \text{a.s.}
\end{equation}
\end{lemma}

\begin{proof}
Fix $j\in[m]$ and condition on $\mathcal F_{\mathrm{tr}}$ throughout.
Define the \emph{fully-oracle} conformal p-value
\begin{equation}
\label{eq:fully_oracle_p}
p_j^{\circ}
\;:=\;
\frac{1+\sum_{i=1}^n \mathbbm 1\{V_i<V_{n+j}\}}{n+1},
\end{equation}
which is generally uncomputable because it uses $V_{n+j}=V(X_{n+j},A_{n+j})$.
By Lemma~\ref{lem:pval_as_conformal}, on the null event $\{A_{n+j}\ge c\}$ we have $V_{n+j}\le \hat V_{n+j}$.
Since the map $v\mapsto 1+\sum_{i=1}^n\mathbbm 1\{V_i<v\}$ is non-decreasing in $v$, it follows that
\[
\{A_{n+j}\ge c\}\ \implies\ p_j^{\circ}\le p_j^\ast.
\]
Hence
\[
\{p_j^\ast\le t,\ A_{n+j}\ge c\}
\subseteq
\{p_j^{\circ}\le t\}.
\]
Therefore it suffices to show $\mathbb P(p_j^{\circ}\le t\mid \mathcal F_{\mathrm{tr}})\le t$.

Under Assumption~\ref{assum:cond_exch_app}, conditional on $\mathcal F_{\mathrm{tr}}$ the $(n+1)$-tuple
$(V_1,\ldots,V_n,V_{n+j})$ is exchangeable (indeed i.i.d.).
Under the no-ties condition \eqref{eq:no_ties_assump} (which implies in particular $\mathbb P(V_i=V_{n+j})=0$),
the rank
\[
R_{n+j}:=1+\sum_{i=1}^n \mathbbm 1\{V_i<V_{n+j}\}
\]
is uniform on $\{1,2,\ldots,n+1\}$.
Thus $p_j^{\circ}=R_{n+j}/(n+1)$ is discrete-uniform on $\{1/(n+1),\ldots,1\}$, which implies
$\mathbb P(p_j^{\circ}\le t\mid\mathcal F_{\mathrm{tr}})\le t$ for all $t\in[0,1]$.
Combining with the inclusion above yields \eqref{eq:joint_validity_oracle}.
\end{proof}

\begin{lemma}[Randomized ranks remove the no-ties restriction]
\label{lem:randomized_tie_validity}
Assume Assumption~\ref{assum:cond_exch_app}, but not \eqref{eq:no_ties_assump}.
Use the randomized p-values \eqref{eq:oracle_pvals_rand_app}--\eqref{eq:proxy_pvals_rand_app}.
Then, for every $j\in[m]$ and $t\in[0,1]$,
\[
\mathbb P\big(p_{j,\mathrm{rand}}^\ast\le t,\ A_{n+j}\ge c \,\big|\, \mathcal F_{\mathrm{tr}}\big)
\le t
\qquad \text{a.s.}
\]
Moreover, the oracle BH guarantee and the proxy perturbation guarantee continue to hold with randomized p-values:
\[
\mathrm{FDR}\!\left(\mathrm{BH}(p_{1:m,\mathrm{rand}}^\ast;\alpha)\right)\le \alpha,
\qquad
\mathrm{FDR}\!\left(\mathrm{BH}(\tilde p_{1:m,\mathrm{rand}};\alpha)\right)
\le \alpha+m\,\mathbb E[\widehat\Delta_{\mathrm{cal}}].
\]
\end{lemma}

\begin{proof}
Define the randomized fully-oracle p-value
\[
p_{j,\mathrm{rand}}^\circ
:=
\frac{1+\sum_{i=1}^n
\mathbbm 1\{(V_i,U_i)\prec_{\mathrm{lex}}(V_{n+j},U_{n+j})\}}
{n+1}.
\]
Conditional on $\mathcal F_{\mathrm{tr}}$, the augmented pairs
$\{(V_i,U_i)\}_{i=1}^{n}\cup\{(V_{n+j},U_{n+j})\}$ are exchangeable and are almost surely distinct because of the independent continuous uniforms.
Hence the randomized rank of $(V_{n+j},U_{n+j})$ is uniform on $\{1,\ldots,n+1\}$, which implies
\[
\mathbb P(p_{j,\mathrm{rand}}^\circ\le t\mid \mathcal F_{\mathrm{tr}})\le t.
\]
On the null event $\{A_{n+j}\ge c\}$, Lemma~\ref{lem:pval_as_conformal}'s monotonicity argument gives
$V_{n+j}\le \hat V_{n+j}$, and the same auxiliary uniform $U_{n+j}$ is used on both sides. Therefore
\[
(V_{n+j},U_{n+j})\preceq_{\mathrm{lex}}(\hat V_{n+j},U_{n+j}),
\]
so
\[
\{p_{j,\mathrm{rand}}^\ast\le t,\ A_{n+j}\ge c\}
\subseteq
\{p_{j,\mathrm{rand}}^\circ\le t\}.
\]
This proves the joint validity inequality.

The proof of Lemma~\ref{lemma:oracle} uses only exchangeability of the conformal ranks, monotonicity of the p-values in the score,
and the fact that the score order is total. Lexicographic ordering supplies such a total order with no ties almost surely, so the
leave-one-out argument applies verbatim after replacing scalar comparisons by $\prec_{\mathrm{lex}}$.
Finally, for every $j$, replacing the oracle labels $\mathbbm 1\{A_i\ge c\}$ by proxy labels
$\mathbbm 1\{\check A_i\ge c\}$ can change at most one summand per mislabeled calibration point, exactly as in
Proposition~\ref{prop:finite_sample_approx_fdr}; hence
$\sup_j|\tilde p_{j,\mathrm{rand}}-p_{j,\mathrm{rand}}^\ast|\le \widehat\Delta_{\mathrm{cal}}$
and the same perturbation bound follows.
\end{proof}

\subsection{Proof of Lemma~\ref{lemma:oracle}: finite-sample oracle FDR control}
\label{app:proofs:oracle_fdr}

\begin{proof}
This proof is a careful adaptation of Appendix~B.3 of \citet{jin2023selection} (Theorem~2.6 therein) to our notation.
We reproduce the argument in full detail for completeness.

\textbf{Step 0: BH as a functional of p-values.}
For any vector $a=(a_1,\ldots,a_m)\in[0,1]^m$, write
\[
\mathcal R(a)\subseteq[m]
\]
for the BH rejection set at level $\alpha$ computed from $a$.
Let $\mathcal R:=\mathcal R(p^\ast)$ denote the oracle rejection set, where $p^\ast=(p_1^\ast,\ldots,p_m^\ast)$.
Write $R_j:=\mathbbm 1\{j\in\mathcal R\}$ and $|\mathcal R|=\sum_{j=1}^m R_j$.

\textbf{Step 1: A leave-one-out comparison construction.}
Fix $j\in[m]$.
Define the fully-oracle p-value $p_j^{\circ}$ as in \eqref{eq:fully_oracle_p}.
For each $\ell\neq j$, define the ``$j$-augmented'' p-values
\begin{equation}
\label{eq:pell_j_aug}
p^{(j)}_{\ell}
\;:=\;
\frac{\sum_{i=1}^n \mathbbm 1\{V_i<\hat V_{n+\ell}\} \;+\; \mathbbm 1\{V_{n+j}<\hat V_{n+\ell}\}}
{n+1}.
\end{equation}
Note that these p-values are \emph{only} for analysis: they depend on $V_{n+j}$ and are generally uncomputable.

Define the modified p-value vector
\[
p^{(j)}
\;:=\;
\big(p^{(j)}_1,\ldots,p^{(j)}_{j-1},p_j^{\circ},p^{(j)}_{j+1},\ldots,p^{(j)}_m\big),
\]
and let
\[
\mathcal R_j^{\ast}:=\mathcal R(p^{(j)})
\]
be the BH rejection set computed from this modified vector.

\textbf{Step 2: A ``replacement invariance'' claim.}
We claim that on the event $\{A_{n+j}\ge c,\ j\in\mathcal R\}$,
\begin{equation}
\label{eq:replacement_invariance}
\mathcal R=\mathcal R_j^{\ast}.
\end{equation}
This claim is the key combinatorial step.

\begin{proof}[Proof of \eqref{eq:replacement_invariance}]
Work on the event $\{A_{n+j}\ge c,\ j\in\mathcal R\}$.
By Lemma~\ref{lem:pval_as_conformal}, $A_{n+j}\ge c$ implies $V_{n+j}\le \hat V_{n+j}$.
Since the mapping $v\mapsto \sum_{i=1}^n \mathbbm 1\{V_i<v\}$ is non-decreasing, we have
\begin{equation}
\label{eq:pj_oracle_le}
p_j^{\circ}
=
\frac{1+\sum_{i=1}^n \mathbbm 1\{V_i<V_{n+j}\}}{n+1}
\;\le\;
\frac{1+\sum_{i=1}^n \mathbbm 1\{V_i<\hat V_{n+j}\}}{n+1}
=
p_j^\ast.
\end{equation}

Now fix any $\ell\neq j$. We compare $p_\ell^\ast$ to $p_\ell^{(j)}$.
By \eqref{eq:oracle_pvals_app} and Lemma~\ref{lem:pval_as_conformal},
\[
p_\ell^\ast=\frac{1+\sum_{i=1}^n \mathbbm 1\{V_i<\hat V_{n+\ell}\}}{n+1}.
\]
By \eqref{eq:pell_j_aug},
\[
p_\ell^{(j)}=\frac{\sum_{i=1}^n \mathbbm 1\{V_i<\hat V_{n+\ell}\}+\mathbbm 1\{V_{n+j}<\hat V_{n+\ell}\}}{n+1}.
\]
Because there are no ties, exactly one of the relations $\hat V_{n+\ell}<\hat V_{n+j}$ or $\hat V_{n+\ell}>\hat V_{n+j}$ holds.

\emph{Case (i): $\hat V_{n+\ell}>\hat V_{n+j}$.}
On $\{A_{n+j}\ge c\}$ we have $V_{n+j}\le \hat V_{n+j}<\hat V_{n+\ell}$, hence $\mathbbm 1\{V_{n+j}<\hat V_{n+\ell}\}=1$.
Therefore
\[
p_\ell^{(j)}
=
\frac{\sum_{i=1}^n \mathbbm 1\{V_i<\hat V_{n+\ell}\}+1}{n+1}
=
p_\ell^\ast.
\]

\emph{Case (ii): $\hat V_{n+\ell}<\hat V_{n+j}$.}
Then, because $j\in\mathcal R$ and BH is a step-up procedure, every p-value no larger than $p_j^\ast$ is also rejected.
Indeed, if $j\in\mathcal R$, then by definition of BH there exists a data-dependent threshold $T_{\mathrm{BH}}=\alpha |\mathcal R|/m$
such that $p_j^\ast\le T_{\mathrm{BH}}$ and $\mathcal R=\{r: p_r^\ast\le T_{\mathrm{BH}}\}$.
Since $p_\ell^\ast$ is non-decreasing in $\hat V_{n+\ell}$ and $\hat V_{n+\ell}<\hat V_{n+j}$, it follows that $p_\ell^\ast\le p_j^\ast\le T_{\mathrm{BH}}$,
so $\ell\in\mathcal R$.

Moreover, in this case $p_\ell^{(j)}\le p_\ell^\ast$ always, since the extra indicator in \eqref{eq:pell_j_aug} is either $0$ or $1$ and
\[
p_\ell^{(j)} \in \Big\{p_\ell^\ast-\tfrac{1}{n+1},\ p_\ell^\ast\Big\}\le p_\ell^\ast\le p_j^\ast.
\]
Thus, after the replacement, every p-value $\le p_j^\ast$ remains $\le p_j^\ast$.

Putting the two cases together:
\begin{itemize}
\item all p-values \emph{strictly larger} than $p_j^\ast$ are unchanged (Case (i));
\item all p-values \emph{no larger} than $p_j^\ast$ (including $p_j^\ast$ itself) remain $\le p_j^\ast$ after replacement (Case (ii) and \eqref{eq:pj_oracle_le}).
\end{itemize}
Because BH is a step-up procedure whose rejection set depends only on how many p-values fall below each candidate threshold,
this implies that the BH rejection set is unchanged by the replacement, proving \eqref{eq:replacement_invariance}.
\end{proof}

\textbf{Step 3: Leave-one-out expansion of the FDR.}
By definition,
\[
\mathrm{FDR}(\mathcal R)
=
\mathbb E\!\left[
\frac{\sum_{j=1}^m H_j R_j}{|\mathcal R|\vee 1}
\right]
=
\sum_{j=1}^m
\mathbb E\!\left[
\frac{H_j R_j}{|\mathcal R|\vee 1}
\right].
\]
Next, expand according to the value of $|\mathcal R|$:
\begin{equation}
\label{eq:fdr_expand_k}
\mathrm{FDR}(\mathcal R)
=
\sum_{j=1}^m\sum_{k=1}^m \frac{1}{k}\,
\mathbb E\!\left[
\mathbbm 1\{|\mathcal R|=k\}\, H_j\,\mathbbm 1\{j\in\mathcal R\}
\right].
\end{equation}

If $|\mathcal R|=k$ and $j\in\mathcal R$, then by BH step-up we necessarily have
\begin{equation}
\label{eq:bh_threshold_implication}
p_j^\ast \le \alpha k/m.
\end{equation}
Therefore
\[
\mathbbm 1\{|\mathcal R|=k\}\,H_j\,\mathbbm 1\{j\in\mathcal R\}
\le
\mathbbm 1\{|\mathcal R|=k\}\,H_j\,\mathbbm 1\{p_j^\ast\le \alpha k/m\}\,\mathbbm 1\{j\in\mathcal R\}.
\]
Substituting into \eqref{eq:fdr_expand_k} gives
\begin{equation}
\label{eq:fdr_expand_k2}
\mathrm{FDR}(\mathcal R)
\le
\sum_{j=1}^m\sum_{k=1}^m \frac{1}{k}\,
\mathbb E\!\left[
\mathbbm 1\{|\mathcal R|=k\}\, H_j\,\mathbbm 1\{p_j^\ast\le \alpha k/m\}\,\mathbbm 1\{j\in\mathcal R\}
\right].
\end{equation}

\textbf{Step 4: Replace $\mathcal R$ by $\mathcal R_j^\ast$ on the null event.}
On the event $\{H_j=1,\ j\in\mathcal R\}$ we have \eqref{eq:replacement_invariance}, hence
$|\mathcal R|=|\mathcal R_j^\ast|$ and $p_j^\ast\ge p_j^{\circ}$ by \eqref{eq:pj_oracle_le}.
Thus the integrand in \eqref{eq:fdr_expand_k2} is bounded by
\[
\mathbbm 1\{|\mathcal R_j^\ast|=k\}\,\mathbbm 1\{p_j^{\circ}\le \alpha k/m\}.
\]
Dropping the factor $H_j\le 1$ yields
\begin{equation}
\label{eq:fdr_after_replace}
\mathrm{FDR}(\mathcal R)
\le
\sum_{j=1}^m\sum_{k=1}^m \frac{1}{k}\,
\mathbb E\!\left[
\mathbbm 1\{|\mathcal R_j^\ast|=k\}\,\mathbbm 1\{p_j^{\circ}\le \alpha k/m\}
\right].
\end{equation}

\textbf{Step 5: Convert the event $\{p_j^{\circ}\le \alpha k/m\}$ into membership in $\mathcal R_j^\ast$.}
On the event $\{|\mathcal R_j^\ast|=k\}$, BH rejects exactly those $\ell$ with $p_\ell^{(j)}\le \alpha k/m$ and the $j$-th p-value is $p_j^\circ$.
Hence
\[
\mathbbm 1\{|\mathcal R_j^\ast|=k\}\,\mathbbm 1\{p_j^\circ\le \alpha k/m\}
=
\mathbbm 1\{|\mathcal R_j^\ast|=k\}\,\mathbbm 1\{p_j^\circ\in \mathcal R_j^\ast\}.
\]
Substitute into \eqref{eq:fdr_after_replace} to obtain
\begin{equation}
\label{eq:fdr_membership}
\mathrm{FDR}(\mathcal R)
\le
\sum_{j=1}^m\sum_{k=1}^m \frac{1}{k}\,
\mathbb E\!\left[
\mathbbm 1\{|\mathcal R_j^\ast|=k\}\,\mathbbm 1\{p_j^{\circ}\in \mathcal R_j^\ast\}
\right].
\end{equation}

\textbf{Step 6: A second step-up invariance and a conditional super-uniform bound.}
If $p_j^\circ\in\mathcal R_j^\ast$, then by step-up monotonicity, setting the $j$-th p-value to $0$ cannot reduce the rejection set.
Define
\[
\mathcal R_{j\to 0}^\ast
\;:=\;
\mathcal R\big(p_1^{(j)},\ldots,p_{j-1}^{(j)},0,p_{j+1}^{(j)},\ldots,p_m^{(j)}\big).
\]
Then on $\{p_j^\circ\in\mathcal R_j^\ast\}$ we have $\mathcal R_j^\ast=\mathcal R_{j\to 0}^\ast$.
Therefore
\[
\mathbbm 1\{|\mathcal R_j^\ast|=k\}\,\mathbbm 1\{p_j^\circ\in\mathcal R_j^\ast\}
\le
\mathbbm 1\{|\mathcal R_{j\to 0}^\ast|=k\}\,\mathbbm 1\{p_j^\circ\in\mathcal R_{j\to 0}^\ast\}.
\]
Summing over $k$ in \eqref{eq:fdr_membership} yields
\begin{equation}
\label{eq:fdr_pre_final}
\mathrm{FDR}(\mathcal R)
\le
\sum_{j=1}^m
\mathbb E\!\left[
\frac{\mathbbm 1\{p_j^\circ\in\mathcal R_{j\to 0}^\ast\}}{|\mathcal R_{j\to 0}^\ast|\vee 1}
\right].
\end{equation}

By BH, the event $\{p_j^\circ\in\mathcal R_{j\to 0}^\ast\}$ is equivalent to
\[
p_j^\circ \le \alpha |\mathcal R_{j\to 0}^\ast|/m.
\]
Thus
\begin{equation}
\label{eq:fdr_pre_final2}
\mathrm{FDR}(\mathcal R)
\le
\sum_{j=1}^m
\mathbb E\!\left[
\frac{\mathbbm 1\{p_j^\circ \le \alpha |\mathcal R_{j\to 0}^\ast|/m\}}{|\mathcal R_{j\to 0}^\ast|\vee 1}
\right].
\end{equation}

\textbf{Step 7: Conditional independence given the unordered multiset and the final bound.}
Fix $j$ and condition on $\mathcal F_{\mathrm{tr}}$.
Let $[V_1,\ldots,V_n,V_{n+j}]$ denote the unordered multiset (equivalently, the sigma-field generated by all symmetric functions of these values).
Observe:
\begin{itemize}
\item $p_j^\circ$ depends on the ordered vector $(V_1,\ldots,V_n,V_{n+j})$ only through the rank of $V_{n+j}$ among them,
hence is measurable with respect to $[V_1,\ldots,V_n,V_{n+j}]$.
\item For $\ell\neq j$, the augmented p-values $p_\ell^{(j)}$ in \eqref{eq:pell_j_aug} depend on $(V_1,\ldots,V_n,V_{n+j})$
only through the unordered multiset $[V_1,\ldots,V_n,V_{n+j}]$ because the defining expression is invariant under permutations of these $(n+1)$ values.
Consequently, $\mathcal R_{j\to 0}^\ast$ and hence $|\mathcal R_{j\to 0}^\ast|$ are measurable with respect to
the sigma-field generated by $\{p_\ell^{(j)}\}_{\ell\neq j}$ and thus measurable with respect to $[V_1,\ldots,V_n,V_{n+j}]$ and $\{\hat V_{n+\ell}\}_{\ell\neq j}$.
\end{itemize}

Moreover, under Assumption~\ref{assum:cond_exch_app}, the multiset $[V_1,\ldots,V_n,V_{n+j}]$ is exchangeable (indeed i.i.d.),
so the rank of $V_{n+j}$ among these $n+1$ values is uniform on $\{1,\ldots,n+1\}$ conditionally on the unordered multiset.
This implies the following conditional super-uniform statement:
for any random variable $T\in[0,1]$ measurable with respect to $[V_1,\ldots,V_n,V_{n+j}]$,
\begin{equation}
\label{eq:cond_superunif_rank}
\mathbb P\!\left(p_j^\circ \le T \,\middle|\, [V_1,\ldots,V_n,V_{n+j}],\ \mathcal F_{\mathrm{tr}}\right) \le T.
\end{equation}
(Indeed, conditional on the multiset, $p_j^\circ$ is discrete-uniform on $\{1/(n+1),\ldots,1\}$.)

Apply \eqref{eq:cond_superunif_rank} with
$T=\alpha |\mathcal R_{j\to 0}^\ast|/m$, which is measurable with respect to $[V_1,\ldots,V_n,V_{n+j}]$ and $\mathcal F_{\mathrm{tr}}$.
Then taking conditional expectations in \eqref{eq:fdr_pre_final2} yields
\[
\mathbb E\!\left[
\frac{\mathbbm 1\{p_j^\circ \le \alpha |\mathcal R_{j\to 0}^\ast|/m\}}{|\mathcal R_{j\to 0}^\ast|\vee 1}
\ \middle|\ [V_1,\ldots,V_n,V_{n+j}],\ \mathcal F_{\mathrm{tr}}
\right]
\le
\frac{\alpha |\mathcal R_{j\to 0}^\ast|/m}{|\mathcal R_{j\to 0}^\ast|\vee 1}
\le
\frac{\alpha}{m}.
\]
Taking expectations and summing over $j\in[m]$ in \eqref{eq:fdr_pre_final2} gives
\[
\mathrm{FDR}(\mathcal R)\le \sum_{j=1}^m \frac{\alpha}{m}=\alpha.
\]
This proves Lemma~\ref{lemma:oracle}.
\end{proof}

\subsection{Proof of Proposition~\ref{prop:finite_sample_approx_fdr}: perturbation by proxy mislabeling}
\label{app:proofs:proxy_perturb}

Recall the proxy p-values $\tilde p_j$ in \eqref{eq:proxy_pvals_app}.
Define the calibration mislabeling rate (as in \eqref{eq:mislabel_rate}):
\[
\widehat\Delta_{\mathrm{cal}}
=
\frac{1}{n}\sum_{i=1}^n
\left|\mathbbm 1\{\check A_i\ge c\}-\mathbbm 1\{A_i\ge c\}\right|.
\]
Let $M_{\mathrm{cal}}:=n\,\widehat\Delta_{\mathrm{cal}}$ denote the number of mislabeled calibration units.

\begin{proof}[Proof of Proposition~\ref{prop:finite_sample_approx_fdr}]
Fix $j\in[m]$. Let
\[
U_j^\ast := \sum_{i=1}^n \mathbbm 1\{A_i\ge c,\ \hat A_i<\hat A_{n+j}\},
\qquad
\tilde U_j := \sum_{i=1}^n \mathbbm 1\{\check A_i\ge c,\ \hat A_i<\hat A_{n+j}\}.
\]
Then $p_j^\ast=(1+U_j^\ast)/(n+1)$ and $\tilde p_j=(1+\tilde U_j)/(n+1)$.
Hence
\begin{equation}
\label{eq:p_diff_basic}
|\tilde p_j-p_j^\ast|
=
\frac{|\,\tilde U_j-U_j^\ast\,|}{n+1}.
\end{equation}
For each $i\in[n]$,
\[
\left|
\mathbbm 1\{\check A_i\ge c,\ \hat A_i<\hat A_{n+j}\}
-
\mathbbm 1\{A_i\ge c,\ \hat A_i<\hat A_{n+j}\}
\right|
\le
\left|\mathbbm 1\{\check A_i\ge c\}-\mathbbm 1\{A_i\ge c\}\right|.
\]
Summing over $i$ yields
\[
|\,\tilde U_j-U_j^\ast\,|
\le
\sum_{i=1}^n
\left|\mathbbm 1\{\check A_i\ge c\}-\mathbbm 1\{A_i\ge c\}\right|
=
M_{\mathrm{cal}}.
\]
Plug into \eqref{eq:p_diff_basic}:
\[
|\tilde p_j-p_j^\ast|
\le
\frac{M_{\mathrm{cal}}}{n+1}
=
\frac{n}{n+1}\widehat\Delta_{\mathrm{cal}}
\le
\widehat\Delta_{\mathrm{cal}}.
\]
This proves the first claim of Proposition~\ref{prop:finite_sample_approx_fdr}.

\medskip
\noindent
\textbf{A correct ``approximate validity'' inequality (joint form).}
By Lemma~\ref{lem:joint_validity_oracle} and the bound above, for any $u\in[0,1]$ and any $j$,
\[
\{ \tilde p_j\le u,\ A_{n+j}\ge c\}
\subseteq
\{ p_j^\ast \le u+\widehat\Delta_{\mathrm{cal}},\ A_{n+j}\ge c\},
\]
hence
\[
\mathbb P\big(\tilde p_j\le u,\ A_{n+j}\ge c \,\big|\, \mathcal F_{\mathrm{tr}}\big)
\le
\mathbb P\big(p_j^\ast\le u+\widehat\Delta_{\mathrm{cal}},\ A_{n+j}\ge c \,\big|\, \mathcal F_{\mathrm{tr}}\big)
\le
u+\mathbb E[\widehat\Delta_{\mathrm{cal}}\mid \mathcal F_{\mathrm{tr}}],
\]
where the last step uses \eqref{eq:joint_validity_oracle} and the tower property.

\medskip
\noindent
\textbf{Approximate FDR control.}
Repeat the proof of Lemma~\ref{lemma:oracle} with $p_j^\ast$ replaced by $\tilde p_j$, and note that the only step
that uses exact conditional super-uniformity is \eqref{eq:cond_superunif_rank}. Under the perturbation bound
$|\tilde p_j-p_j^\ast|\le \widehat\Delta_{\mathrm{cal}}$, this step becomes
\[
\mathbb P\!\left(\tilde p_j \le T \,\middle|\, [V_1,\ldots,V_n,V_{n+j}],\ \mathcal F_{\mathrm{tr}}\right)
\le
\mathbb P\!\left(p_j^\circ \le T+\widehat\Delta_{\mathrm{cal}} \,\middle|\, [V_1,\ldots,V_n,V_{n+j}],\ \mathcal F_{\mathrm{tr}}\right)
\le
T+\widehat\Delta_{\mathrm{cal}}.
\]
Consequently, the end of the argument yields
\[
\mathrm{FDR}(\mathcal S)
\le
\alpha + m\cdot \mathbb E[\widehat\Delta_{\mathrm{cal}}],
\]
which is the stated perturbation form.
\end{proof}

\subsection{Boundary stability under variance misspecification}
\label{app:boundary_stability}

Assumptions~\ref{assum:nuisance}--\ref{assum:continuity} should be read as sufficient conditions for boundary stability, not as requiring pointwise-perfect variance estimation.
By Proposition~\ref{prop:finite_sample_approx_fdr}, validity loss is governed by the calibration mislabeling rate $\widehat\Delta_{\mathrm{cal}}$, so variance misspecification affects FDR only through threshold-label flips between $\mathbbm 1\{\check A_i\ge c\}$ and $\mathbbm 1\{A_i\ge c\}$.
The primitive object is therefore the proxy/oracle label agreement on calibration, not exact recovery of $\Var(\phi\mid X)$.

To make this precise, write
\[
R_i(\rho):=\check A_i^2(\rho)-A_i^2
\]
for the squared-scale perturbation induced by denoising, proxy noise, and variance-estimation error.
Since $A_i,\check A_i(\rho)\ge 0$, a label flip can only occur when the oracle error lies within the perturbation radius of the tolerance boundary:
\[
\left\{
\mathbbm 1\{\check A_i(\rho)\ge c\}\neq \mathbbm 1\{A_i\ge c\}
\right\}
\subseteq
\left\{
|A_i^2-c^2|\le |R_i(\rho)|
\right\}.
\]
Thus, under a standard margin condition around the boundary,
\[
    \mathbb P\big(|A_i^2-c^2|\le t\big)\le C t^\gamma,\qquad t>0,
\]

small perturbations imply small boundary mislabeling even when $\widehat V$ is imperfect.
Variance consistency should thus be viewed as one convenient sufficient route to reducing
proxy perturbations, rather than as the primitive requirement itself. The fundamental
requirement for FDR control is that the denoised proxy preserve the reliable/unreliable
label except on a vanishing fraction of calibration points near $c$.

\begin{lemma}[Margin condition implies proxy boundary stability]
Suppose that for some constants $C>0$ and $\gamma>0$,
\[
    \mathbb P(|A_i^2-c^2|\le t)\le Ct^\gamma,\qquad t>0,
\]
and that $R_i(\rho)\to0$ in probability. Then
\[
    \mathbb P\left(
    \mathbf 1\{\check A_i(\rho)\ge c\}
    \ne
    \mathbf 1\{A_i\ge c\}
    \right)\to0.
\]
Consequently, for i.i.d. calibration units satisfying the same perturbation condition,
\(\widehat\Delta_{\rm cal}\to0\) in probability. This is the sufficient perturbation
condition needed to make proxy boundary instability negligible.
\end{lemma}

\begin{proof}
A threshold-label flip implies
\[
\mathbbm 1\{\check A_i(\rho)\ge c\}\neq \mathbbm 1\{A_i\ge c\}
\ \Longrightarrow\
|A_i^2-c^2|\le |R_i(\rho)|.
\]
For any $\eta>0$,
\[
\mathbb P\big(|A_i^2-c^2|\le |R_i(\rho)|\big)
\le
\mathbb P\big(|R_i(\rho)|>\eta\big)
+\mathbb P\big(|A_i^2-c^2|\le \eta\big).
\]
The first term vanishes because $R_i(\rho)\to 0$ in probability. The second term is at most $C\eta^\gamma$ by the margin condition, so it can be made arbitrarily small by choosing $\eta\downarrow 0$. This proves the first claim. The statement for $\widehat\Delta_{\mathrm{cal}}$ follows by applying the same argument across calibration units and averaging.
\end{proof}

\subsection{Proof of Theorem~\ref{thm:main}: asymptotic FDR control}
\label{app:proofs:asymptotic_fdr}

Theorem~\ref{thm:main} follows by combining the finite-sample perturbation bound above with proxy label stability.

\begin{lemma}[Proxy label stability implies $\widehat\Delta_{\mathrm{cal}}\to 0$]
\label{lem:delta_to_zero}
Assume Assumptions~\ref{assum:nuisance}--\ref{assum:continuity} in the main text.
Then $\widehat\Delta_{\mathrm{cal}}\xrightarrow{p}0$ as $n\to\infty$.
Moreover, $\mathbb E[\widehat\Delta_{\mathrm{cal}}]\to 0$.
\end{lemma}

\begin{proof}
Assumption~\ref{assum:nuisance} states (in particular) that $\check A_i\xrightarrow{p}A_i$ for each fixed $i$,
and Assumption~\ref{assum:continuity} states $\mathbb P(A_i=c)=0$.
Fix any $i$ and define the indicator map $\psi(x)=\mathbbm 1\{x\ge c\}$.
Since $\psi$ is continuous at all points except $c$, the continuous mapping theorem yields
\[
\mathbbm 1\{\check A_i\ge c\}
=
\psi(\check A_i)
\ \xrightarrow{p}\
\psi(A_i)
=
\mathbbm 1\{A_i\ge c\}.
\]
Because the indicators are bounded in $[0,1]$, convergence in probability implies convergence in $L^1$:
\[
\mathbb E\!\left[\left|\mathbbm 1\{\check A_i\ge c\}-\mathbbm 1\{A_i\ge c\}\right|\right]\to 0.
\]
Now $\widehat\Delta_{\mathrm{cal}}$ is the empirical average of these absolute differences over i.i.d.\ calibration indices.
By the weak law of large numbers,
\[
\widehat\Delta_{\mathrm{cal}}
=
\frac{1}{n}\sum_{i=1}^n
\left|\mathbbm 1\{\check A_i\ge c\}-\mathbbm 1\{A_i\ge c\}\right|
\xrightarrow{p}
\mathbb E\!\left[\left|\mathbbm 1\{\check A\ge c\}-\mathbbm 1\{A\ge c\}\right|\right]
=
0.
\]
This proves $\widehat\Delta_{\mathrm{cal}}\xrightarrow{p}0$.

For the expectation, note that $0\le \widehat\Delta_{\mathrm{cal}}\le 1$, hence the sequence is uniformly integrable.
Therefore convergence in probability to $0$ implies $\mathbb E[\widehat\Delta_{\mathrm{cal}}]\to 0$.
\end{proof}

\begin{proof}[Proof of Theorem~\ref{thm:main}]
Apply Proposition~\ref{prop:finite_sample_approx_fdr} to obtain, for each $(n,m)$,
\[
\mathrm{FDR}(\mathcal S_{n,m})
\le
\alpha + m\,\mathbb E[\widehat\Delta_{\mathrm{cal}}].
\]
By Lemma~\ref{lem:delta_to_zero}, $\mathbb E[\widehat\Delta_{\mathrm{cal}}]\to 0$ as $n\to\infty$.
Hence, along any growth regime such that $m\,\mathbb E[\widehat\Delta_{\mathrm{cal}}]\to 0$,
\[
\limsup_{n,m\to\infty}\mathrm{FDR}(\mathcal S_{n,m})
\le \alpha,
\]
which has been proved.
\end{proof}

\subsection{Proof of Proposition~\ref{prop:snr_barrier}: signal-to-noise barrier}
\label{app:proofs:snr}

\begin{proof}[Proof of Proposition~\ref{prop:snr_barrier}]
Under the stylized model $\tilde A^2=A^2+V(X)+\varepsilon$ with $\varepsilon$ mean-zero and independent of $(A^2,V(X))$,
\[
\mathrm{Cov}(\tilde A^2,A^2)
=
\mathrm{Cov}(A^2+V(X)+\varepsilon,\ A^2)
=
\mathrm{Var}(A^2)+\mathrm{Cov}(V(X),A^2)+\mathrm{Cov}(\varepsilon,A^2).
\]
By the stated independence and the assumption that $A^2$ is weakly related to $V(X)$ (e.g.\ $\mathrm{Cov}(V(X),A^2)=0$),
this reduces to $\mathrm{Cov}(\tilde A^2,A^2)=\mathrm{Var}(A^2)$.
Also,
\[
\mathrm{Var}(\tilde A^2)
=
\mathrm{Var}(A^2)+\mathrm{Var}(V(X))+\mathrm{Var}(\varepsilon),
\]
again by independence/uncorrelatedness.
Therefore
\[
\mathrm{Corr}(\tilde A^2,A^2)
=
\frac{\mathrm{Cov}(\tilde A^2,A^2)}{\sqrt{\mathrm{Var}(\tilde A^2)\mathrm{Var}(A^2)}}
=
\frac{\mathrm{Var}(A^2)}{\sqrt{\mathrm{Var}(A^2)\big(\mathrm{Var}(A^2)+\mathrm{Var}(V(X))+\mathrm{Var}(\varepsilon)\big)}}.
\]
If $\mathrm{Var}(V(X))+\mathrm{Var}(\varepsilon)\gg \mathrm{Var}(A^2)$, the denominator dominates and the ratio tends to $0$.
\end{proof}

\subsection{BH cutoff consistency for the score threshold and coupling to proxy stability}
\label{app:proofs:bh_cutoff}

This subsection proves the hardest technical ingredient requested: the consistency of the BH-induced score cutoff
$\hat t(\alpha)\to t(\alpha)$, and its coupling to proxy perturbations.

\textbf{From BH p-values to a score threshold.}
Because each p-value $p_j^\ast$ in \eqref{eq:oracle_pvals_app} is (strictly) increasing in $\hat A_{n+j}$, the BH rejection set is
equivalent to thresholding the score:
there exists a (random) cutoff $\hat t_{n,m}$ such that
\begin{equation}
\label{eq:bh_is_threshold}
\mathcal S^\ast
=
\mathrm{BH}(p^\ast_{1:m};\alpha)
=
\{j\in[m]: \hat A_{n+j}\le \hat t_{n,m}\}.
\end{equation}
(Under no ties, $\hat t_{n,m}$ can be taken as $\max\{\hat A_{n+j}: j\in\mathcal S^\ast\}$ with the convention $\max\emptyset=-\infty$.)

\textbf{Population target.}
Define the population quantities
\begin{equation}
\label{eq:def_H_R_app}
H(t):=\mathbb P(A\ge c,\ \hat A\le t),
\qquad
R(t):=\mathbb P(\hat A\le t),
\end{equation}
and the population cutoff
\begin{equation}
\label{eq:def_t_alpha_app}
t(\alpha)
:=
\sup\Big\{t\in\mathbb R:\ \frac{H(t)}{R(t)}\le \alpha\Big\}.
\end{equation}
Under the regularity condition in Proposition~\ref{prop:power_denoised}(iii),
$t(\alpha)$ is the unique boundary point where $H(t)/R(t)$ crosses $\alpha$.

\subsubsection{Oracle cutoff consistency}
\label{app:proofs:bh_cutoff_oracle}

We present a self-contained proof modeled on Appendix~B.4 of \citet{jin2023selection} (Proposition~2.10 therein),
specialized to our score-threshold formulation.

\begin{assumption}[Regularity at the boundary]
\label{assum:boundary_reg_app}
Assume:
(i) $(\hat A,A)$ has a continuous joint distribution and $\hat A$ has a continuous marginal distribution;
(ii) $R(t)>0$ for $t>t_{\min}$ and $H(t)$ is continuous in $t$;
(iii) there exists $\varepsilon>0$ such that
\[
\frac{H(t)}{R(t)}<\alpha
\quad\text{for all } t\in[t(\alpha)-\varepsilon,\ t(\alpha)).
\]
\end{assumption}

\begin{theorem}[BH cutoff consistency for oracle DCA]
\label{thm:bh_cutoff_consistency_oracle}
Assume Assumption~\ref{assum:cond_exch_app}, \eqref{eq:no_ties_assump}, and Assumption~\ref{assum:boundary_reg_app}.
Let $\hat t_{n,m}$ be the BH-induced score cutoff in \eqref{eq:bh_is_threshold} for oracle p-values $p^\ast_{1:m}$.
Then, as $n\to\infty$ and $m\to\infty$,
\[
\hat t_{n,m}\xrightarrow{p} t(\alpha).
\]
\end{theorem}

\begin{proof}
Condition on $\mathcal F_{\mathrm{tr}}$ throughout; all limits below are unconditional because $\mathcal F_{\mathrm{tr}}$ is independent of candidates.

\textbf{Step 1: A convenient representation of BH via the empirical CDF of p-values.}
Let $p_{(1)}^\ast\le\cdots\le p_{(m)}^\ast$ be the ordered oracle p-values.
Define the BH p-value cutoff
\[
\hat u_{n,m}
:=
\alpha \hat k/m,
\qquad
\hat k:=\max\{k: p_{(k)}^\ast\le \alpha k/m\}.
\]
Equivalently (see Storey et al., 2004, and Appendix~B.4 of \citealp{jin2023selection}), $\hat u_{n,m}$ can be written as
\begin{equation}
\label{eq:bh_tau_rep}
\hat u_{n,m}
=
\sup\left\{
u\in[0,1]:
\frac{u}{\widehat F_m^\ast(u)}\le \alpha
\right\},
\qquad
\widehat F_m^\ast(u):=\frac{1}{m}\sum_{j=1}^m \mathbbm 1\{p_j^\ast\le u\},
\end{equation}
with the convention $u/0:=+\infty$.
Moreover, under no ties, $\mathcal S^\ast=\{j: p_j^\ast\le \hat u_{n,m}\}$.

\textbf{Step 2: Identify the limiting p-value map as $p=H(\hat A)$.}
Recall from Lemma~\ref{lem:pval_as_conformal} that
\[
p_j^\ast
=
\frac{1+\sum_{i=1}^n \mathbbm 1\{V_i<\hat V_{n+j}\}}{n+1}
\quad\text{with}\quad \hat V_{n+j}=\hat A_{n+j}.
\]
Because $\hat A_{n+j}\in[0,\bar M]$ and the good calibration units have $V_i\ge 2\bar M$, the indicator $\mathbbm 1\{V_i<\hat A_{n+j}\}$ equals $\mathbbm 1\{A_i\ge c,\ \hat A_i<\hat A_{n+j}\}$.
Thus
\[
p_j^\ast
=
\frac{1+\sum_{i=1}^n \mathbbm 1\{A_i\ge c,\ \hat A_i<\hat A_{n+j}\}}{n+1}.
\]
Define the empirical ``null-mass'' function
\[
\widehat H_n(t)
:=
\frac{1+\sum_{i=1}^n \mathbbm 1\{A_i\ge c,\ \hat A_i<t\}}{n+1}.
\]
Then $p_j^\ast=\widehat H_n(\hat A_{n+j})$.

By the Glivenko--Cantelli theorem for the VC class of half-lines and the strong law of large numbers,
$\widehat H_n(t)\to H(t)$ uniformly in $t$ over compact sets, in probability (indeed a.s.\ under i.i.d.).
In particular, for any fixed $t$,
\begin{equation}
\label{eq:Hhat_to_H}
\widehat H_n(t)\xrightarrow{p} H(t).
\end{equation}

\textbf{Step 3: Convergence of the empirical CDF of oracle p-values.}
Let $P:=H(\hat A)$ denote the population limit of oracle p-values (a random variable on $[0,1]$).
Let $F_P(u):=\mathbb P(P\le u)$ denote its CDF.
For each fixed $u$, by \eqref{eq:Hhat_to_H} and the continuous mapping theorem,
\[
\mathbbm 1\{p_j^\ast\le u\}
=
\mathbbm 1\{\widehat H_n(\hat A_{n+j})\le u\}
\ \xrightarrow{p}\
\mathbbm 1\{H(\hat A_{n+j})\le u\}.
\]
Moreover, conditional on the calibration set, $\{p_j^\ast\}_{j=1}^m$ are i.i.d.\ because test candidates are i.i.d.
Thus, for each $u$, by the weak law of large numbers,
\begin{equation}
\label{eq:Fhatm_pointwise}
\widehat F_m^\ast(u)
=
\frac{1}{m}\sum_{j=1}^m \mathbbm 1\{p_j^\ast\le u\}
\xrightarrow{p}
F_P(u).
\end{equation}
Using standard empirical-process arguments for one-dimensional thresholds (or DKW conditional on calibration and then averaging),
the convergence \eqref{eq:Fhatm_pointwise} can be strengthened to uniform convergence over $u\in[0,1]$:
\begin{equation}
\label{eq:Fhatm_uniform}
\sup_{u\in[0,1]} \big|\widehat F_m^\ast(u)-F_P(u)\big|\xrightarrow{p}0.
\end{equation}

\textbf{Step 4: Consistency of the BH p-value cutoff $\hat u_{n,m}$.}
Define the population BH fixed point
\begin{equation}
\label{eq:u_star_def}
u_\star(\alpha)
:=
\sup\Big\{u\in[0,1]: \frac{u}{F_P(u)}\le \alpha\Big\},
\qquad (u/0:=+\infty).
\end{equation}
The mapping $T(F):=\sup\{u: u/F(u)\le \alpha\}$ is continuous at $F_P$ under the boundary regularity
Assumption~\ref{assum:boundary_reg_app} (this is exactly the ``no-flat-spot'' condition used in \citealp{jin2023selection}).
Concretely, Assumption~\ref{assum:boundary_reg_app}(iii) implies that there is an $\varepsilon>0$ such that
$u/F_P(u)<\alpha$ for all $u\in[u_\star(\alpha)-\varepsilon,u_\star(\alpha))$ and $u/F_P(u)>\alpha$ for $u\in(u_\star(\alpha),u_\star(\alpha)+\varepsilon]$,
so the argmax is locally unique and stable.
Combining this stability with \eqref{eq:Fhatm_uniform} yields
\begin{equation}
\label{eq:uhat_to_ustar}
\hat u_{n,m}
=
T(\widehat F_m^\ast)
\xrightarrow{p}
T(F_P)
=
u_\star(\alpha).
\end{equation}

\textbf{Step 5: Translate p-value cutoff to score cutoff.}
Because $p_j^\ast=\widehat H_n(\hat A_{n+j})$ is monotone increasing in $\hat A_{n+j}$,
the BH rejection set $\{j: p_j^\ast\le \hat u_{n,m}\}$ is of the form $\{j:\hat A_{n+j}\le \hat t_{n,m}\}$ with
\[
\hat t_{n,m}
:=
\sup\{t: \widehat H_n(t)\le \hat u_{n,m}\}.
\]
By uniform convergence $\widehat H_n\to H$ and \eqref{eq:uhat_to_ustar}, we obtain $\hat t_{n,m}\to t_\star$ in probability,
where $t_\star:=\sup\{t: H(t)\le u_\star(\alpha)\}$.

Finally, we identify $t_\star$ with $t(\alpha)$.
Since $P=H(\hat A)$ and $F_P(u)=\mathbb P(H(\hat A)\le u)=\mathbb P(\hat A\le H^{-1}(u))=R(H^{-1}(u))$,
the fixed point $u_\star(\alpha)$ in \eqref{eq:u_star_def} satisfies
\[
\frac{u_\star(\alpha)}{F_P(u_\star(\alpha))}\le \alpha
\iff
\frac{H(t_\star)}{R(t_\star)}\le \alpha,
\]
with equality at the boundary under Assumption~\ref{assum:boundary_reg_app}(iii).
Therefore $t_\star$ coincides with $t(\alpha)$ in \eqref{eq:def_t_alpha_app}.
This completes the proof.
\end{proof}

\subsubsection{Proxy cutoff consistency via perturbation coupling}
\label{app:proofs:bh_cutoff_proxy}

Let $\tilde{\mathcal S}$ be BH applied to the proxy p-values $\tilde p_{1:m}$, and define the induced proxy cutoff $\tilde t_{n,m}$ by
\[
\tilde{\mathcal S}=\{j:\hat A_{n+j}\le \tilde t_{n,m}\}.
\]

\begin{theorem}[BH cutoff consistency for DCA via proxy stability]
\label{thm:bh_cutoff_consistency_proxy}
Assume the conditions of Theorem~\ref{thm:bh_cutoff_consistency_oracle} and additionally that
$\widehat\Delta_{\mathrm{cal}}\xrightarrow{p}0$ (e.g.\ Lemma~\ref{lem:delta_to_zero}).
Then $\tilde t_{n,m}\xrightarrow{p} t(\alpha)$.
\end{theorem}

\begin{proof}
By Proposition~\ref{prop:finite_sample_approx_fdr}, we have the uniform perturbation bound
\[
\sup_{j\in[m]} |\tilde p_j-p_j^\ast|\le \widehat\Delta_{\mathrm{cal}}.
\]
Since $\widehat\Delta_{\mathrm{cal}}\xrightarrow{p}0$, it follows that
\[
\sup_{u\in[0,1]}\left|
\frac{1}{m}\sum_{j=1}^m \mathbbm 1\{\tilde p_j\le u\}
-
\frac{1}{m}\sum_{j=1}^m \mathbbm 1\{p_j^\ast\le u\}
\right|
\xrightarrow{p}0,
\]
i.e., the empirical CDF of proxy p-values converges uniformly to the same limit $F_P$ as in \eqref{eq:Fhatm_uniform}.
Therefore the BH p-value cutoff computed from $\tilde p_{1:m}$ converges to the same population fixed point $u_\star(\alpha)$,
and repeating Step~5 of the oracle proof yields $\tilde t_{n,m}\xrightarrow{p}t(\alpha)$.
\end{proof}

\subsection{Proof of Proposition~\ref{prop:power_denoised}: asymptotic power after denoising}
\label{app:proofs:power}

We now prove the asymptotic power characterization stated in Proposition~\ref{prop:power_denoised}.
Recall that $\tilde{\mathcal S}=\{j:\hat A_{n+j}\le \tilde t_{n,m}\}$ and $\tilde t_{n,m}\xrightarrow{p} t(\alpha)$ by
Theorem~\ref{thm:bh_cutoff_consistency_proxy}.

\begin{proof}[Proof of Proposition~\ref{prop:power_denoised}]
Let $G_j:=\mathbbm 1\{A_{n+j}<c\}$ denote the indicator that test unit $j$ is good.
Define the empirical power (as in the main text) by
\[
\mathrm{Power}(\tilde{\mathcal S})
=
\mathbb E\!\left[
\frac{\sum_{j=1}^m \mathbbm 1\{j\in\tilde{\mathcal S}\}\,G_j}{\sum_{j=1}^m G_j\ \vee\ 1}
\right].
\]
Condition on the random cutoff $\tilde t_{n,m}$ and on $\mathcal F_{\mathrm{tr}}$.
Given $\tilde t_{n,m}$, the indicators $\mathbbm 1\{j\in\tilde{\mathcal S}\}=\mathbbm 1\{\hat A_{n+j}\le \tilde t_{n,m}\}$
are i.i.d.\ across $j$ and jointly i.i.d.\ with $G_j$ under Assumption~\ref{assum:cond_exch_app}.
Thus, by the law of large numbers,
\[
\frac{1}{m}\sum_{j=1}^m \mathbbm 1\{j\in\tilde{\mathcal S}\}\,G_j
\xrightarrow{p}
\mathbb P(\hat A\le t(\alpha),\ A<c),
\qquad
\frac{1}{m}\sum_{j=1}^m G_j
\xrightarrow{p}
\mathbb P(A<c),
\]
where we used $\tilde t_{n,m}\xrightarrow{p}t(\alpha)$ and the continuous mapping theorem to replace the random cutoff.
Since the ratio is bounded in $[0,1]$, convergence in probability upgrades to convergence in expectation by uniform integrability:
\[
\mathrm{Power}(\tilde{\mathcal S})
\longrightarrow
\frac{\mathbb P(\hat A\le t(\alpha),\ A<c)}{\mathbb P(A<c)}
=
\mathbb P(\hat A\le t(\alpha)\mid A<c),
\]
which is exactly the claim (noting $\hat A=g(X)$).
\end{proof}

\subsection{Proof of Proposition~\ref{prop:optimal_threshold_power}: optimal power among threshold rules}
\label{app:proofs:optimal_power}

We give a fully explicit proof for the threshold-direction used in Proposition~\ref{prop:power_denoised}, i.e.\ selecting
$\mathcal S_t=\{j:\hat A_{n+j}\le t\}$. (If one uses $\{g(X)\ge t\}$ instead, replace $\hat A$ by $-\hat A$.)

\begin{proof}[Proof of Proposition~\ref{prop:optimal_threshold_power}]
Fix the alignment score $\hat A=g(X)$.
For any threshold $t\in\mathbb R$, define the score-threshold rule
\[
\mathcal S_t:=\{j\in[m]: \hat A_{n+j}\le t\}.
\]
The population FDR of this rule is
\[
\mathrm{FDR}(t)
=
\frac{\mathbb P(A\ge c,\ \hat A\le t)}{\mathbb P(\hat A\le t)}
=
\frac{H(t)}{R(t)},
\]
with the convention $0/0:=0$.
The population power of this rule is
\[
\mathrm{Pow}(t)
:=
\mathbb P(\hat A\le t\mid A<c).
\]

Observe that $\mathrm{Pow}(t)$ is non-decreasing in $t$:
if $t_2>t_1$ then $\{\hat A\le t_1\}\subseteq\{\hat A\le t_2\}$, hence
$\mathbb P(\hat A\le t_1\mid A<c)\le \mathbb P(\hat A\le t_2\mid A<c)$.
Therefore, among all feasible thresholds satisfying $\mathrm{FDR}(t)\le\alpha$,
the power is maximized by taking the \emph{largest} feasible threshold. By definition \eqref{eq:def_t_alpha_app},
$t(\alpha)$ is precisely the supremum of the feasible set $\{t: H(t)/R(t)\le\alpha\}$.
Hence $\mathcal S_{t(\alpha)}$ maximizes $\mathrm{Pow}(t)$ among score-threshold rules under the FDR constraint.
\end{proof}

\subsection{Corollary: asymptotic optimality of DCA among score-threshold rules}
\label{app:proofs:cor_opt}

For completeness (and to make the ``maximize power subject to FDR'' objective explicit),
we record the standard corollary combining cutoff consistency with Proposition~\ref{prop:optimal_threshold_power}.

\begin{corollary}[Asymptotic optimality of DCA among score-threshold rules]
\label{cor:asymptotic_optimality_dca_app}
Under the assumptions of Theorem~\ref{thm:bh_cutoff_consistency_proxy},
\[
\lim_{n,m\to\infty}\mathrm{Power}(\tilde{\mathcal S})
=
\sup_{t:\,H(t)/R(t)\le \alpha}\ \mathbb P(\hat A\le t\mid A<c)
=
\mathbb P(\hat A\le t(\alpha)\mid A<c).
\]
\end{corollary}

\begin{proof}
By Proposition~\ref{prop:power_denoised} and Theorem~\ref{thm:bh_cutoff_consistency_proxy},
$\mathrm{Power}(\tilde{\mathcal S})\to \mathbb P(\hat A\le t(\alpha)\mid A<c)$.
By Proposition~\ref{prop:optimal_threshold_power}, the right-hand side equals the supremum of $\mathrm{Pow}(t)$
over feasible thresholds. This proves the claim.
\end{proof}

\section{Experimental Protocols, Robustness Checks, and Additional Benchmarks}
\label{app:exp_details}

This appendix collects the experimental material omitted from the main text: the common protocol,
implementation choices, synthetic data-generating mechanisms (DGPs), robustness and sensitivity checks,
and additional semi-synthetic/real-data benchmarks.
We first state the fixed pipeline used across experiments, then describe the main synthetic settings
(Settings~1--4), then collect ablations and robustness checks, and finally report the IHDP, external benchmark,
and multi-treatment studies. This ordering mirrors the main experimental claims: safety and power under noisy
counterfactual proxies, robustness to covariate shift, and external validity beyond the base synthetic setting.
The code for all our experiments is available at \url{https://anonymous.4open.science/r/Anonymous_code_fgdx}.

\textbf{What each setting is designed to stress-test.}
\begin{itemize}
    \item \textbf{Setting 1 (Gaussian + heteroskedastic):}
    an unbiased DR proxy with conditional variance strongly dependent on X, isolating the cleanest version of the proxy-noise ranking barrier by washing out valid ranking information.
    \item \textbf{Setting 2 (Hard overlap + heavy tails):}
    a much noisier DR proxy via near-violations of overlap and heavy-tailed (Student-$t$) outcome noise, which stresses both proxy ranking and variance estimation and empirically favors \emph{conservative} denoising (smaller $\rho$).
    \item \textbf{Setting 3 (Mixture / mean / scale covariate shift):} 
    breaks exchangeability between calibration (source) and candidate pool (target) covariates, and uses weighted conformal $p$-values as a drop-in calibration correction without changing the rest of DCA.
    \item \textbf{Setting 4 (Covariance / rotation shift):} a harder covariate shift where the mean can match but covariance/geometry differs,
    illustrating that weighting is not only for trivial mean shifts.
    \item \textbf{Setting 5 (IHDP semi-synthetic):} 
     realistic covariates and treatment assignment, with ground-truth CATE available as the benchmark provides simulated potential outcomes (via $\mu_0,\mu_1$) that are synthetic by construction (not from observed true counterfactuals in IHDP).
\end{itemize}

Across settings, variations are introduced only through the data-generating process (DGP) or the shift mechanism; the learning pipeline---including sample splitting, learners, BH selection, and the thresholding rule---is kept identical, except that covariate-shift settings employ weighted conformal $p$-values.

\subsection{Common Protocol Across Experiments}
\label{app:common_protocol}

This subsection defines the fixed pipeline used before any setting-specific DGP or benchmark details.
Robustness checks that vary parts of this protocol are collected later in Appendix~\ref{app:robustness_ablations}.

\subsubsection{Sample splitting and leakage control}
\label{app:protocol_splitting}

Each dataset is split into four disjoint parts:
$\mathcal{D}_{\mathrm{tr1}}$ for nuisance/variance estimation,
$\mathcal{D}_{\mathrm{tr2}}$ for training the alignment predictor $g$,
$\mathcal{D}_{\mathrm{cal}}$ for conformal calibration,
and $\mathcal{D}_{\mathrm{test}}$ as an unlabeled candidate pool for selection and evaluation.
\textbf{Implementation:} nuisance models $(\hat\mu_0,\hat\mu_1,\hat e)$ and variance models are trained \emph{only} on $\mathcal{D}_{\mathrm{tr1}}$.
The DR pseudo-outcome $\phi_i$ and the variance estimate $\hat V(X_i)$ are then computed on $\mathcal{D}_{\mathrm{tr2}}$ and $\mathcal{D}_{\mathrm{cal}}$
using these $\mathcal{D}_{\mathrm{tr1}}$-trained models. Thus, for any unit used in $g$-training or calibration, its outcome is never used to fit the nuisance models.
\textbf{Clarification: }We use sample-splitting to control information leakage; no additional secondary cross-fitting is performed on $\mathcal{D}_{\mathrm{tr2}} \cup \mathcal{D}_{\mathrm{cal}}$ in the experimental code (this is theoretically feasible but unnecessary, and would introduce extra implementation complexity).

\subsubsection{Proxy construction and denoising scale}
\label{app:protocol_proxy}

Our denoising identity is naturally written in \emph{squared error}:
\[
\mathbb{E}\big[(\hat\tau(X)-\phi)^2\mid X\big]=A(X)^2+\Var(\phi\mid X),
\quad A(X):=|\hat\tau(X)-\tau(X)|.
\]
Therefore, throughout experiments we implement proxy scores on the squared scale and optionally take a square-root only
for reporting/interpretability. Concretely:
\[
\widetilde A_i^2 := (\hat\tau(X_i)-\phi_i)^2,\qquad
\check A_i^2(\rho) := \big(\widetilde A_i^2-\rho\,\hat V(X_i)\big)_+.
\]
\textbf{Naive baseline is defined as the same proxy family with $\rho=0$:}
$\check A_i^2(0)=\widetilde A_i^2$.
This alignment ensures any gain is attributable to \emph{variance subtraction}, not to a change of metric
(e.g., absolute vs squared).

\subsubsection{Conformal $p$-values, BH selection, and evaluation metrics}
\label{app:protocol_conformal_eval}

Let $s(x)=g(x)$ be the alignment score (lower means more reliable).
We form the ``null'' subset on calibration as
\[
\mathcal{I}_0=\{i\in\mathcal{D}_{\mathrm{cal}}:\ \text{proxy}_i \ge c\},
\]
where $c$ is an application-level acceptability threshold (Section~\ref{sec:problem}) that can be instantiated in many reasonable ways depending on the use case; in experiments it is fixed using only calibration-side information and kept identical across methods.
(See code for concrete instantiations/ablations at {\url{https://anonymous.4open.science/r/Anonymous_code_fgdx}}.)

Here, the proxy is method-specific for ours/naive, but for certain baselines (Ensemble/CQR) the null mask is defined using the \emph{shared} raw proxy $(\hat\tau-\phi)^2$ to keep comparability.
Then for a test candidate $j$,
\[
p_j \;=\; \frac{1+\#\{i\in\mathcal{I}_0:\ s(X_i)\le s(X_{j})\}}{|\mathcal{D}_{\mathrm{cal}}|+1}.
\]
Smaller $p_j$ corresponds to unusually small score relative to the ``bad'' calibration subset, hence more likely reliable.
When exact ties occur (most notably because $\check A^2=(\widetilde A^2-\rho\hat V)_+$ can truncate scores at zero),
the tie-safe implementation replaces the displayed comparison by the randomized lexicographic comparison in
Appendix~\ref{app:randomized_ties}; the ablation in Appendix~\ref{app:robustness_ties} verifies that this choice does not drive the empirical curves.

For each target level $\alpha$, we apply BH to $\{p_j\}_{j\in\mathcal{D}_{\mathrm{test}}}$ to obtain $\mathcal S(\alpha)$.
When ground-truth CATE is available (synthetic and semi-synthetic IHDP), we compute on the \emph{test} pool:
\[
\mathrm{FDP}(\alpha)=\frac{|\mathcal S(\alpha)\cap\mathcal H_0|}{|\mathcal S(\alpha)|\vee 1},\quad
\mathrm{Power}(\alpha)=\frac{|\mathcal S(\alpha)\cap\mathcal H_1|}{|\mathcal H_1|\vee 1},
\]
where $\mathcal H_0=\{A\ge c\}$ and $\mathcal H_1=\{A<c\}$ are defined on the \emph{test} candidate pool.

\subsubsection{Learners, baselines, and denoising-strength selection}
\label{app:rho_selection}

Unless stated otherwise:
Random Forests are used for $\hat\mu_t(\cdot)$ and $\hat V(\cdot)$, logistic regression for $\hat e(\cdot)$,
and a Random Forest regressor for $g(\cdot)$.
All models use the same features $X$ (no handcrafted feature expansion),
and we standardize covariates for linear/logistic models.
We intentionally avoid heavy hyperparameter tuning to emphasize that DCA is a \emph{wrapper}:
gains come from denoising + conformal selection, not model-specific tweaks.

Baselines:
(i) \textbf{Ensemble Std}: train $M$ independent nuisance models on $\mathcal{D}_{\mathrm{tr1}}$ with different seeds and score by $\mathrm{sd}(\hat\tau^{(m)}(x))$;
(ii) \textbf{CQR Width}: fit quantile regressors for $Y\mid X,T=t$ (GBDT quantile loss) and score by prediction interval width.

$\rho$ controls how aggressively we subtract estimated conditional variance. Larger $\rho$ improves SNR when $\hat V$ is accurate,
but risks over-subtraction when $\hat V$ is noisy (especially under heavy tails / hard overlap), which can flip proxy labels near $c$.
Hence, the empirically optimal $\rho$ is \emph{regime-dependent}.

\textbf{Why not simply set $\rho=1$?}
On the squared scale, the raw DR proxy has the approximate decomposition
\[
\widetilde A_i^2
\approx
A_i^2+V_i+\varepsilon_i,
\qquad
V_i:=\Var(\phi_i\mid X_i),
\]
where $\varepsilon_i$ denotes the residual proxy noise. If $\widehat V_i=V_i+\delta_i$, then after denoising
\[
\check A_i^2(\rho)
=
\big(\widetilde A_i^2-\rho\widehat V_i\big)_+
\approx
\big(A_i^2+(1-\rho)V_i+\varepsilon_i-\rho\delta_i\big)_+ .
\]
Thus, $\rho=1$ is attractive only in the idealized goal of reconstructing $A_i^2$ when $V_i$ is known accurately.
For selective inference, however, the downstream target is not unbiased reconstruction of $A_i^2$; it is stable classification
of calibration units relative to the tolerance boundary $c$.
When $\widehat V$ is noisy, larger $\rho$ amplifies the perturbation $-\rho\delta_i$ and can over-subtract high-variance points,
creating proxy/oracle label flips near $c$.
This is exactly the quantity controlled by Proposition~\ref{prop:finite_sample_approx_fdr}:
\[
\widehat\Delta_{\mathrm{cal}}(\rho)
=
\frac{1}{n}\sum_{i=1}^n
\left|
\mathbbm 1\{\check A_i(\rho)\ge c\}
-
\mathbbm 1\{A_i\ge c\}
\right|.
\]
The practical role of $\rho$ is therefore to improve ranking signal while keeping $\widehat\Delta_{\mathrm{cal}}(\rho)$ small.
Moderate denoising can dominate full subtraction because it removes a substantial portion of variance-driven ranking noise
without making boundary labels overly sensitive to variance-estimation error.

\textbf{A fixed, data-splitting tuning rule (no test peeking).}
For each candidate $\rho$ in a prespecified grid, we run the full DCA pipeline using the same splitting protocol.
We choose $\rho$ on a \emph{validation slice carved out from the outcome-observed reference side}
(so selection never uses test outcomes). Concretely, we pick the $\rho$ that maximizes average power
over a small grid of target levels $\alpha$ subject to a mild FDR feasibility constraint
(realized FDR not exceeding the nominal line by more than a small slack).
This rule is fixed across all experiments; only the candidate grid changes by setting (below).
In semi-synthetic experiments, the validation slice has simulator-provided CATE and the objective can be evaluated directly.
In purely real-data deployment, where oracle CATE errors are unavailable, the same principle suggests using a conservative
pre-specified $\rho$ or a reference-side proxy-stability criterion rather than tuning on the final candidate pool.

\textbf{Default interpretation.}
The sensitivity plots should be read as evidence for stable regions, not as single-point hyperparameter optimization.
In the standard heteroskedastic Gaussian setting, the useful region is a moderate band around $\rho\approx0.5$--$0.8$;
under hard overlap, heavy tails, or finite semi-synthetic folds, the preferred value shifts smaller because $\widehat V$ is less stable.
Accordingly, a practical default is to start with moderate subtraction and lower $\rho$ in regimes where inverse-propensity weights,
heavy tails, or noisy variance regressions make over-subtraction more likely.

\textbf{Grids used in sensitivity plots.}
We additionally report sensitivity plots over $\rho$ and noise scale parameters $\sigma$ to show robustness and to
communicate that DCA does \emph{not} require fragile tuning.
The $\rho$ and $\sigma$ grids are listed below. Values $\rho>1$ are included only as stress tests for over-subtraction;
the conservative denoising rule in the main theory uses $\rho\in[0,1]$.
\begin{itemize}
    \item Setting 1: $\rho\in\{0.25,0.5,0.75,1.0,1.5\}$, $\sigma\in\{0.5,1.0,1.5\}$.
    \item Setting 2: $\rho\in\{0.1,0.15,0.2,0.25,0.5\}$, $\sigma\in\{0.5,1.0,1.5\}$.
    \item Setting 3: $\rho\in\{0.25,0.5,0.75,1.0,1.5\}$, $\sigma\in\{0.5,1.0,1.5\}$.
    \item Setting 4: $\rho\in\{0.25,0.5,0.75,1.0,1.5\}$, $\sigma\in\{0.5,1.0,1.5\}$.
    \item IHDP (Setting 5): $\rho\in\{0.1,0.15,0.2,0.25,0.5\}$.
\end{itemize}

\subsection{Synthetic Data: DGPs for Settings 1--4}
\label{app:dgp_all}

Across synthetic settings, covariates are $d$-dimensional (default $d=10$).
Potential outcomes follow
\[
Y(t)=\mu_0(X)+t\cdot\tau(X)+\epsilon_t,
\]
with propensity, $(\mu_0,\tau)$, and noise depending on the setting. We provide both
a compact mathematical description and an \emph{implementation-matched} pseudocode
specification to make the DGPs exactly reproducible.

\textbf{Common sampling template.}
In all settings, once $(X,e(X),\mu_0(X),\tau(X),\epsilon_0,\epsilon_1)$ are specified,
we generate
\[
T\sim \mathrm{Bernoulli}(e(X)),\qquad
Y_0=\mu_0(X)+\epsilon_0,\qquad
Y_1=\mu_0(X)+\tau(X)+\epsilon_1,\qquad
Y=T\,Y_1+(1-T)\,Y_0,
\]
and record $\mathrm{ITE\_true}=\tau(X)$.

\begin{algorithm}
\caption{Synthetic DGP template (all settings)}
\label{alg:dgp_template}
\small
\begin{algorithmic}[1]
\REQUIRE Sample size $n$, dimension $d$, setting-specific functions: $e(\cdot), \mu_0(\cdot), \tau(\cdot)$ and noise generator.
\ENSURE $\{(X_i,T_i,Y_i,\tau(X_i),e(X_i))\}_{i=1}^n$.
\STATE Sample covariates $X\in\mathbb{R}^{n\times d}$ from the setting-specific covariate distribution.
\STATE Compute propensity vector $e\gets e(X)$ and clip elementwise to the setting-specific overlap interval.
\STATE Sample $T_i\sim\mathrm{Bernoulli}(e_i)$ independently for $i=1,\dots,n$.
\STATE Compute $\mu_0\gets \mu_0(X)$ and $\tau\gets \tau(X)$.
\STATE Draw noise $\epsilon_0,\epsilon_1$ from the setting-specific noise distribution (may depend on $X$).
\STATE Set $Y_0\gets \mu_0+\epsilon_0$, $Y_1\gets \mu_0+\tau+\epsilon_1$, and $Y\gets T\odot Y_1+(1-T)\odot Y_0$.
\STATE Output $X,T,Y,\mathrm{ITE\_true}\gets\tau,\ e\_\mathrm{true}\gets e$.
\end{algorithmic}
\end{algorithm}


\subsubsection{Setting 1: Gaussian outcomes with heteroskedastic noise}
\label{app:setting1}

\textbf{DGP.}
Covariates are i.i.d.\ $X\sim\mathcal N(0,I_d)$. Treatment assignment follows a logistic propensity:
\[
e(X)=\mathrm{sigmoid}(0.5X_0-0.5X_2),\quad e(X)\in[0.1,0.9]\ \text{(clipped into $[\eta,1-\eta]$ to enforce overlap)}.
\]
Baseline and CATE:
\[
\mu_0(X)=2\sin(X_0)+\max(X_1,0),\qquad
\tau(X)=\frac{2}{1+e^{-X_0}}+X_2.
\]
Noise is Gaussian but heteroskedastic:
\[
\epsilon_t\sim \mathcal N\big(0,\sigma^2(X)\big),\quad
\sigma(X)=\sigma_{\mathrm{base}}\,(1+0.5|X_0|),\quad
\sigma_{\mathrm{base}}\in\{0.5,1.0,1.5\}.
\]

\textbf{Why this setting is a clean ``proxy-noise barrier'' showcase.}
In this regime the DR pseudo-outcome $\phi$ is well-behaved in expectation, but the \emph{conditional proxy variance} $\Var(\phi\mid X)$ changes sharply with $X$ through $\sigma(X)$.
As a result, the raw proxy $(\hat\tau-\phi)^2$ is systematically inflated in high-variance regions even when $\hat\tau$ is accurate there, so naive proxy ranking becomes a poor surrogate for the true error ranking.
This is exactly the failure mode DCA targets: it is not a modeling issue in $\hat\tau$, but a \emph{noise-dominated reliability signal}.


\textbf{Results and denoising choice (main text vs appendix).}
Figure~\ref{fig:synth_noshift} (main text) reports a representative curve using $\rho=0.65$, while Figure~\ref{fig:setting1} (appendix) shows sensitivity across $\rho$ and noise scales.

\emph{What Fig.~\ref{fig:synth_noshift} demonstrates.}
Across all three noise scales, DCA produces a materially larger selected reliable set at the same target FDR level:
the \textbf{power curve turns on earlier} (selection becomes nontrivial at smaller $\alpha$) and rises faster once selection starts.
In contrast, the plug-in proxy (and uncertainty baselines such as ensemble std / CQR width) typically require much larger $\alpha$
before any meaningful yield appears.
Operationally, this is the gain we care about: \textbf{at a fixed acceptability tolerance, DCA makes the same conformal+BH pipeline actually select}
under heteroskedastic proxy noise, instead of remaining overly conservative and nearly empty.

\emph{How Fig.~\ref{fig:setting1} supports the choice of $\rho$ and robustness.}
The sensitivity plot shows a \textbf{broad band of moderate denoising} (roughly $\rho\in[0.5,0.8]$) where DCA achieves
(i) realized FDR that is closest to the target line $y=x$ (i.e., most nearly calibrated rather than overly conservative),
\emph{while simultaneously} (ii) attaining strong power.
We highlight $\rho=0.65$ as a representative choice because it sits near the center of this stable region and empirically offers an attractive
trade-off: FDR tracks $\alpha$ closely across $\sigma_{\mathrm{base}}\in\{0.5,1.0,1.5\}$, yet the selection yield remains high.
When $\rho$ is smaller, proxy variance is under-subtracted and behavior reverts toward the noisy plug-in ranking;
when $\rho$ is too large, over-subtraction can shrink proxy scores and weaken the calibration separation, reducing power.
Overall, the figure indicates DCA’s improvements are \textbf{not} a knife-edge tuning artifact but persist across a reasonably wide range of $\rho$ values.

\begin{figure*}
    \centering
    \includegraphics[width=\textwidth]{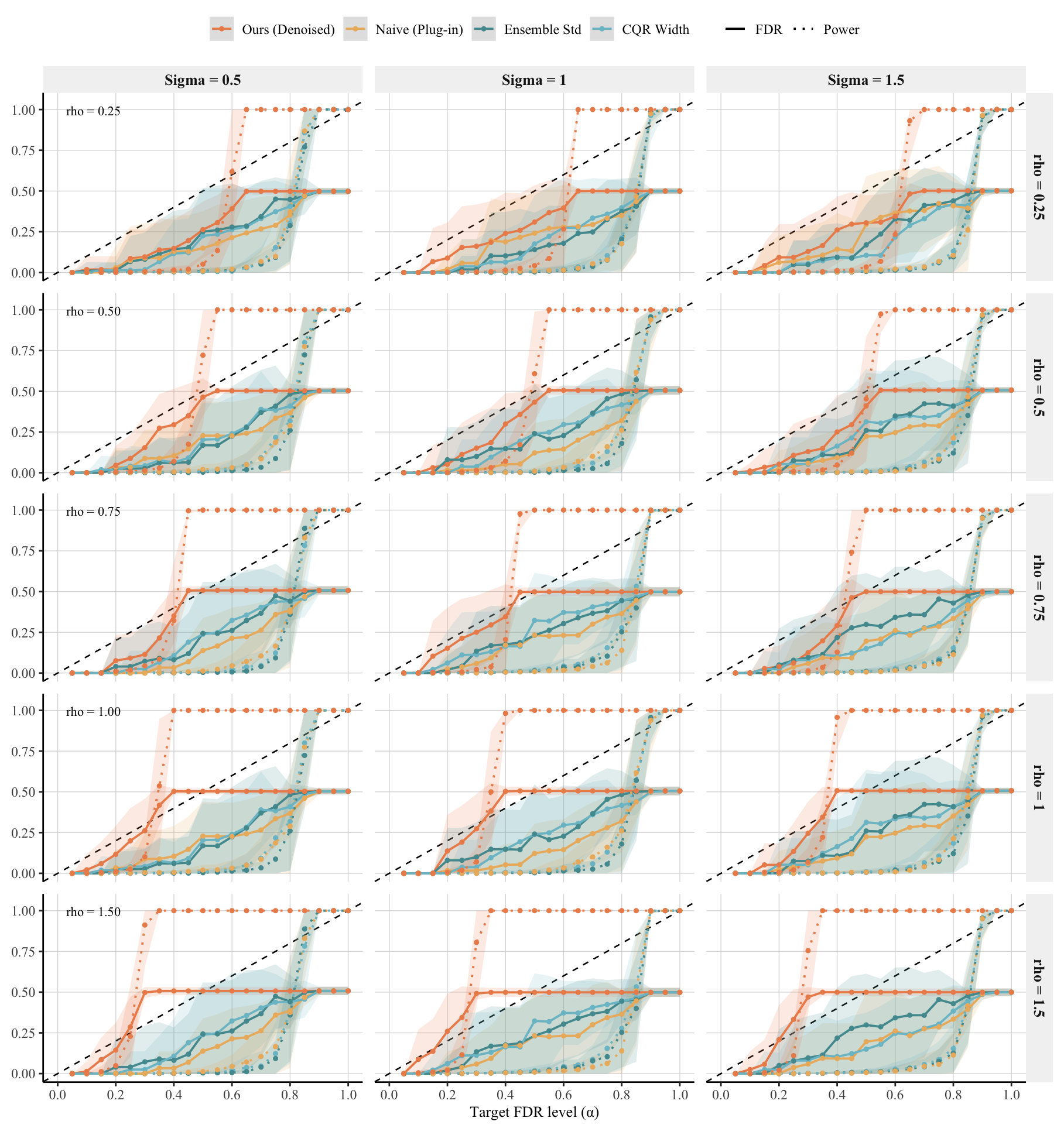}
    \caption{\textbf{Setting 1 sensitivity (Gaussian + heteroskedastic).}
    Realized FDR and power across $\rho\in\{0.25,0.5,0.75,1.0,1.5\}$ and $\sigma_{\mathrm{base}}\in\{0.5,1.0,1.5\}$.
    A wide range of moderate denoising strengths improves selection yield while preserving conservative or near-nominal FDR,
    consistent with DCA mitigating the heteroskedastic proxy-noise barrier rather than exploiting fragile hyperparameter tuning.}
    \label{fig:setting1}
\end{figure*}

\begin{algorithm}
\caption{Setting 1 DGP}
\label{alg:dgp_setting1}
\small
\begin{algorithmic}[1]
\REQUIRE $n,d,\sigma_{\mathrm{base}}$
\STATE Sample $X\in\mathbb R^{n\times d}$ with i.i.d.\ $\mathcal N(0,1)$ entries.
\STATE $e \leftarrow \mathrm{sigmoid}(0.5X_{:,0}-0.5X_{:,2})$; $e\leftarrow \mathrm{clip}(e,0.1,0.9)$.
\STATE Sample $T_i\sim \mathrm{Bernoulli}(e_i)$ independently for $i=1,\dots,n$.
\STATE $\mu_0 \leftarrow 2\sin(X_{:,0})+\max(X_{:,1},0)$.
\STATE $\tau \leftarrow 2/(1+\exp(-X_{:,0})) + X_{:,2}$.
\STATE $\sigma \leftarrow \sigma_{\mathrm{base}}\cdot (1+0.5|X_{:,0}|)$ \hfill (elementwise)
\STATE Draw $\epsilon_0,\epsilon_1\in\mathbb R^n$ with $\epsilon_t(i)\sim\mathcal N(0,\sigma_i^2)$ independently.
\STATE $Y(0)\leftarrow \mu_0+\epsilon_0$; \quad $Y(1)\leftarrow \mu_0+\tau+\epsilon_1$.
\STATE $Y\leftarrow T\odot Y(1)+(1-T)\odot Y(0)$.
\ENSURE $(X,T,Y,\mathrm{ITE\_true}=\tau,\ e_{\mathrm{true}}=e)$
\end{algorithmic}
\end{algorithm}

\textbf{DGP-Setting1 (pseudocode).}
\begin{verbatim}
X ~ N(0, I_d)
logit = 0.5*X[:,0] - 0.5*X[:,2]
e_true = sigmoid(logit);  e_true = clip(e_true, 0.1, 0.9)
T ~ Bernoulli(e_true)

mu0 = 2*sin(X[:,0]) + max(X[:,1], 0)
tau = 2/(1+exp(-X[:,0])) + X[:,2]

noise_scale = sigma_base * (1 + 0.5*abs(X[:,0]))
eps0 ~ Normal(0, noise_scale);  eps1 ~ Normal(0, noise_scale)

y0 = mu0 + eps0
y1 = mu0 + tau + eps1
Y  = T*y1 + (1-T)*y0
return (X,T,Y, ITE_true=tau, e_true)
\end{verbatim}

\subsubsection{Setting 2: Harder overlap + heavy-tailed noise}
\label{app:setting2}

\textbf{DGP.}
Covariates: $X\sim\mathcal N(0,I_d)$.
Harder propensity via nonlinear logit + amplification:
\[
\ell(X)=1.2\sin(X_0)-X_2+0.8X_0X_3,\qquad
e(X)=\mathrm{sigmoid}\big(2\,\ell(X)\big),\quad e(X)\in[0.05,0.95]\ \text{(clipped)}.
\]
Baseline and CATE (richer nonlinear forms):
\[
\mu_0(X)=2\sin(X_0)+|X_1|+0.5X_2+0.5\cos(X_3)+0.3X_0X_4,
\qquad
\tau(X)=\frac{1.5}{1+e^{-X_0}}+\max(X_2,0)+0.5\sin(X_5).
\]
Heavy-tailed heteroskedastic noise:
\[
\epsilon_t=\sigma(X)\,\xi_t,\quad \xi_t\sim t_{\nu},\ \nu=3,\qquad
\sigma(X)=\sigma_{\mathrm{base}}\big(1+0.6|X_0|+0.3|X_6|\big),
\]
with the same sweep $\sigma_{\mathrm{base}}\in\{0.5,1.0,1.5\}$.

\textbf{Why this setting is a ``stress test'' for proxy-based selection.}
Setting~1 isolates heteroskedastic proxy variance under light-tailed noise; Setting~2 deliberately compounds \emph{two} real-world failure modes:
(i) \textbf{harder overlap}, where many propensities approach the clipping bounds so inverse-propensity factors amplify the DR correction;
(ii) \textbf{heavy tails}, where rare but extreme residuals dominate squared losses.
Together these effects make the raw proxy $(\hat\tau-\phi)^2$ not just biased as a ranking signal, but \emph{highly unstable}:
a small fraction of calibration points can exhibit proxy spikes that swamp comparisons, and variance estimation $\hat V(X)$ is itself more fragile.
This is precisely where a denoising method must show it is not merely ``helpful when everything is nice,'' but still produces \emph{usable} selections under adversarial noise.

\textbf{Representative results at the best-performing denoising level.}
Figure~\ref{fig:main_setting2_heavytail} in the main text reports the representative curves at $\rho=0.15$, the best-performing denoising strength in this regime.
The key qualitative takeaway is that DCA restores \textbf{early, nontrivial selection}:
for moderate target FDR levels, DCA yields substantially larger power than the plug-in proxy and uncertainty baselines, while remaining controlled/conservative in realized FDR.
Intuitively, by subtracting a \emph{small amount} of estimated DR variance, DCA prevents ``proxy spikes'' caused by heavy-tail/outlier behavior from dominating the null subset comparisons, so the conformal/BH stage can meaningfully discriminate candidates instead of collapsing toward near-empty selections.

\textbf{Why DCA prefers smaller $\rho$ here (and what this teaches).}
Compared to Setting~1, aggressive subtraction is risky because $\hat V(X)$ is harder to estimate under heavy tails and near-extreme propensities.
Large $\rho$ can over-subtract, shrink many proxy labels toward zero, and blur the calibration separation used to form $\mathcal I_0$ (the ``bad'' subset),
which weakens conformal discrimination.
Empirically, the best-performing value is $\rho=0.15$, reflecting a \textbf{conservative denoising strategy}:
even a modest variance subtraction is enough to reduce the dominance of heavy-tail-driven proxy spikes,
while avoiding over-correction when $\hat V$ is noisy.
This is a useful message for deployment: \emph{DCA does not require large corrections to help; small, stable denoising can already unlock substantial selection gains in difficult regimes.}

\textbf{Sensitivity (robustness rather than knife-edge tuning).}
Figure~\ref{fig:setting2} shows sensitivity over $\rho\in\{0.1,0.15,0.2,0.25,0.5\}$ and $\sigma_{\mathrm{base}}\in\{0.5,1.0,1.5\}$.
Two patterns support the method’s claim:
(i) the optimal region shifts smaller relative to Setting~1, consistent with variance-estimation stress under heavy tails;
(ii) performance degrades smoothly---slightly under-denoising ($\rho=0.1$) trends toward the plug-in behavior,
while over-denoising ($\rho\ge 0.25$) progressively reduces power by washing out the calibration contrast.
The existence of a \emph{stable neighborhood} around $\rho\approx 0.15$ indicates the gains are not due to a brittle hyperparameter coincidence,
but to the intended mechanism: mitigating proxy-noise domination under hard overlap + heavy tails.

\begin{figure}
    \centering
    \includegraphics[width=0.95\textwidth]{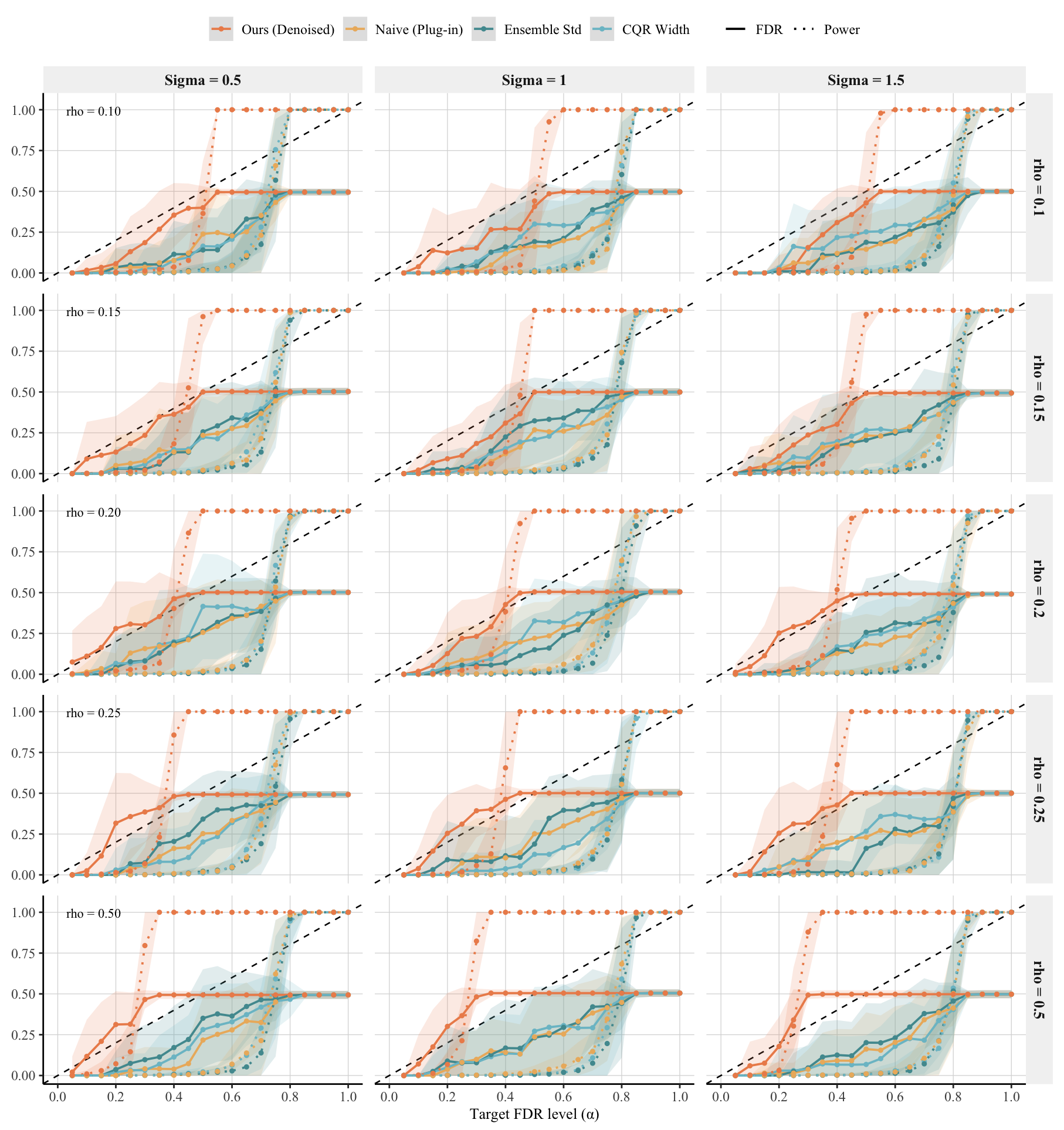}
    \caption{\textbf{Setting 2 sensitivity (hard overlap + heavy tails).}
    Compared to Setting~1, the optimal denoising regime shifts toward smaller $\rho$:
    conservative subtraction improves yield while maintaining controlled/conservative realized FDR,
    whereas overly aggressive subtraction degrades power by over-correcting under noisy $\hat V$ estimation.}
    \label{fig:setting2}
\end{figure}

\begin{algorithm}
\caption{Setting 2 DGP}
\label{alg:dgp_setting2}
\small
\begin{algorithmic}[1]
\REQUIRE $n,d,\sigma_{\mathrm{base}},\nu(=3)$
\STATE Sample $X\sim\mathcal N(0,I_d)$.
\STATE $\ell\leftarrow 1.2\sin(X_{:,0})-X_{:,2}+0.8X_{:,0}\odot X_{:,3}$.
\STATE $e\leftarrow \mathrm{sigmoid}(2\ell)$; $e\leftarrow \mathrm{clip}(e,0.05,0.95)$; sample $T\sim\mathrm{Bernoulli}(e)$.
\STATE $\mu_0\leftarrow 2\sin(X_{:,0})+|X_{:,1}|+0.5X_{:,2}+0.5\cos(X_{:,3})+0.3X_{:,0}\odot X_{:,4}$.
\STATE $\tau\leftarrow 1.5/(1+\exp(-X_{:,0}))+\max(X_{:,2},0)+0.5\sin(X_{:,5})$.
\STATE $\sigma\leftarrow \sigma_{\mathrm{base}}\cdot(1+0.6|X_{:,0}|+0.3|X_{:,6}|)$.
\STATE $\epsilon_0\leftarrow \sigma\odot t_\nu$; $\epsilon_1\leftarrow \sigma\odot t_\nu$ (i.i.d.\ draws).
\STATE $Y(0)\leftarrow \mu_0+\epsilon_0$; $Y(1)\leftarrow \mu_0+\tau+\epsilon_1$; $Y\leftarrow T\odot Y(1)+(1-T)\odot Y(0)$.
\ENSURE $(X,T,Y,\mathrm{ITE\_true}=\tau,\ e_{\mathrm{true}}=e)$
\end{algorithmic}
\end{algorithm}

\textbf{DGP-Setting2 (pseudocode).}
\begin{verbatim}
X ~ N(0, I_d)
ell   = 1.2*sin(X[:,0]) - 1.0*X[:,2] + 0.8*X[:,0]*X[:,3]
e_true = sigmoid(2.0*ell);  e_true = clip(e_true, 0.05, 0.95)
T ~ Bernoulli(e_true)

mu0 = 2*sin(X[:,0]) + abs(X[:,1]) + 0.5*X[:,2] + 0.5*cos(X[:,3]) + 0.3*X[:,0]*X[:,4]
tau = 1.5/(1+exp(-X[:,0])) + max(X[:,2],0) + 0.5*sin(X[:,5])

noise_scale = sigma_base * (1 + 0.6*abs(X[:,0]) + 0.3*abs(X[:,6]))
eps0 = noise_scale * t_df(df=3)
eps1 = noise_scale * t_df(df=3)

y0 = mu0 + eps0
y1 = mu0 + tau + eps1
Y  = T*y1 + (1-T)*y0
return (X,T,Y, ITE_true=tau, e_true)
\end{verbatim}


\subsubsection{Setting 3: Covariate shift (mean / scale / mixture) + weighted conformal $p$-values}
\label{app:setting3}

\textbf{Source vs.\ target covariates (shift families).}
To match the bounded-likelihood-ratio assumption used by weighted conformal calibration,
we sample covariates from \emph{truncated} Gaussians with bounded support.
Let
\[
\mathcal X_B:=\{x\in\mathbb R^d:\ \|x\|_\infty \le B\},
\]
and define $\mathcal N_{\mathcal X_B}(\mu,\Sigma)$ as a Gaussian $\mathcal N(\mu,\Sigma)$
conditioned on $X\in\mathcal X_B$.
We generate a reference/source covariate distribution
\[
X_{\mathrm{src}}\sim \mathcal N_{\mathcal X_B}(0,I_d),
\]
and a target covariate distribution obtained by one of the following parameterized shifts
(with $\Delta=0.8$ and $B=3$ in the shown runs):
\begin{itemize}
\item \textbf{Mean shift:} $X_{\mathrm{tgt}}\sim\mathcal N_{\mathcal X_B}(\mu, I_d)$ where $\mu_0=\Delta,\ \mu_1=-0.5\Delta$, and other coordinates have mean $0$.
\item \textbf{Scale shift:} $X_{\mathrm{tgt}}\sim\mathcal N_{\mathcal X_B}(0,\Sigma)$ where $\Sigma=\mathrm{diag}((1+\Delta)^2,(1+\Delta)^2,1,\ldots,1)$.
\item \textbf{Mixture shift:} $X_{\mathrm{tgt}}\sim 0.7\,\mathcal N_{\mathcal X_B}(0,I_d)+0.3\,\mathcal N_{\mathcal X_B}(\mu',I_d)$ where $\mu'_0=\Delta,\ \mu'_2=0.5\Delta$.
\end{itemize}
With bounded support $\mathcal X_B$, the density ratio $w(x)=dP_{\rm tgt}/dP_{\rm src}(x)$ is bounded on $\mathcal X_B$.

\textbf{Structural mechanism is invariant given $X$ (same as Setting~1).}
\textbf{Crucially, the conditional data-generating mechanism $Y\mid(X,T)$ is identical across source and target.}
Given any realized $X$ (source or target),
\[
e(X)=\mathrm{sigmoid}(0.5X_0-0.5X_2)\ \text{clipped to }[0.1,0.9],\quad
\mu_0(X)=2\sin(X_0)+\max(X_1,0),\quad
\tau(X)=\frac{2}{1+e^{-X_0}}+X_2,
\]
and $\epsilon_t\sim\mathcal N(0,\sigma^2(X))$ with $\sigma(X)=\sigma_{\mathrm{base}}(1+0.5|X_0|)$.
Therefore the shift is \emph{purely marginal} ($P_{\mathrm{src}}(X)\neq P_{\mathrm{tgt}}(X)$) while the causal mechanism remains unchanged.

\textbf{Why this setting matters (and what ``drop-in'' weighted conformal buys us).}
This setting isolates the covariate-shift failure mode: even if the proxy construction and the learned score $s(x)=g(x)$ are unchanged,
\emph{unweighted} conformal calibration implicitly treats calibration samples as if they were drawn from the same $X$-distribution as test candidates.
When $P_{\mathrm{src}}(X)\neq P_{\mathrm{tgt}}(X)$, this breaks the distributional comparison underlying conformal $p$-values and can distort FDR.
Our fix is deliberately minimal and modular: we keep \textbf{every other component identical} (proxy construction, $g$ training, BH),
and \textbf{only} replace conformal counts by importance-weighted sums. This is why we call it a \emph{drop-in} modification.

\textbf{Weights.}
Because $P_{\mathrm{src}}(X)$ and $P_{\mathrm{tgt}}(X)$ are explicitly defined, the code computes the exact density ratio
$w(x)=dP_{\mathrm{tgt}}/dP_{\mathrm{src}}(x)$ in closed form (Gaussian or mixture log-densities) for the chosen shift type,
followed by stability clipping:
\[
w(x)\leftarrow \mathrm{clip}(w(x),\,w_{\min},w_{\max}),\qquad (w_{\min},w_{\max})=(0.05,20).
\]
(When densities are unknown, one may estimate $w$ via a domain classifier; see the main paper for that general recipe.)

\textbf{Weighted conformal $p$-values.}
Let $s(\cdot)$ be the predicted alignment score (output of $g$; smaller means more reliable).
Define the calibration ``bad/null'' subset using the proxy threshold on the squared scale,
\[
\mathcal I_0=\Big\{i\in\mathcal D_{\mathrm{cal}}:\ \check A_i\ge c\Big\}.
\]
Given calibration weights $\{w_i\}_{i\in\mathcal D_{\mathrm{cal}}}$ and a test weight $w_j$,
the code computes the weighted conformal $p$-value
\[
p_j
=
\frac{w_j+\sum_{i\in\mathcal I_0} w_i\,\mathbf 1\{s(X_i)\le s(X_j)\}}
{\sum_{i\in\mathcal D_{\mathrm{cal}}} w_i + w_j}.
\]
This is exactly the ``drop-in'' replacement of counts by weighted sums referenced beneath Algorithm~\ref{alg:causal_fdr}.

\textbf{Why this is an asymptotic extension rather than a finite-sample BH theorem.}
The weighted conformal literature shows that importance weighting corrects the calibration distribution under covariate shift, but the resulting weighted $p$-values are not automatically endowed with the same dependency structure as the unweighted split-conformal case.
In particular, the classical BH guarantee requires a condition such as PRDS or an equivalent positive dependence property on the null $p$-values; weighted conformal $p$-values need not satisfy this condition in general.
Accordingly, our use of BH in the covariate-shift setting should be read as an asymptotic, plug-in deployment rule: weighting corrects the calibration mismatch at the level of the conformal comparison, while asymptotic FDR control still relies on the same proxy-stability and sample-size regime used elsewhere in the paper.
For readers interested in finite-sample control under covariate shift, the right formal object is Weighted Conformalized Selection (WCS), which augments weighted conformal p-values with a leave-one-in pruning step to handle the induced dependence.
We do not re-develop WCS in the main pipeline because it would substantially change the selection procedure; instead, we keep the paper’s core DCA wrapper fixed and use weighted conformal p-values as the minimal drop-in correction for experiments.

\textbf{Tie handling under truncation in the weighted case.}
The truncation $(\widetilde A^2-\rho\widehat V)_+$ can create many exact zeros, so tie handling matters here as well.
The tie-safe construction in Appendix~\ref{app:randomized_ties} applies the same randomized lexicographic ranking idea to the score comparisons.
Operationally, the weighted covariate-shift experiments use the same practical safeguard as the unweighted case:
when needed, we compare standard, jittered, and smoothed tie-breaking variants and obtain qualitatively similar curves
(see Figure~\ref{fig:tie_breaking}).
Thus, the weighted experiments are not driven by a pathological pile-up at zero.

\textbf{Results and representative denoising choice.}
Figure~\ref{fig:main_setting3_shift} in the main text reports a representative covariate-shift run at $\rho=0.60$,
and Figure~\ref{fig:setting3} reports sensitivity over the same $\rho$ grid as Setting~1.
\emph{What Fig.~\ref{fig:main_setting3_shift} illustrates:}
even under a pure marginal shift, the weighted conformal step preserves the intended FDR behavior while allowing DCA to retain its power advantage,
i.e., we do not have to trade off distribution-shift robustness against selection yield.
\emph{What Fig.~\ref{fig:setting3} adds:}
the improvement is stable across shift types/noise levels and does not hinge on a brittle hyperparameter choice;
moderate denoising continues to help, while the weighted calibration prevents the miscalibration that unweighted conformal would incur.

\begin{figure}[t]
    \centering
    \includegraphics[width=\textwidth]{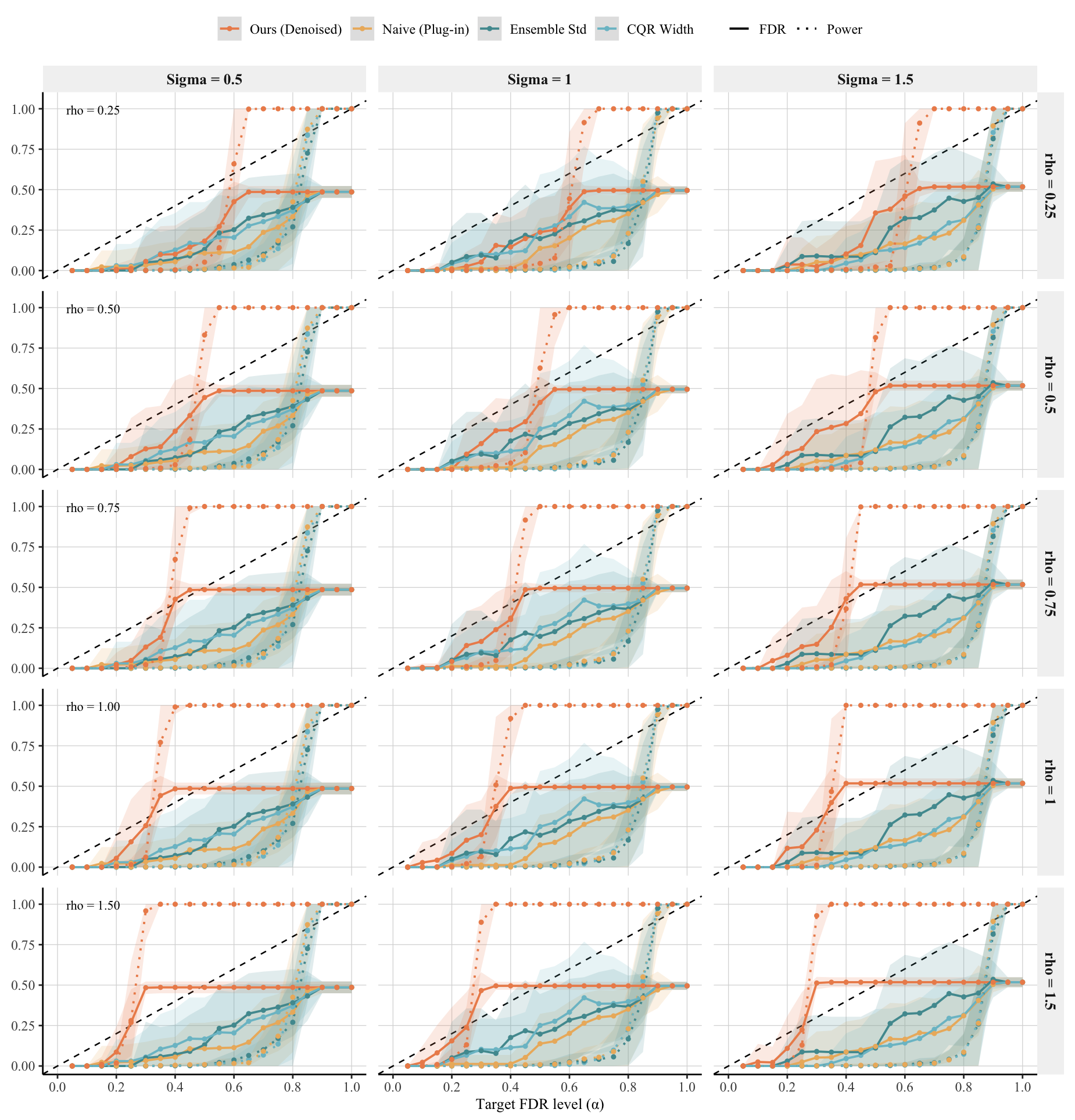}
    \caption{\textbf{Setting 3 sensitivity (covariate shift).}
    Across shift types and noise scales, importance weighting stabilizes calibration under $P_{\mathrm{src}}(X)\neq P_{\mathrm{tgt}}(X)$,
    while moderate denoising strengths continue to improve power without inducing fragile tuning behavior.}
    \label{fig:setting3}
\end{figure}

\begin{algorithm}
\caption{Setting 3 covariate shift + exact density ratio}
\label{alg:dgp_setting3}
\small
\begin{algorithmic}[1]
\REQUIRE shift type in \{\texttt{mean},\texttt{scale},\texttt{mix}\}, magnitude $\Delta(=0.8)$
\STATE Sample source $X_{\mathrm{src}}\sim\mathcal N(0,I_d)$.
\STATE Sample target $X_{\mathrm{tgt}}$ by the selected shift rule (mean/scale/mixture) with magnitude $\Delta$.
\STATE Generate $(T,Y,\tau)$ on both source/target using the Setting~1 structural mechanism conditional on $X$.
\STATE Compute $w(x)=q(x)/p(x)$ exactly from the known $p,q$ (Gaussian or mixture log-densities); clip to $(0.05,20)$.
\ENSURE source data $(X_{\mathrm{src}},T_{\mathrm{src}},Y_{\mathrm{src}},\tau_{\mathrm{src}})$, target candidates $(X_{\mathrm{tgt}},T_{\mathrm{tgt}},Y_{\mathrm{tgt}},\tau_{\mathrm{tgt}})$, weights
\end{algorithmic}
\end{algorithm}

\textbf{DGP-Setting3 (pseudocode).}
\begin{verbatim}
# source:
X_src ~ N(0, I_d)
# target (choose one):
if shift_type=="mean":
    mu = 0; mu[0]=Delta; mu[1]=-0.5*Delta
    X_tgt ~ N(mu, I_d)
if shift_type=="scale":
    Z ~ N(0, I_d); X_tgt=Z; X_tgt[:,0:2] *= (1+Delta)
if shift_type=="mix":
    with prob 0.7: X_tgt ~ N(0, I_d)
    with prob 0.3: X_tgt ~ N(mu', I_d), mu'[0]=Delta, mu'[2]=0.5*Delta

# structural (applied to both X_src and X_tgt):
logit = 0.5*X[:,0] - 0.5*X[:,2]
e_true = sigmoid(logit); e_true = clip(e_true,0.1,0.9)
T ~ Bernoulli(e_true)
mu0 = 2*sin(X[:,0]) + max(X[:,1],0)
tau = 2/(1+exp(-X[:,0])) + X[:,2]
noise_scale = sigma_base*(1+0.5*abs(X[:,0]))
eps0,eps1 ~ Normal(0, noise_scale)
Y = T*(mu0+tau+eps1) + (1-T)*(mu0+eps0)
\end{verbatim}

\subsubsection{Setting 4: Covariance / rotation shift + weighted conformal $p$-values}
\label{app:setting4}

\textbf{Source vs.\ target covariates (geometry shift families).}
As in Setting~3, we sample $X$ from truncated Gaussians supported on $\mathcal X_B=\{x:\|x\|_\infty\le B\}$
to ensure a bounded density ratio for importance weighting.
Reference covariates remain $X_{\mathrm{src}}\sim\mathcal N(0,I_d)$.
Target covariates follow $X_{\mathrm{tgt}}\sim\mathcal N(0,\Sigma_{\mathrm{tgt}})$ where $\Sigma_{\mathrm{tgt}}$ is constructed by one of two \emph{geometry} shifts (with $\Delta=0.6$ in the shown runs):

\begin{itemize}
\item \textbf{Covariance shift (\texttt{cov}):} inflate variances of the first few coordinates and induce correlations among them.
Concretely, set $v_0=1+1.5\Delta,\ v_1=1+1.0\Delta,\ v_2=1+0.8\Delta$ (others $=1$) and start with $\Sigma=\mathrm{diag}(v)$; then set
\[
\Sigma_{01}=\Sigma_{10}=0.35\Delta\sqrt{v_0v_1},\quad
\Sigma_{02}=\Sigma_{20}=0.25\Delta\sqrt{v_0v_2},\quad
\Sigma_{12}=\Sigma_{21}=0.20\Delta\sqrt{v_1v_2}.
\]
Finally, take $\Sigma_{\mathrm{tgt}}=\Sigma$.
\item \textbf{Rotation shift (\texttt{rot}):} construct anisotropic eigenvalues
$\lambda_0=1+2.0\Delta,\ \lambda_1=1+1.2\Delta,\ \lambda_2=1+0.8\Delta,\ \lambda_3=1+0.5\Delta$ (others $=1$),
let $D=\mathrm{diag}(\lambda)$, and set $\Sigma_{\mathrm{tgt}}=RDR^\top$ with a fixed random orthogonal $R$ (seeded QR).
\end{itemize}
Our representative figure uses \texttt{rot} with $\Delta=0.6$.

\textbf{Structural mechanism is invariant given $X$ (same as Setting~1).}
\textbf{Crucially, the conditional mechanism $Y\mid(X,T)$ is unchanged across environments:}
given any realized $X$ (source or target), we reuse the \emph{Setting~1} causal mechanism:
\[
e(X)=\mathrm{sigmoid}(0.5X_0-0.5X_2)\ \text{clipped to }[0.1,0.9],\quad
\mu_0(X)=2\sin(X_0)+\max(X_1,0),\quad
\tau(X)=\frac{2}{1+e^{-X_0}}+X_2,
\]
and $\epsilon_t\sim\mathcal N(0,\sigma^2(X))$ with $\sigma(X)=\sigma_{\mathrm{base}}(1+0.5|X_0|)$.
Therefore the shift is still \emph{purely covariate shift} (only $P(X)$ changes), but in a way that is \emph{not} explainable by a simple mean translation.

\textbf{Why this is a meaningful stress test (beyond mean/scale).}
Unlike Setting~3, here the target distribution can keep mean $0$ while changing \emph{geometry} (correlations / rotations / anisotropy).
This stresses calibration because exchangeability breaks through changes in feature geometry rather than location.
It therefore tests whether the weighted conformal step is genuinely robust to covariate shift---including subtler shifts that can look benign marginally but still change where probability mass concentrates in high-dimensional space.

\textbf{Weights (exact MVN density ratio).}
Because both $P_{\mathrm{src}}(X)=\mathcal N(0,I_d)$ and $P_{\mathrm{tgt}}(X)=\mathcal N(0,\Sigma_{\mathrm{tgt}})$ are known,
the code computes the exact density ratio
\[
w(x)=\frac{\phi_{0,\Sigma_{\mathrm{tgt}}}(x)}{\phi_{0,I}(x)}
\]
using log-density evaluations (Cholesky-based) for numerical stability, and then clips
\[
w(x)\leftarrow \mathrm{clip}(w(x),\,0.05,\,20).
\]
These weights are then used in the same weighted conformal $p$-value formula as in Setting~3 (drop-in replacement of counts by weighted sums).

\textbf{Results and representative denoising choice (appendix).}
Figure~\ref{fig:setting4best} reports a representative geometry-shift run at $\rho=0.55$,
and Figure~\ref{fig:setting4} reports sensitivity over the same $\rho$ grid as Setting~1.
\emph{What Fig.~\ref{fig:setting4best} illustrates:}
even when the shift is purely geometric (rotation/covariance) rather than mean/scale, the weighted conformal step continues to stabilize calibration,
so DCA can translate denoising gains into actual selection yield under the intended FDR target.
\emph{What Fig.~\ref{fig:setting4} adds:}
the behavior is not an artifact of a specific shift instance; performance remains stable across noise levels and $\rho$ values in a moderate band,
with graceful degradation when $\rho$ is too small (reverting toward the noisy proxy ranking) or too large (over-subtraction under imperfect variance estimation).

\begin{figure}[t]
    \centering
    \includegraphics[width=0.8\textwidth]{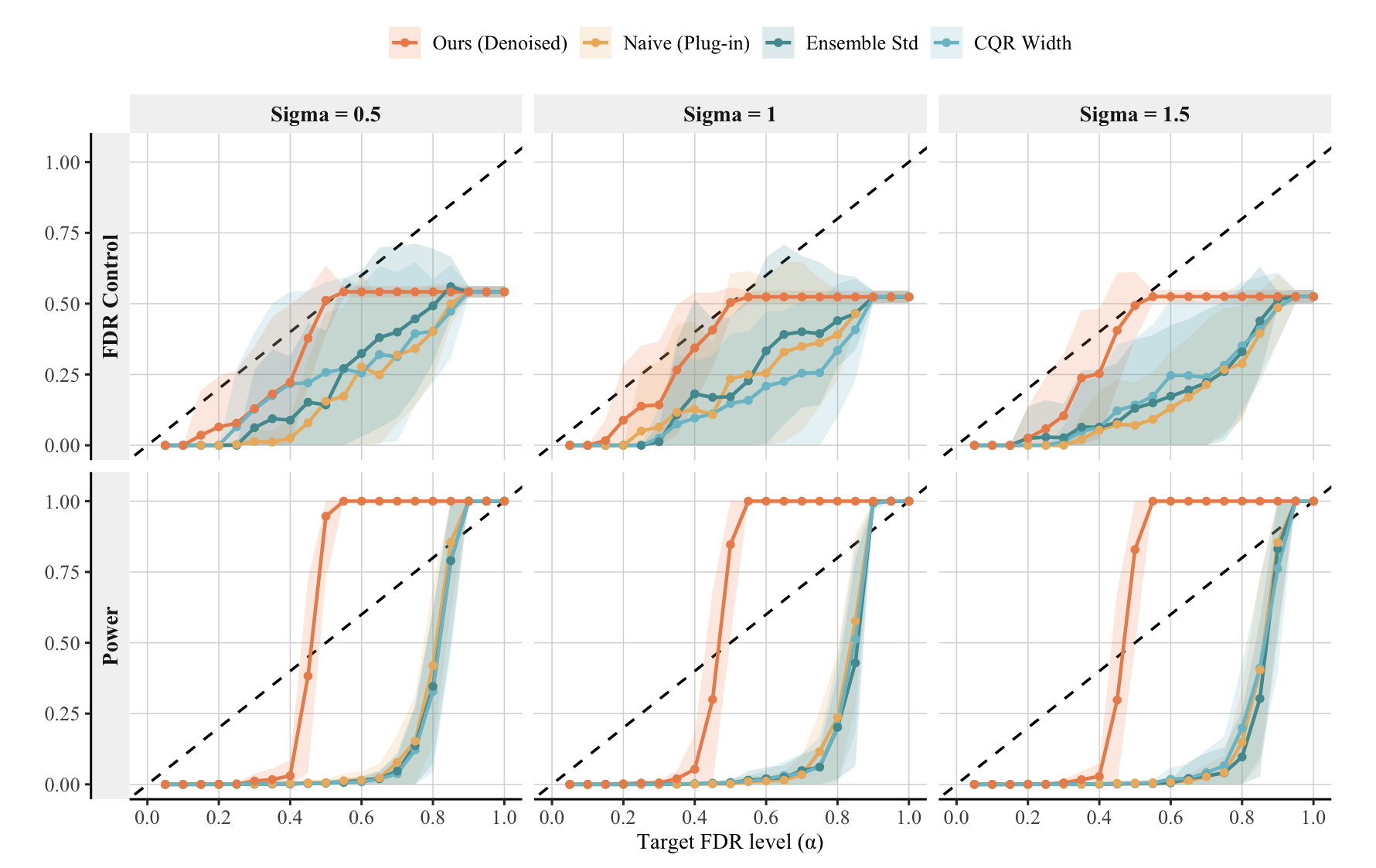}
    \caption{\textbf{Setting 4 representative (covariance/rotation shift).}
    Under geometry-only covariate shift with invariant $Y\mid(X,T)$, importance-weighted conformal calibration remains a drop-in fix:
    it stabilizes FDR behavior while allowing DCA to preserve higher selection yield than unweighted/uncertainty-style baselines.}
    \label{fig:setting4best}
\end{figure}

\begin{figure}[t]
    \centering
    \includegraphics[width=\textwidth]{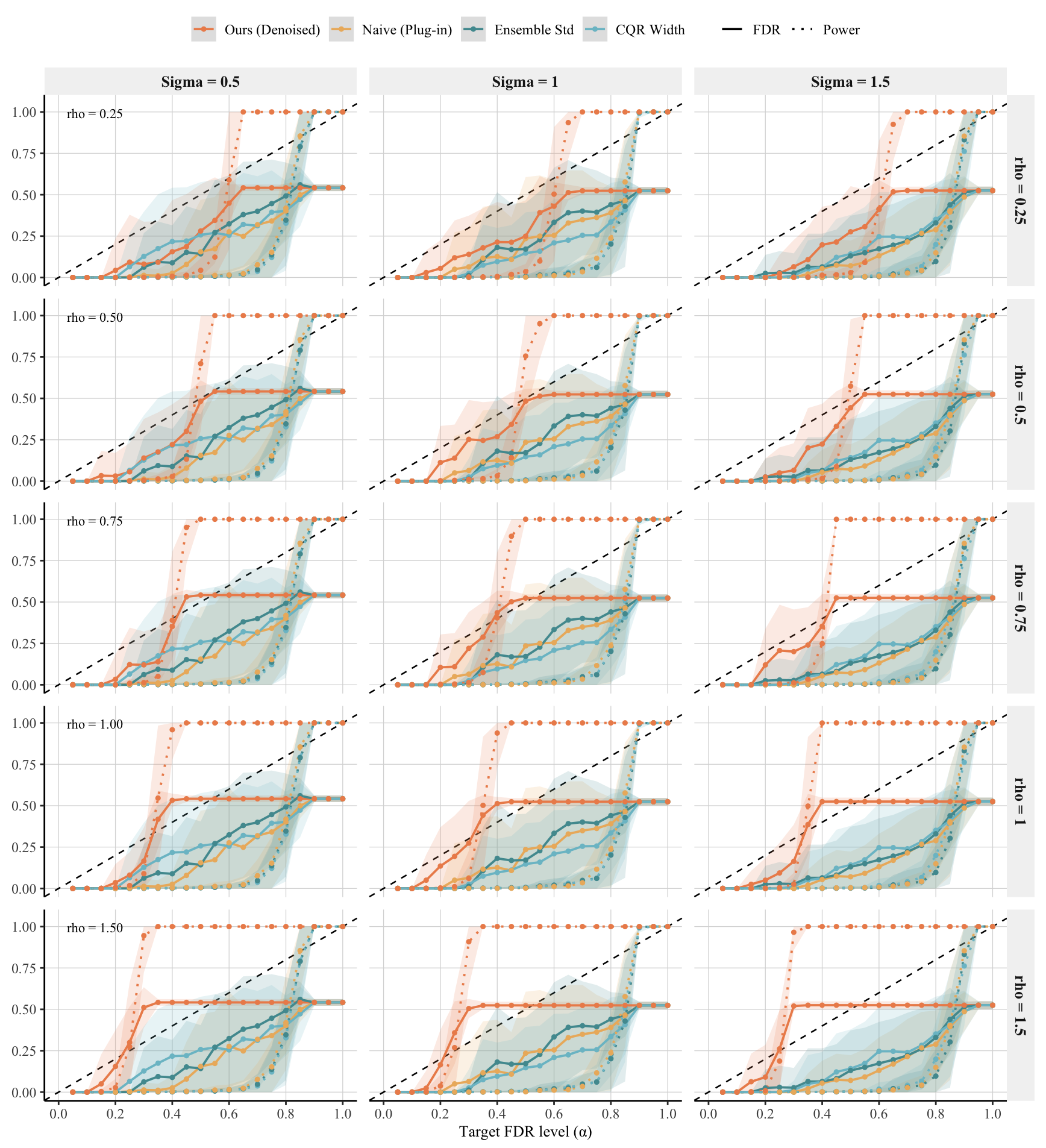}
    \caption{\textbf{Setting 4 sensitivity (covariance/rotation shift).}
    Across geometry-shift constructions and noise scales, weighted $p$-values maintain calibration under $P_{\mathrm{src}}(X)\neq P_{\mathrm{tgt}}(X)$,
    while moderate denoising continues to improve power without relying on fragile tuning.}
    \label{fig:setting4}
\end{figure}

\begin{algorithm}
\caption{Setting 4 covariance/rotation shift + exact MVN density ratio}
\label{alg:dgp_setting4}
\small
\begin{algorithmic}[1]
\REQUIRE shift type in \{\texttt{cov},\texttt{rot}\}, magnitude $\Delta(=0.6)$
\STATE Sample source $X_{\mathrm{src}}\sim\mathcal N(0,I_d)$.
\STATE Build $\Sigma_{\mathrm{tgt}}$ by the selected rule (\texttt{cov} or \texttt{rot}) with $\Delta$.
\STATE Sample target $X_{\mathrm{tgt}}\sim\mathcal N(0,\Sigma_{\mathrm{tgt}})$.
\STATE Generate $(T,Y,\tau)$ on both source/target using the Setting~1 structural mechanism conditional on $X$.
\STATE Compute $w(x)=\varphi_{0,\Sigma_{\mathrm{tgt}}}(x)/\varphi_{0,I}(x)$ via log-densities; clip to $(0.05,20)$.
\ENSURE source data, target candidates, weights
\end{algorithmic}
\end{algorithm}

\textbf{DGP-Setting4 (pseudocode).}
\begin{verbatim}
# source covariates:
X_src ~ N(0, I_d)

# target covariates (choose shift family):
if shift_type=="cov":
    v = ones(d)
    v[0]=1+1.5*Delta; v[1]=1+1.0*Delta; v[2]=1+0.8*Delta
    Sigma = diag(v)
    Sigma[0,1]=Sigma[1,0]=0.35*Delta*sqrt(v[0]*v[1])
    Sigma[0,2]=Sigma[2,0]=0.25*Delta*sqrt(v[0]*v[2])
    Sigma[1,2]=Sigma[2,1]=0.20*Delta*sqrt(v[1]*v[2])
    Sigma_tgt = Sigma
if shift_type=="rot":
    lam = ones(d)
    lam[0]=1+2.0*Delta; lam[1]=1+1.2*Delta
    lam[2]=1+0.8*Delta; lam[3]=1+0.5*Delta
    D = diag(lam)
    R = random_orthogonal(seed=SEED)   # QR-based
    Sigma_tgt = R @ D @ R.T
X_tgt ~ N(0, Sigma_tgt)

# weights (exact MVN ratio in log space, then clip):
logw(x) = logphi_{0,Sigma_tgt}(x) - logphi_{0,I}(x)
w = exp(logw);  w = clip(w, 0.05, 20)

# structural mechanism (applied to both X_src and X_tgt): same as Setting 1
\end{verbatim}

\subsection{Robustness and Ablation Studies}
\label{app:robustness_ablations}

The preceding subsection gives the main synthetic settings and their sensitivity plots. We now collect
additional ablations that vary one component of the pipeline at a time. These experiments are organized
around the corresponding part of the common protocol: sample splitting, denoising/tolerance choices,
tie handling, estimator/proxy compatibility, selection paradigms, and nuisance-estimation stress.

\subsubsection{Sample splitting, calibration size, and cross-fitting}
\label{app:robustness_splitting}

The strict split in Appendix~\ref{app:protocol_splitting} is our validity-first baseline: it separates
nuisance/variance estimation, alignment training, conformal calibration, and final selection.
Figure~\ref{fig:sample_size_scaling} varies calibration size and shows that DCA retains useful power even
when the calibration pool is smaller, so the gains are not an artifact of an unusually large or favorable
calibration split.
Figure~\ref{fig:crossfit_scaling} compares strict splitting with cross-fitted variants across total sample sizes.
Moderate cross-fitting can stabilize selections when the aggregated $p$-values preserve conditionally null
exchangeability, whereas overly aggressive data reuse can become conservative.
These results support using strict splitting as a clean baseline while leaving data-efficient cross-fitted DCA
as a natural extension.

\begin{figure}[t]
    \centering
    \includegraphics[width=\textwidth]{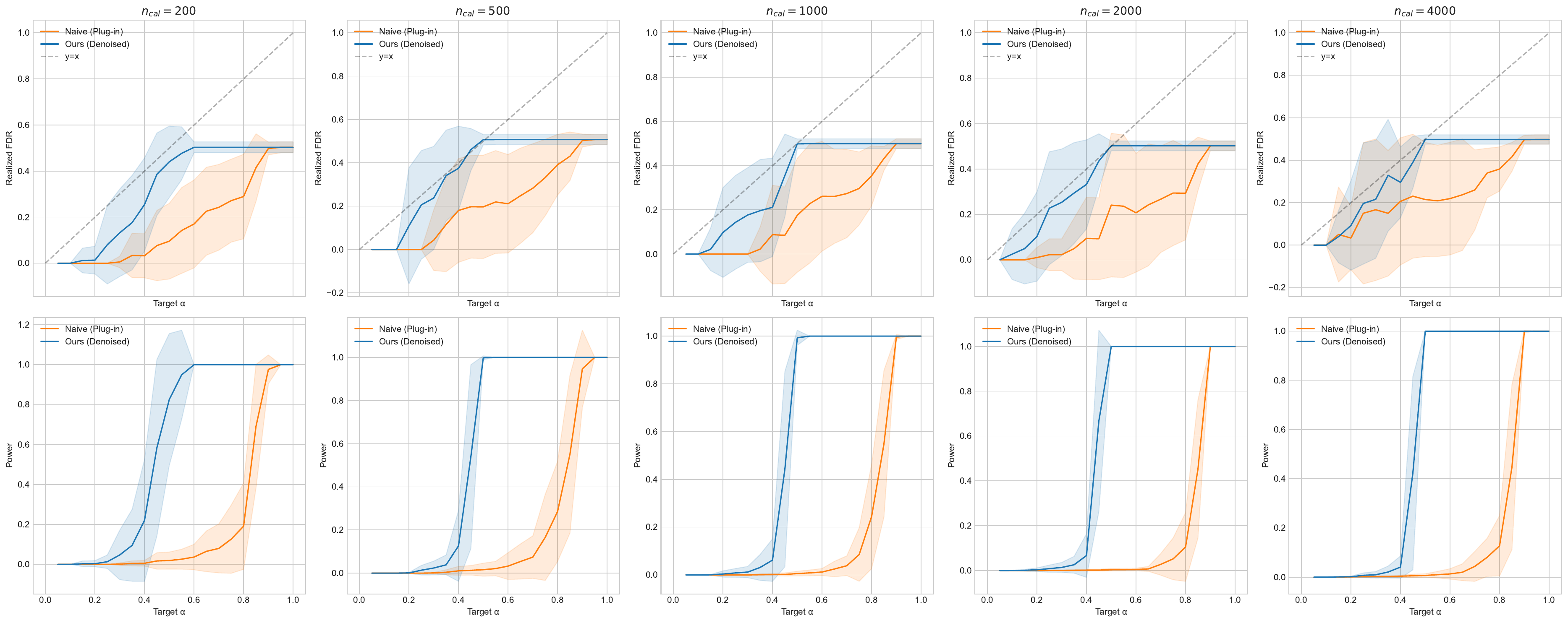}
    \caption{\textbf{Calibration-size scaling.}
    DCA remains useful across calibration sizes, indicating that the baseline sample splitting cost does not erase the benefit of denoising.}
    \label{fig:sample_size_scaling}
\end{figure}

\begin{figure}[t]
    \centering
    \includegraphics[width=\textwidth]{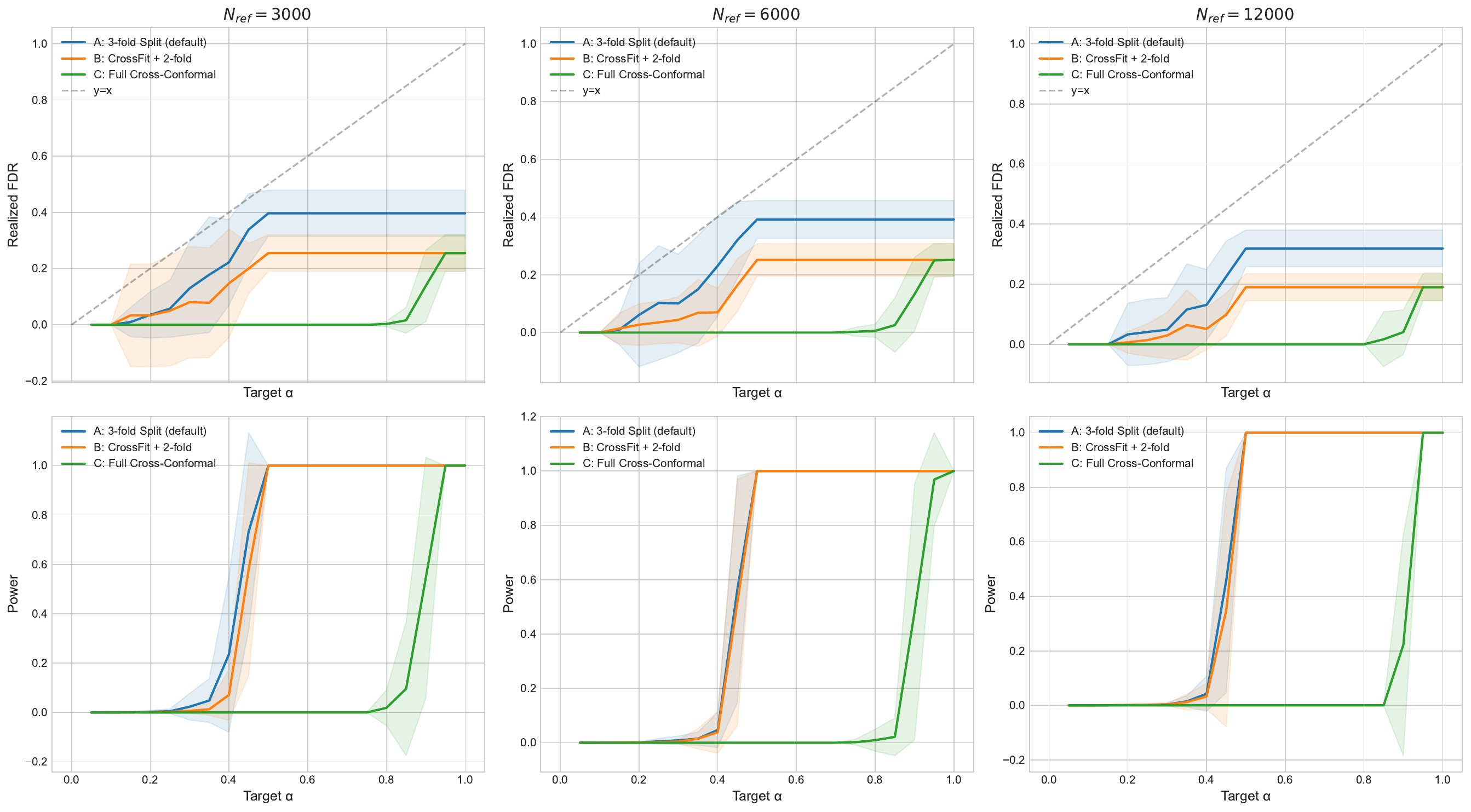}
    \caption{\textbf{Strict splitting versus cross-fitting.}
    Cross-fitted variants across sample sizes show that moderate data reuse can stabilize selections without removing the denoising gain, while strict splitting remains the clean validity-first baseline.}
    \label{fig:crossfit_scaling}
\end{figure}

\subsubsection{Denoising strength, overlap, variance misspecification, and tolerance}
\label{app:robustness_denoising_tolerance}

\textbf{Ablation of denoising strength.}
Panel (a) of Figure~\ref{fig:main_denoising_ablation_misspec} compares no denoising ($\rho=0$), moderate denoising ($\rho=0.65$), full subtraction ($\rho=1.0$), and over-subtraction ($\rho=2.0$) in Setting~1 at $\sigma=1.0$.
Moderate denoising gives the best safety--power trade-off: $\rho=0$ leaves the DR proxy noise-dominated, while aggressive subtraction can over-correct when $\widehat V$ is noisy or when many proxy scores are truncated at zero.
This supports tuning $\rho$ for stable boundary labels rather than simply setting $\rho=1$.

\textbf{Overlap-dependent denoising.}
Figure~\ref{fig:rho_overlap_heatmap} varies overlap quality and $\rho$ at a fixed target level.
When overlap is good, more aggressive denoising can improve separability; when overlap deteriorates, inverse-propensity noise makes $\widehat V$ less stable and smaller $\rho$ is preferred.
This is the empirical counterpart of the boundary-stability argument in Proposition~\ref{prop:finite_sample_approx_fdr}: denoising affects FDR through label flips near the tolerance boundary.

\begin{figure}[t]
    \centering
    \includegraphics[width=0.75\textwidth]{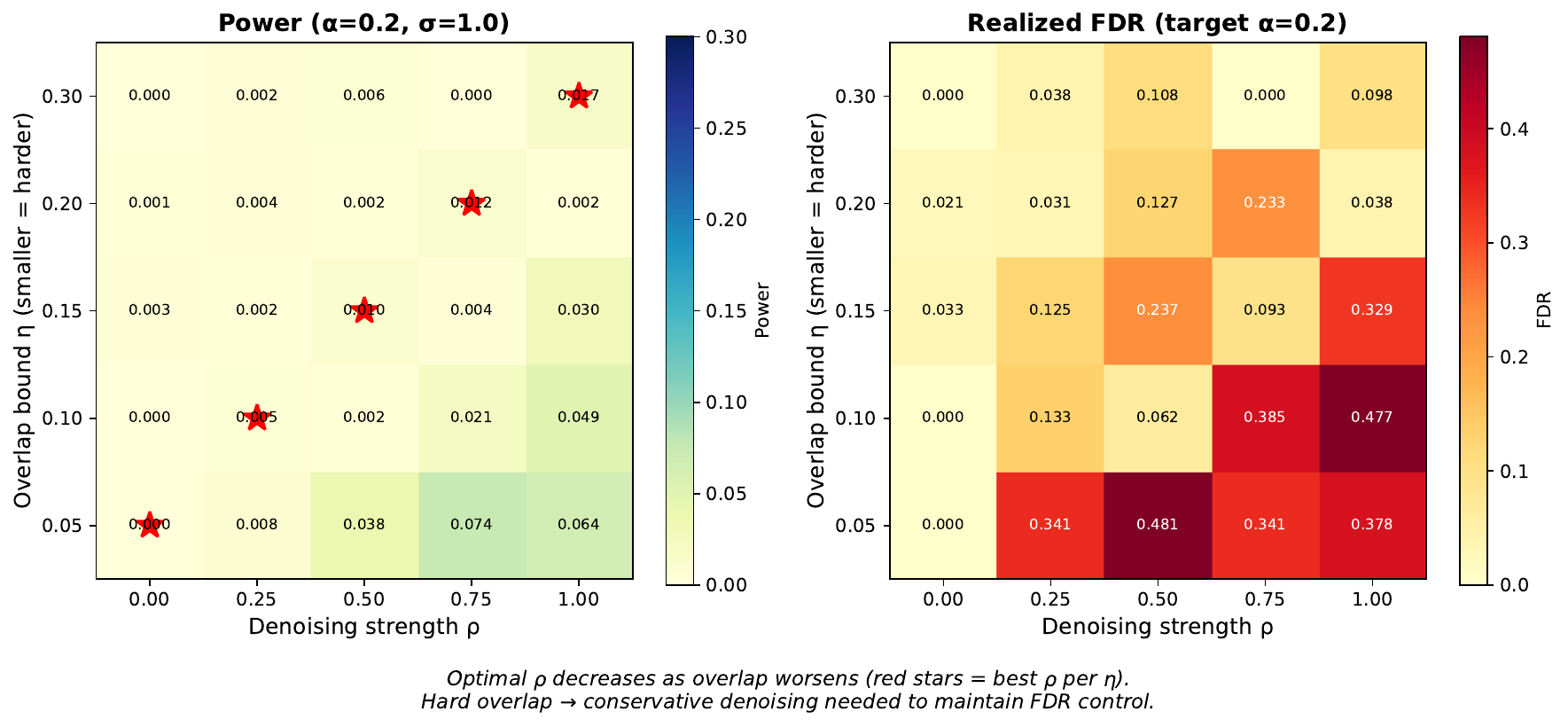}
    \caption{\textbf{$\rho$--overlap interaction.}
    FDR and power over denoising strength and overlap quality in Setting~1.
    Harder overlap shifts the best denoising region toward smaller $\rho$, consistent with conservative subtraction under noisier variance estimates.}
    \label{fig:rho_overlap_heatmap}
\end{figure}

\textbf{Variance misspecification.}
Panel (b) of Figure~\ref{fig:main_denoising_ablation_misspec} perturbs the variance model through well-specified, misspecified, under-estimated, and over-estimated variants.
DCA degrades smoothly rather than catastrophically: imperfect $\widehat V$ mainly reduces power or shifts the preferred denoising level, while realized FDR remains governed by calibration-label stability.
Thus, exact variance recovery is a sufficient route to asymptotic exactness, but the finite-sample object that matters is the induced proxy/oracle mislabeling rate.

\textbf{Tolerance sensitivity.}
Figures~\ref{fig:c_quantile_sensitivity}--\ref{fig:c_absolute_sensitivity} vary the reliability tolerance $c$ using calibration-error quantiles and absolute thresholds.
Across these choices, DCA preserves the same qualitative pattern: denoising raises selection yield while keeping realized FDR near or below the target line.
This confirms that the method is not tied to a single arbitrary tolerance, although $c$ remains an application-level acceptability threshold.

\begin{figure}[t]
    \centering
    \includegraphics[width=0.9\textwidth]{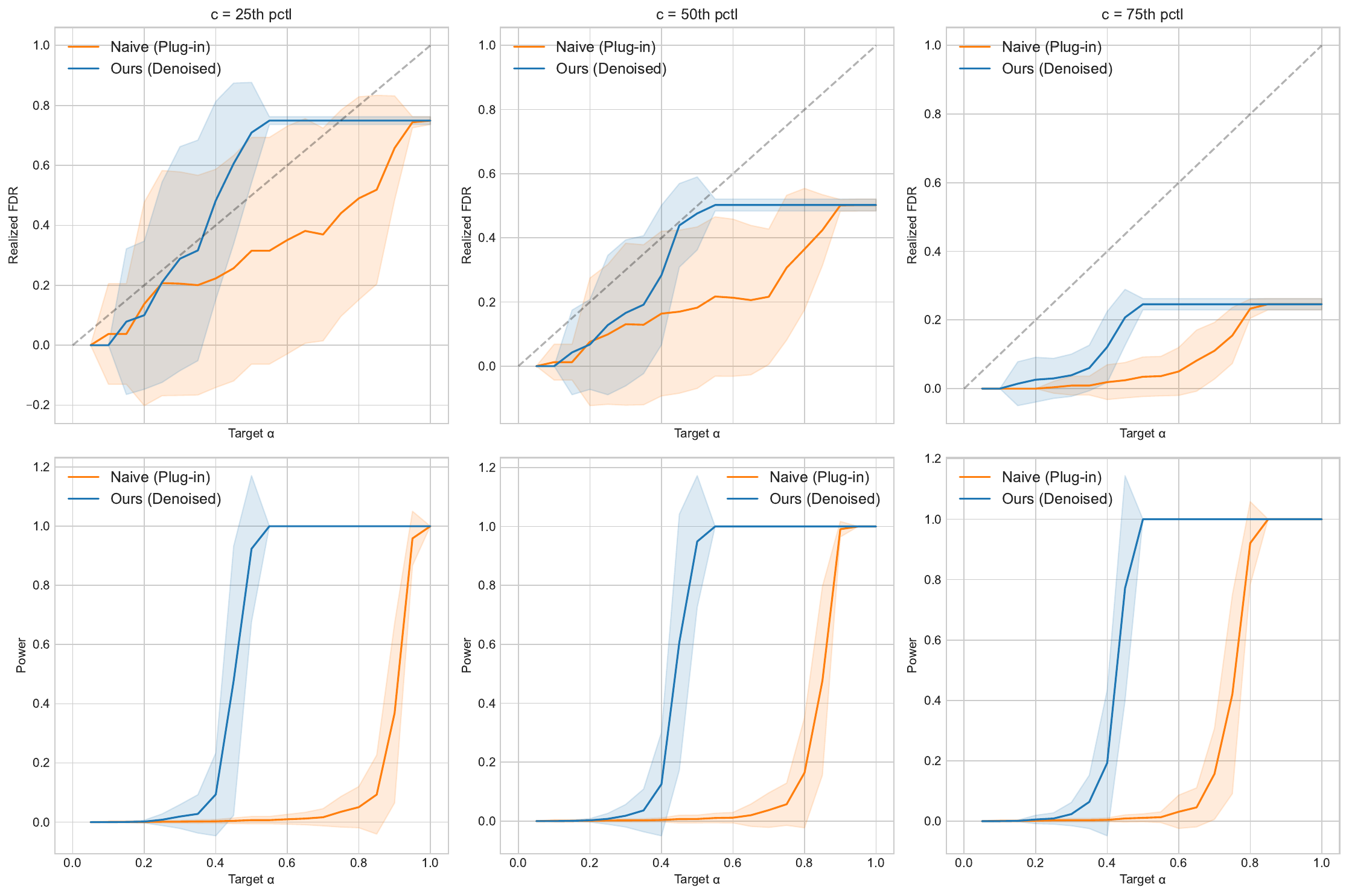}
    \caption{\textbf{Tolerance sensitivity using calibration-error quantiles.}
    Varying $c$ over calibration-side quantiles preserves the same safety--power pattern, indicating that DCA is not a knife-edge tolerance choice.}
    \label{fig:c_quantile_sensitivity}
\end{figure}

\begin{figure}[t]
    \centering
    \includegraphics[width=0.75\textwidth]{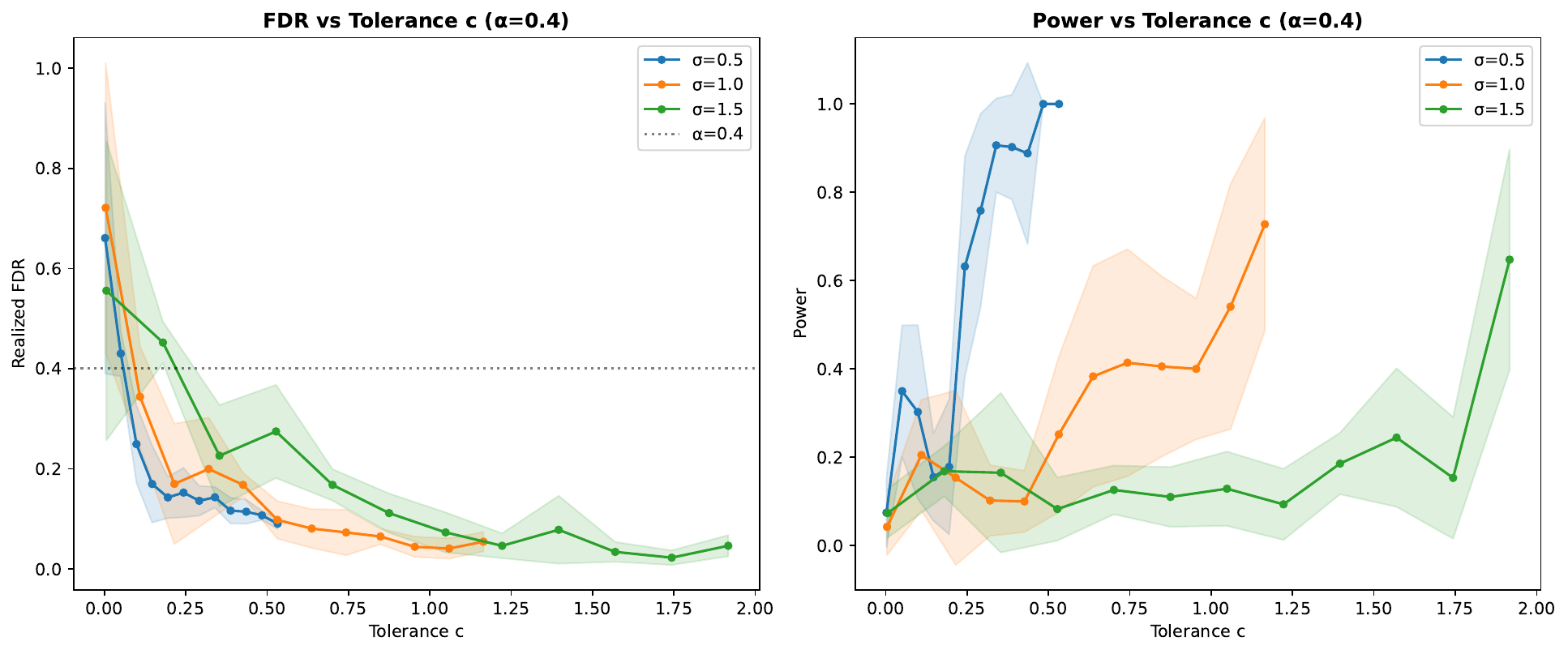}
    \caption{\textbf{Tolerance sensitivity using absolute thresholds.}
    Across noise levels and absolute choices of $c$, denoising improves selection yield while maintaining stable realized FDR.}
    \label{fig:c_absolute_sensitivity}
\end{figure}

\subsubsection{Tie handling under truncation}
\label{app:robustness_ties}

Because $\check A^2=(\widetilde A^2-\rho\widehat V)_+$ can produce ties at zero, the formal tie-safe implementation uses the randomized lexicographic ranks defined in Appendix~\ref{app:randomized_ties}.
We also compare standard, jittered, and smooth tie-breaking variants empirically.
Figure~\ref{fig:tie_breaking} shows similar FDR/power behavior across tie-handling choices, indicating that the empirical results are not driven by deterministic truncation ties.

\begin{figure}[t]
    \centering
    \includegraphics[width=\textwidth]{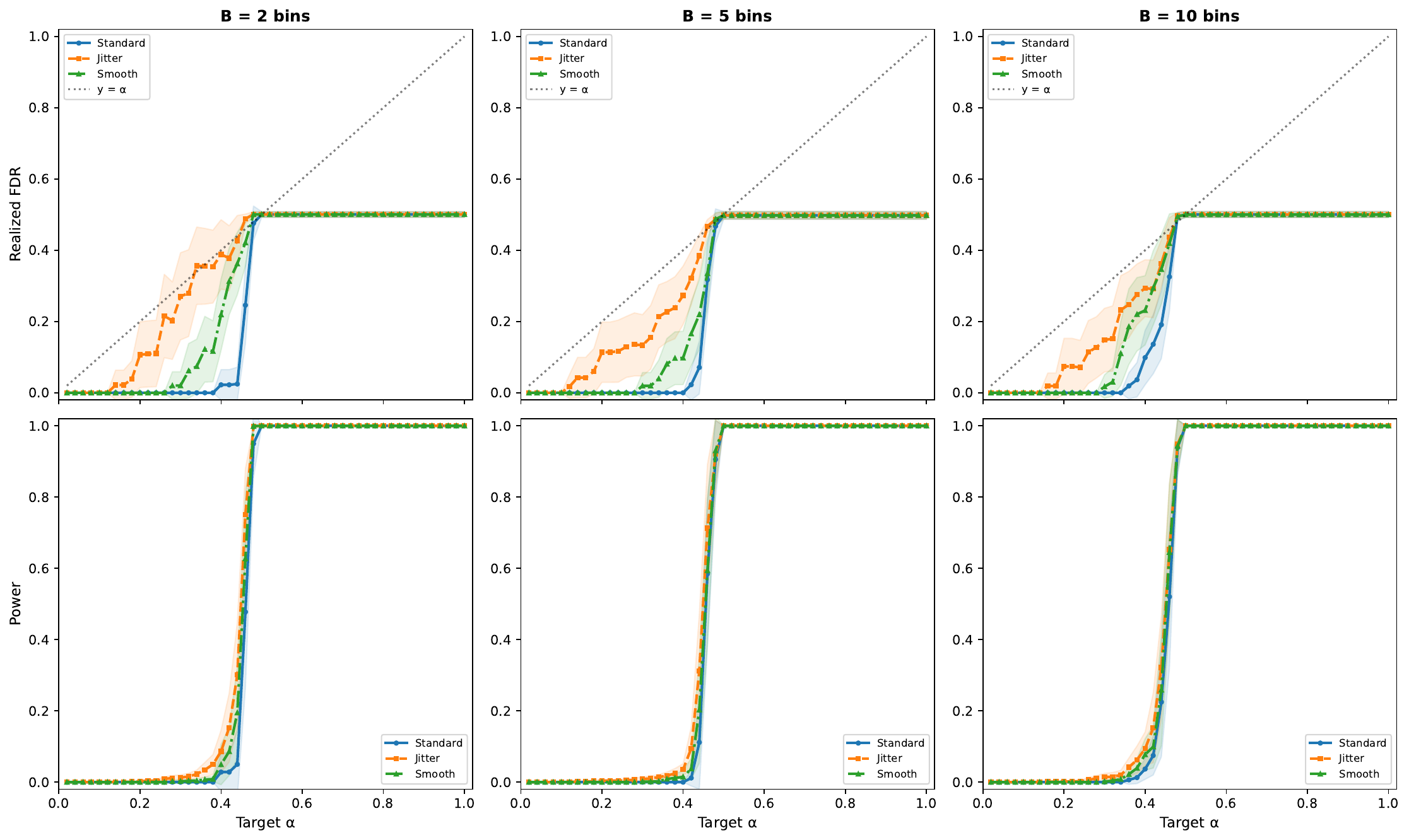}
    \caption{\textbf{Tie handling under truncation.}
    Standard, jittered, and smooth tie-breaking variants show similar FDR/power behavior, suggesting that truncation-induced ties do not drive the empirical results.}
    \label{fig:tie_breaking}
\end{figure}

\subsubsection{Estimator and proxy compatibility}
\label{app:robustness_estimator_proxy}

We additionally vary the base CATE learner and the proxy construction.
Figure~\ref{fig:estimator_agnostic} compares DCA and the naive proxy across Random Forest, GBDT, Linear, and X-learner variants, showing that the denoising gain is not tied to a specific CATE model.
Figures~\ref{fig:if_proxy}--\ref{fig:orthogonal_proxy} compare DR, influence-function-style, AIPW-orthogonal, and R-learner proxy variants.
Across these alternatives, raw proxies can be variance-dominated, while denoising improves the reliability-ranking signal.

\begin{figure}[t]
    \centering
    \includegraphics[width=\textwidth]{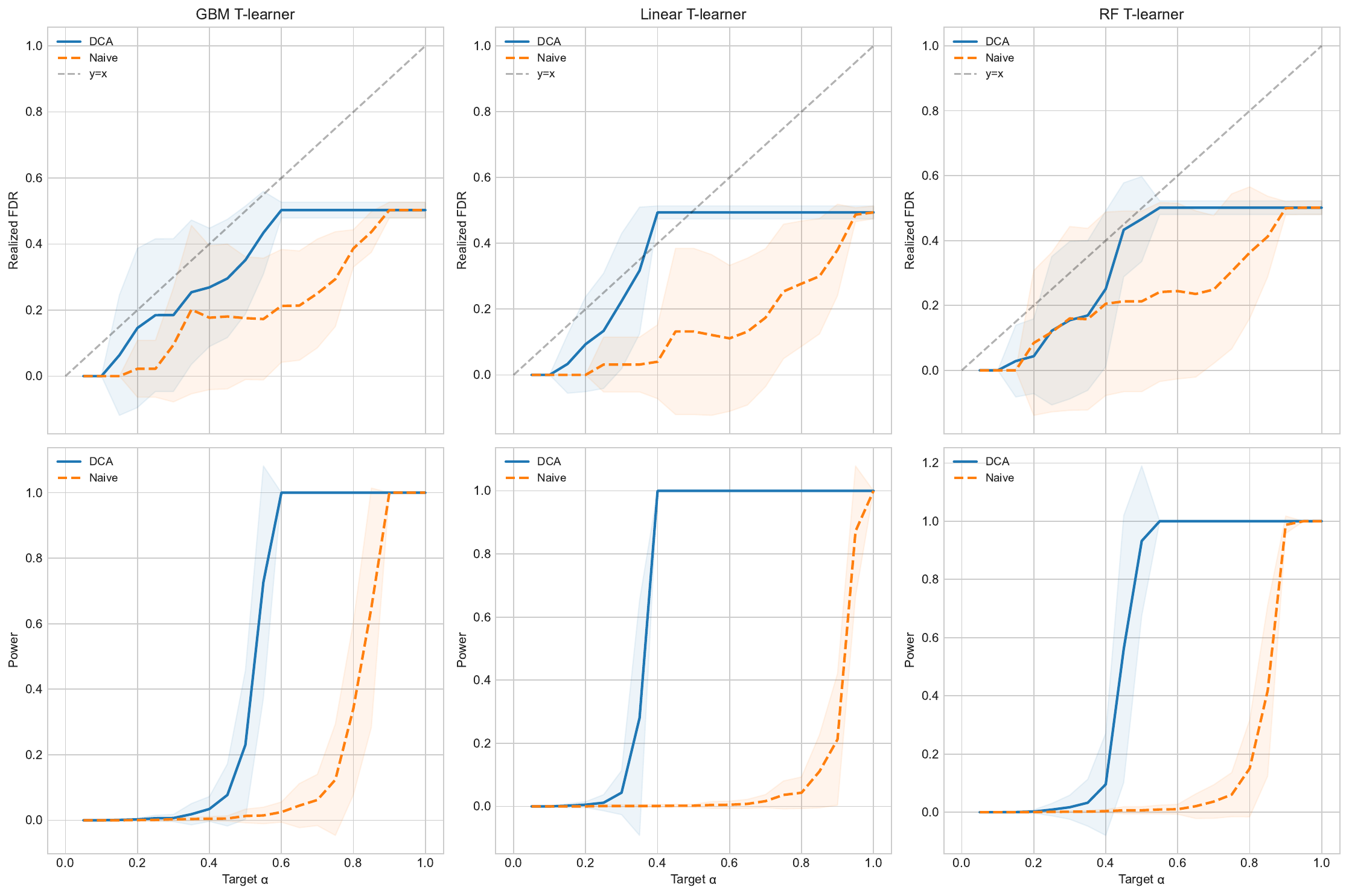}
    \caption{\textbf{CATE estimator agnosticism.}
    DCA versus naive proxy selection across base CATE estimators. Denoising consistently improves selection yield at comparable FDR behavior.}
    \label{fig:estimator_agnostic}
\end{figure}

\begin{figure}[t]
    \centering
    \includegraphics[width=\textwidth]{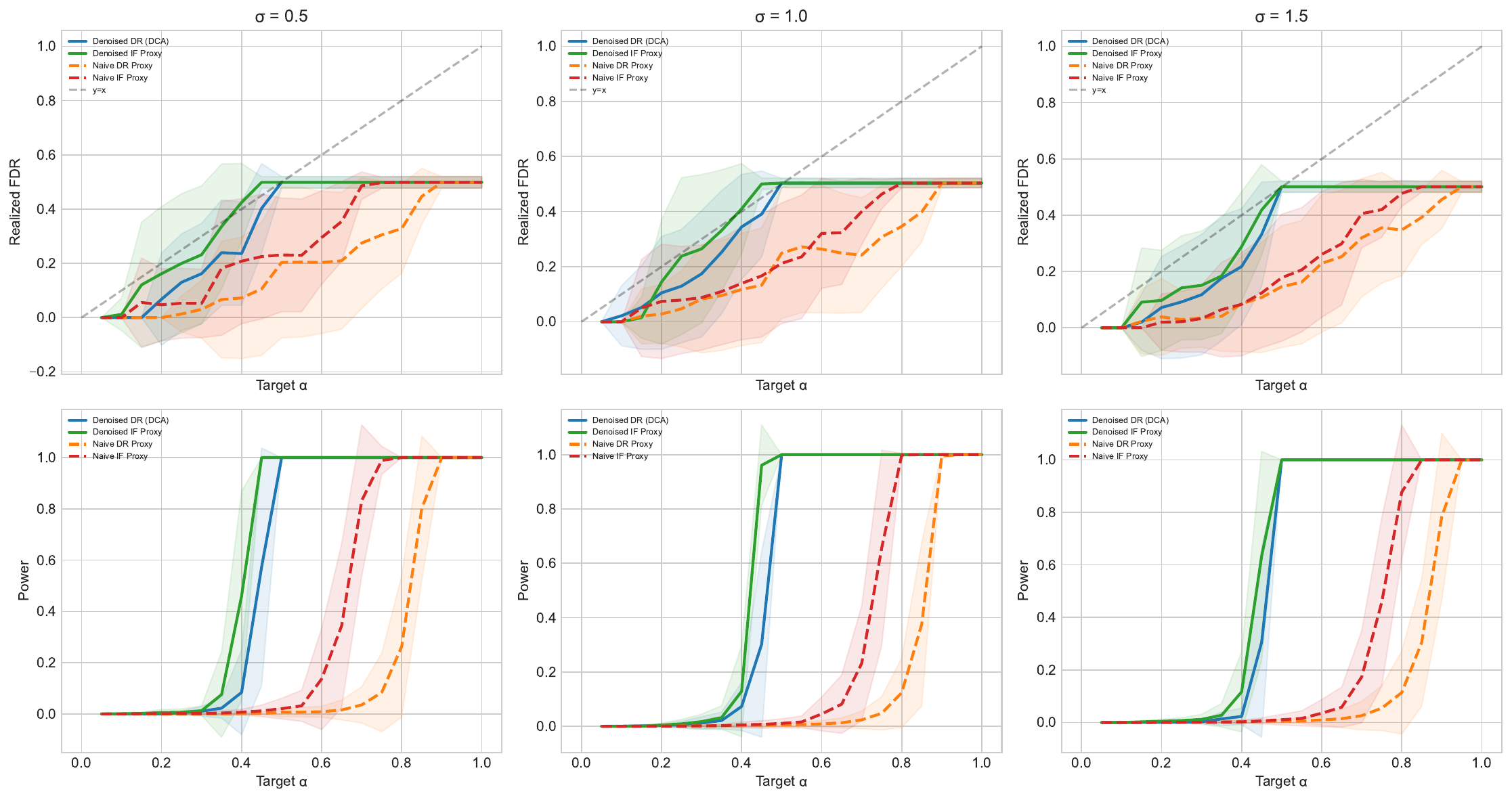}
    \caption{\textbf{Influence-function proxy comparison.}
    Naive and denoised DR/IF proxies across noise levels. Denoising improves the reliability-ranking signal for both proxy families.}
    \label{fig:if_proxy}
\end{figure}

\begin{figure}[t]
    \centering
    \includegraphics[width=\textwidth]{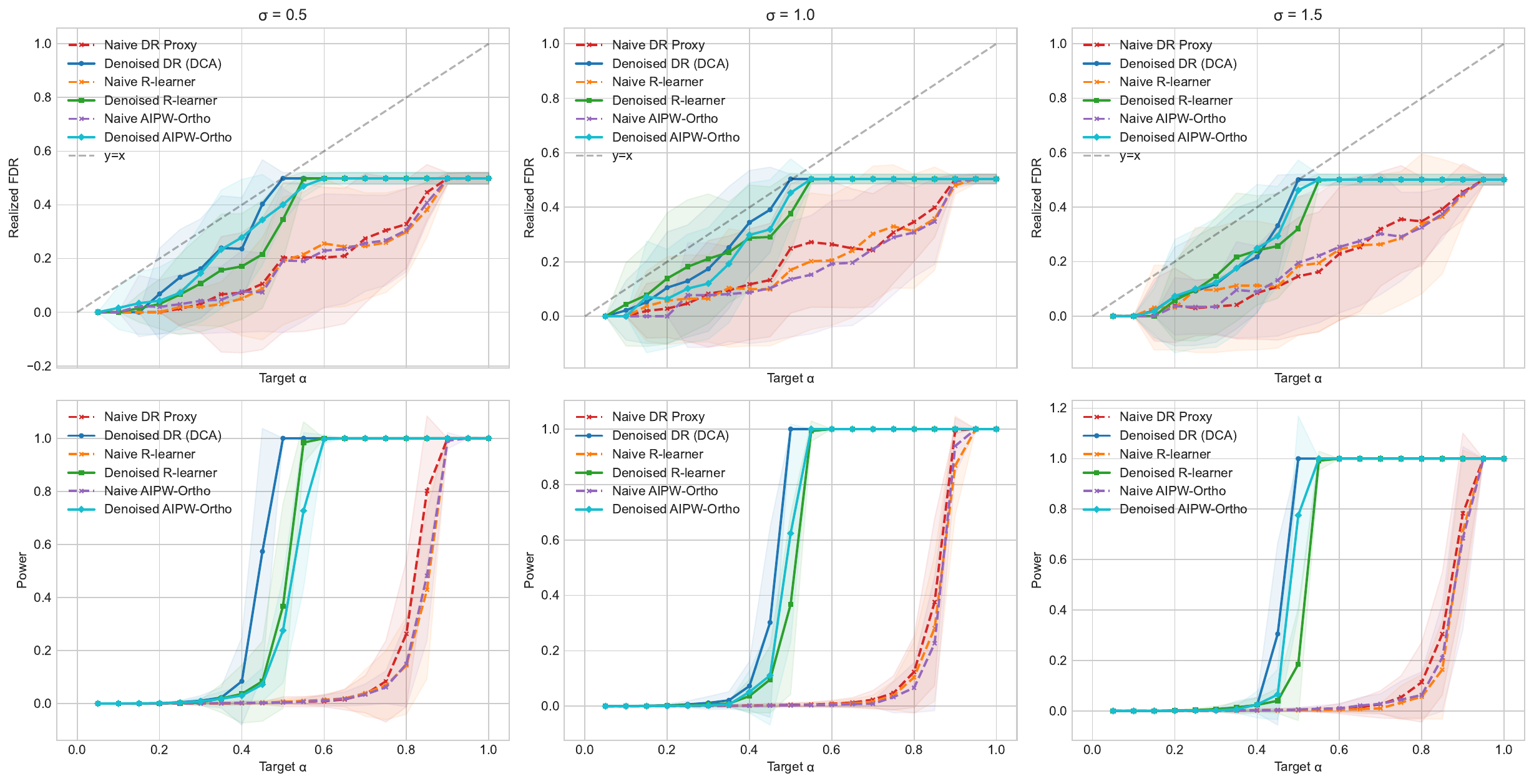}
    \caption{\textbf{Orthogonal and R-learner proxy comparison.}
    Variance-aware denoising remains beneficial for AIPW-orthogonal and R-learner-style proxy constructions, confirming proxy-level compatibility.}
    \label{fig:orthogonal_proxy}
\end{figure}

\subsubsection{Selection paradigms and downstream value}
\label{app:robustness_selection_value}

Figure~\ref{fig:selection_paradigm} compares DCA with naive DR, ensemble uncertainty, Top-$K$ heuristics, and oracle references.
Only conformal/BH procedures are designed to control post-selection FDR; heuristic Top-$K$ rules can select many candidates but lack reliable error control.
Figure~\ref{fig:policy_value} translates selections into downstream policy value and treated-unit counts, showing that DCA's additional discoveries improve utility in the simulated policy environment.

\begin{figure}[t]
    \centering
    \includegraphics[width=0.78\textwidth]{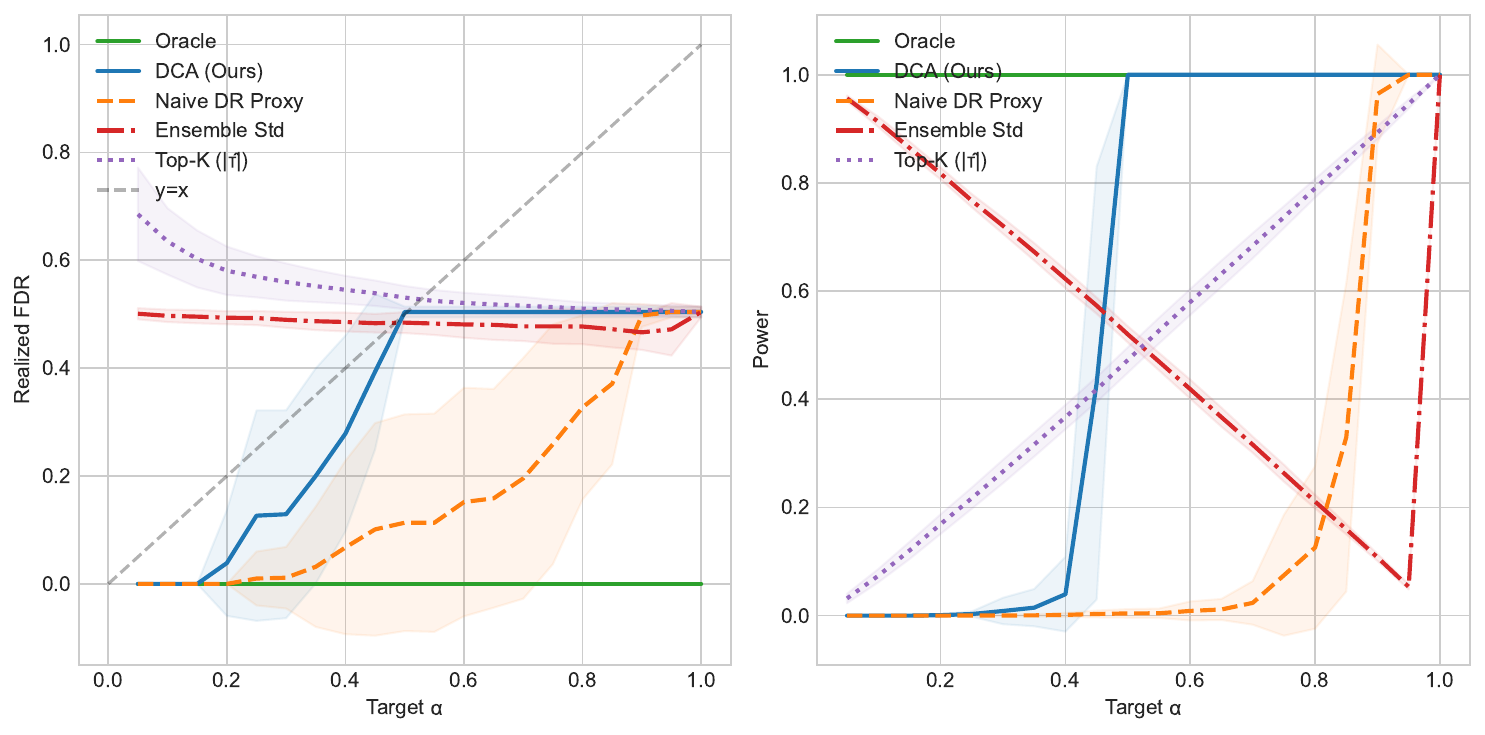}
    \caption{\textbf{Selection paradigm comparison.}
    DCA, naive DR, ensemble uncertainty, Top-$K$, and oracle references. DCA combines FDR control with substantially higher power than proxy or uncertainty baselines.}
    \label{fig:selection_paradigm}
\end{figure}

\begin{figure}[t]
    \centering
    \includegraphics[width=0.9\textwidth]{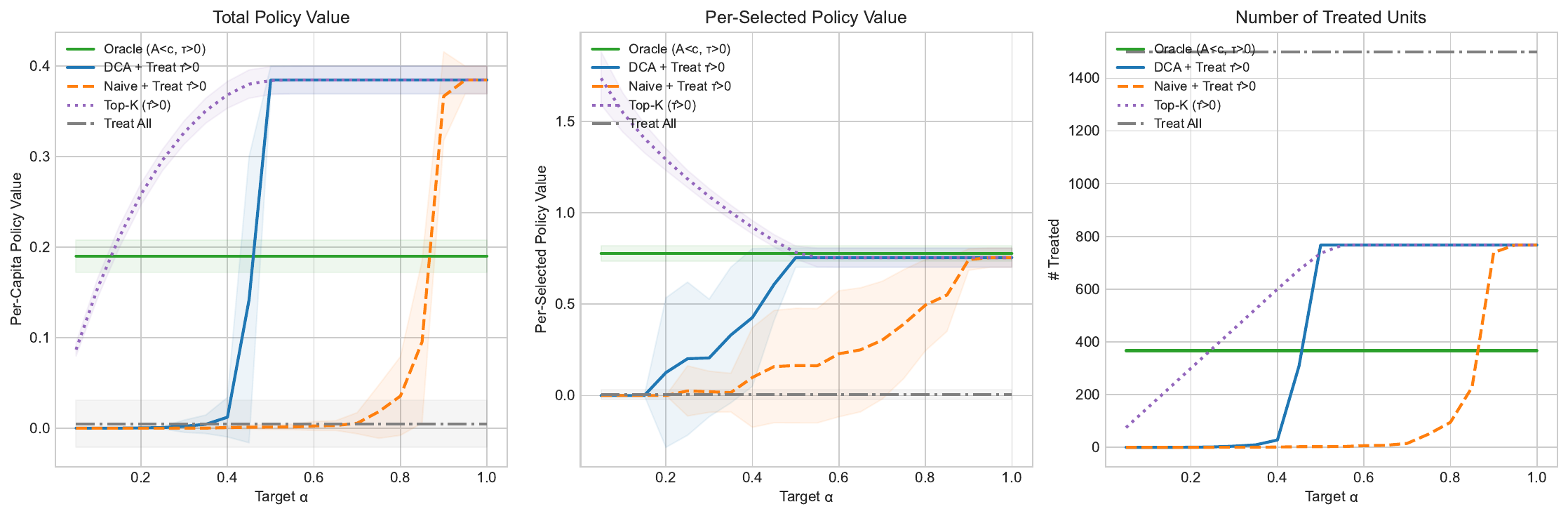}
    \caption{\textbf{Downstream policy value.}
    Per-unit policy value and number of treated units under different selection policies. DCA converts improved reliable selection into higher downstream utility.}
    \label{fig:policy_value}
\end{figure}

\subsubsection{Small-sample and high-dimensional nuisance stress}
\label{app:robustness_nuisance_stress}

Figure~\ref{fig:small_sample} varies total sample size, and Figure~\ref{fig:highdim_nuisance} increases covariate dimension under sparse causal structure.
These settings make nuisance and variance estimation harder.
DCA still improves selection yield when calibration has enough information to select, while difficult nuisance estimation can reduce power through less stable proxy labels.

\begin{figure}[t]
    \centering
    \includegraphics[width=0.9\textwidth]{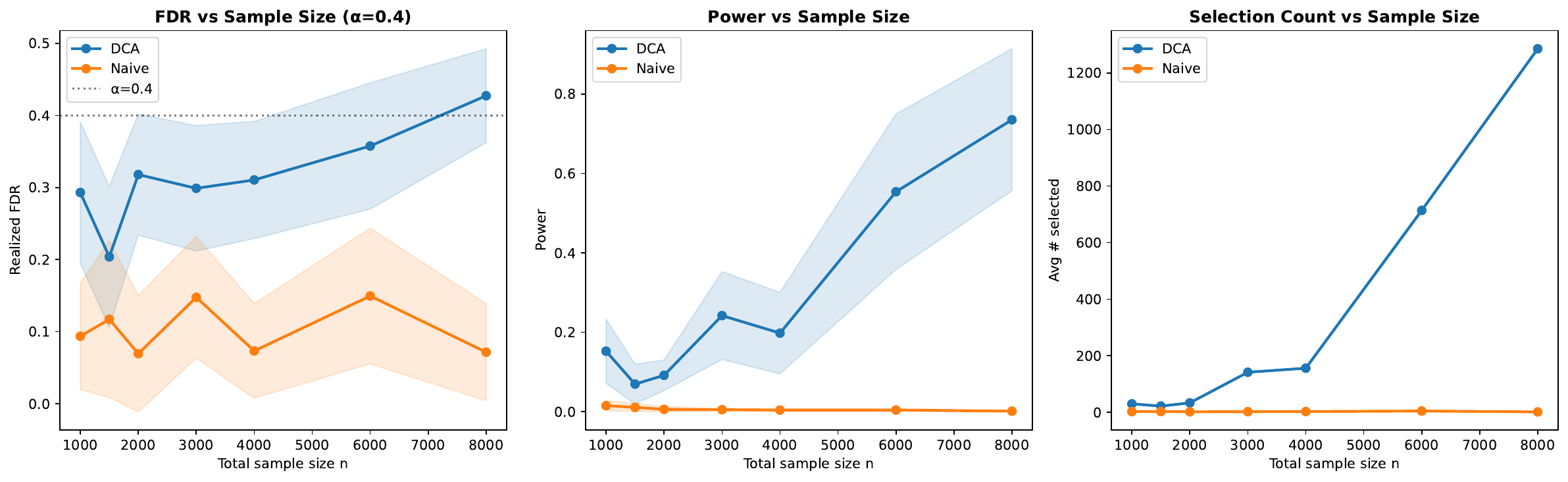}
    \caption{\textbf{Small-sample behavior.}
    DCA remains competitive in small-$N$ regimes, with power improving as calibration and nuisance estimation become more stable.}
    \label{fig:small_sample}
\end{figure}

\begin{figure}[t]
    \centering
    \includegraphics[width=\textwidth]{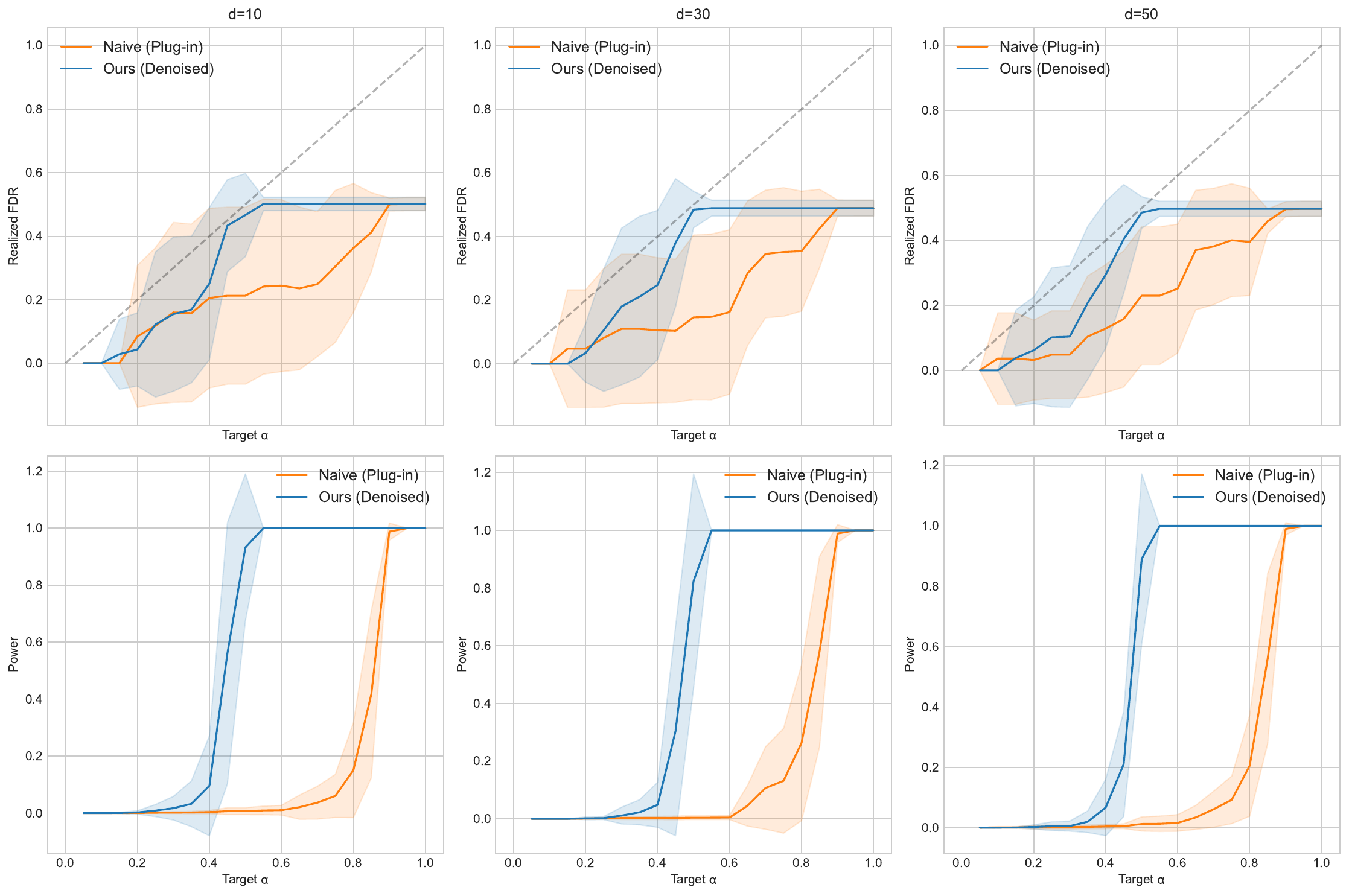}
    \caption{\textbf{High-dimensional nuisance difficulty.}
    Increasing dimension stresses nuisance and variance estimation; DCA retains the denoising advantage when the learned proxy labels remain sufficiently stable.}
    \label{fig:highdim_nuisance}
\end{figure}

\subsection{Semi-synthetic Benchmarks}
\label{app:semisynthetic_benchmarks}

\subsubsection{IHDP semi-synthetic benchmark (Setting 5)}
\label{app:ihdp_details}

\textbf{Why ``ground-truth CATE'' exists (semi-synthetic clarification).}
IHDP covariates originate from a real observational study, but the benchmark used in our experiments is \emph{semi-synthetic}:
potential outcomes are generated by a known simulator. Hence individual-level $\mu_0(X),\mu_1(X)$ (and therefore CATE) are available
\emph{by construction}, avoiding any misconception that real IHDP contains observed counterfactuals.

\textbf{Data, columns, and ground truth.}
The benchmark provides simulator columns $\mu_0,\mu_1$ and a factual outcome $y_{\mathrm{factual}}$ with treatment indicator \texttt{treatment}.
We define the ground-truth CATE as
\[
\tau(X)=\mu_1(X)-\mu_0(X),
\]
use the observed outcome $Y=y_{\mathrm{factual}}$ and treatment $T=\texttt{treatment}$,
and take covariates as all columns whose names start with \texttt{x} (e.g., $x_1,\dots,x_{25}$).
We apply the benchmark’s standard preprocessing (continuous covariates standardized; categorical covariates one-hot encoded if present).

\textbf{Protocol (sample splitting + code safeguard).}
We repeatedly random-split IHDP into four disjoint folds
$\mathcal D_{\mathrm{tr1}},\mathcal D_{\mathrm{tr2}},\mathcal D_{\mathrm{cal}},\mathcal D_{\mathrm{test}}$
following the common protocol in Appendix~\ref{app:common_protocol}.
Nuisance/variance models are fit on $\mathcal D_{\mathrm{tr1}}$, the alignment predictor $g$ is trained on $\mathcal D_{\mathrm{tr2}}$,
and conformal calibration uses $\mathcal D_{\mathrm{cal}}$.
\textbf{Implementation safeguard:} if the dataset size $n$ is smaller than the nominal $(2000,2000,1000)$ split,
the code automatically downscales each fold to keep all splits nonempty.

\textbf{What this benchmark tests (complement to synthetic settings).}
IHDP differs from the synthetic DGPs in two important ways: (i) covariates come from a realistic observational design with complex correlations
and practical imbalance patterns; (ii) the semi-synthetic simulator provides ground truth while preserving real covariate structure.
This makes IHDP a useful ``sanity + realism'' check: it tests whether DCA’s denoising-and-selection pipeline continues to deliver nontrivial yield
under realistic feature geometry, without relying on synthetic simplifications.

\textbf{Sensitivity over $\rho$ (and representative choice).}
We sweep $\rho\in\{0.1,0.15,0.2,0.25,0.5\}$ and report realized FDR and power averaged over random splits.
The representative choice $\rho=0.15$ is used in the main IHDP comparison because it achieves a favorable tradeoff:
it is conservative enough to avoid over-subtraction when variance estimation is imperfect on finite IHDP folds,
yet still materially improves selection yield compared to $\rho=0$ (no denoising) and uncertainty-only baselines.
Figure~\ref{fig:ihdp_setting5} summarizes the $\rho$ sensitivity.

\begin{figure}[t]
    \centering
    \includegraphics[width=0.7\textwidth]{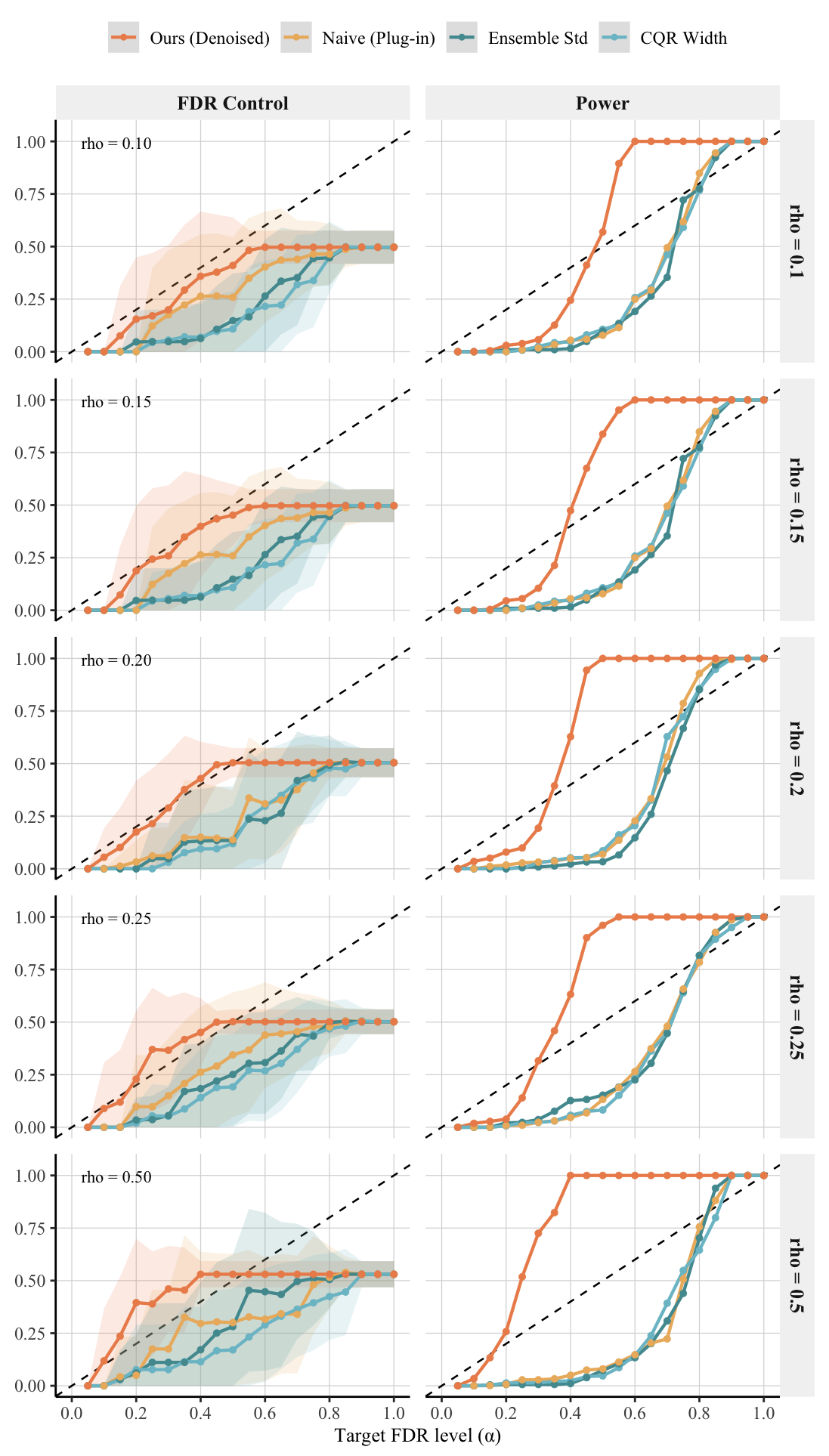}
    \caption{\textbf{Setting 5 (IHDP) sensitivity.}
    Realized FDR and power across $\rho\in\{0.1,0.15,0.2,0.25,0.5\}$ on semi-synthetic IHDP with simulator-provided ground-truth CATE.
    Moderate denoising improves selection yield while maintaining conservative or near-nominal FDR; the representative choice $\rho=0.15$
    reflects a stable point in this tradeoff.}
    \label{fig:ihdp_setting5}
\end{figure}

\textbf{IHDP protocol (pseudocode).}
\begin{verbatim}
# Load IHDP semi-synthetic benchmark
X = all columns starting with "x"          # features
T = treatment                               # binary treatment
Y = y_factual                               # observed outcome
tau_true = mu1 - mu0                        # simulator truth

# Repeated trials:
for r in 1..R:
    split indices into D_tr1, D_tr2, D_cal, D_test
    if n too small for nominal sizes:
        downscale fold sizes to keep each nonempty   # code safeguard

    fit nuisance/variance models on D_tr1 only
    compute DR pseudo-outcomes and Vhat on D_tr2 and D_cal using tr1-fitted models
    train g on D_tr2
    predict scores s=g(X) on D_cal and D_test

    # evaluation threshold (calibration-only):
    A2_true_cal = (tauhat(X_cal) - tau_true_cal)^2
    c = median(A2_true_cal)

    compute conformal p-values using D_cal and threshold c (same as main algorithm)
    apply BH on p-values over D_test
    evaluate FDP/power on D_test using tau_true and the same c
\end{verbatim}

\subsubsection{NLSM semi-synthetic benchmark}
\label{app:external_benchmarks}

To test whether the observed safety--power pattern is specific to our custom synthetic DGPs, we also evaluate on the NLSM benchmark suite adapted from the counterfactual conformal inference literature.
These data-generating mechanisms provide external semi-synthetic covariate/outcome structures with known CATE.
Figure~\ref{fig:nlsm_benchmark} shows that across multiple NLSM settings and noise levels, DCA preserves its qualitative advantage: denoising increases selection yield while maintaining stable realized FDR.

\begin{figure}[t]
    \centering
    \includegraphics[width=\textwidth]{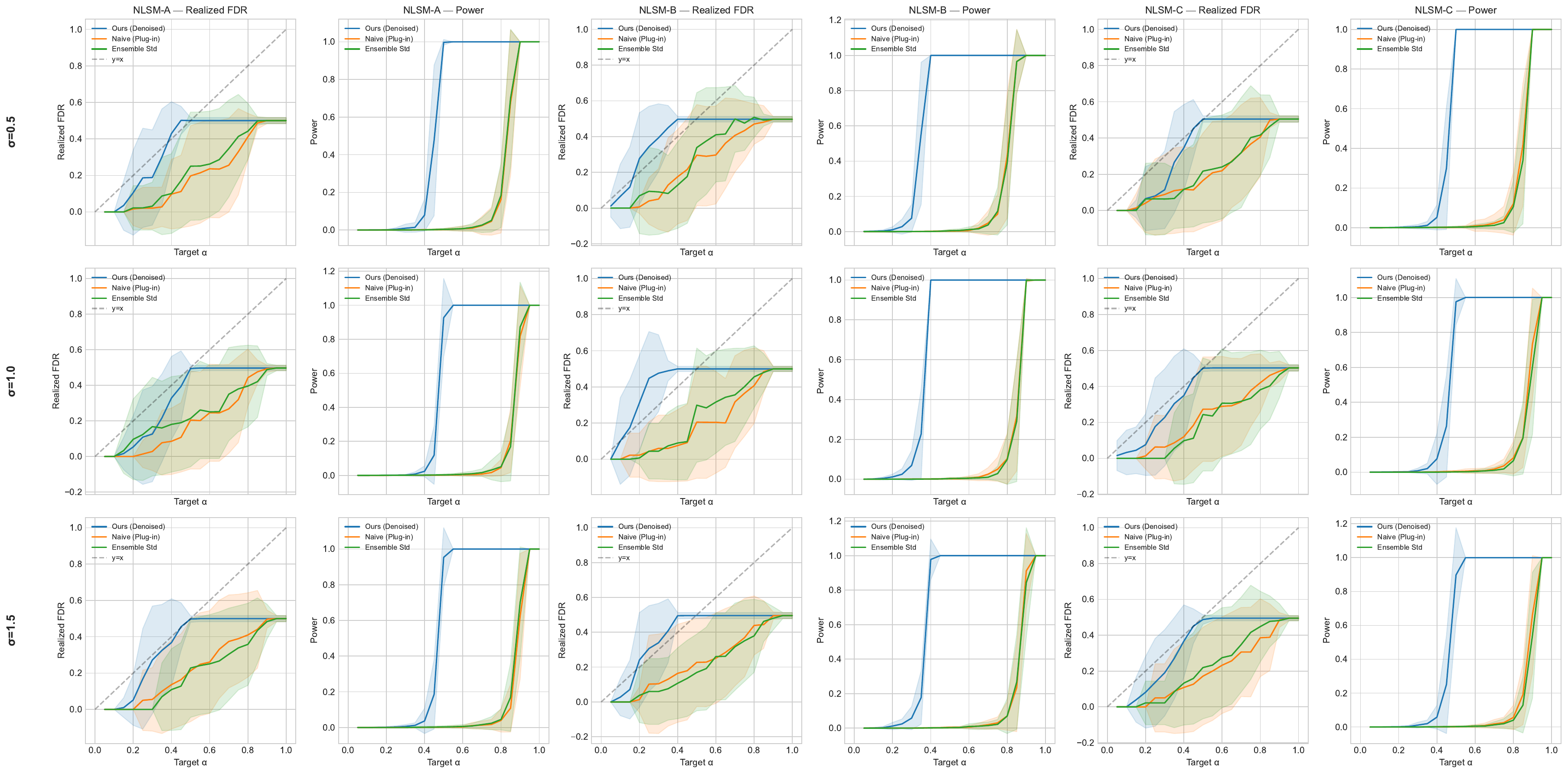}
    \caption{\textbf{NLSM semi-synthetic benchmark.}
    Three NLSM data-generating mechanisms at multiple noise levels. DCA's safety--power pattern persists beyond the custom synthetic settings.}
    \label{fig:nlsm_benchmark}
\end{figure}

\subsubsection{NSW job-training real-data analysis}
\label{app:nsw_real_data}

We additionally run DCA on the NSW job-training data, a canonical observational causal benchmark.
Since true individual CATE errors are unavailable in purely real data, this experiment is a qualitative deployment analysis rather than an oracle FDR/power benchmark.
Figure~\ref{fig:nsw_real_data} summarizes the selected subpopulation, estimated treatment-effect distribution, and policy-relevant diagnostics under DCA.

\begin{figure}[t]
    \centering
    \includegraphics[width=\textwidth]{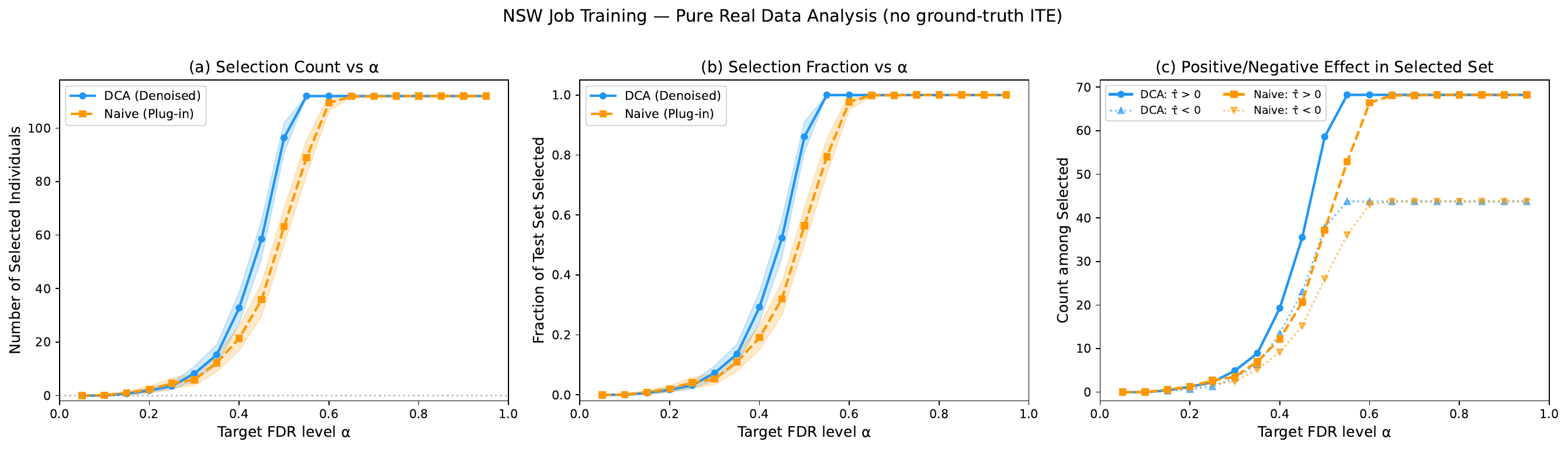}
    \caption{\textbf{NSW job-training real-data analysis.}
    Pure real-data deployment analysis on the NSW benchmark. Since true CATE errors are unobserved, the figure provides qualitative policy and selected-subpopulation diagnostics rather than oracle FDR/power.}
    \label{fig:nsw_real_data}
\end{figure}

\subsection{Multi-treatment extension and overlap stress test}
\label{app:multi_treatment}

This appendix extends Denoised Conformal Alignment (DCA) from binary treatment
\(T\in\{0,1\}\) to the multi-treatment setting \(T\in\{0,1,\dots,K-1\}\),
where \(t=0\) denotes control and each arm \(k\in\{1,\dots,K-1\}\) induces an arm-specific ITE
\(\tau_k(x):=\mu_k(x)-\mu_0(x)\).
We then report an ablation study showing that the extension remains stable for \(K=3\),
but can fail for \(K=5\) under finite-sample overlap degradation, motivating overlap-aware future work.

\paragraph{Setup and targets.}
Let \(\mu_t(x):=\mathbb{E}[Y\mid X=x,T=t]\) and \(e_t(x):=\mathbb{P}(T=t\mid X=x)\) with overlap \(e_t(x)>0\).
We assume standard identification (unconfoundedness and positivity) so that each \(\tau_k(x)\) is identifiable.
Given a black-box ITE estimator \(\hat\tau_k(x)\), our deployment goal is reliable selection with FDR control:
for each arm \(k\), define the squared error
\[
A_k^2(x):=\big(\hat\tau_k(x)-\tau_k(x)\big)^2,
\quad
\mathcal H_{0,k}:=\{A_k^2(X)\ge c_k\},
\quad
\mathcal H_{1,k}:=\{A_k^2(X)<c_k\},
\]
where \(c_k\) is a user-defined tolerance for arm \(k\).
We select hypotheses among candidates \((j,k)\) while controlling the post-selection FDR at level \(\alpha\).

\subsubsection{DR pseudo-outcomes for each arm}
\label{app:multi_dr}

\paragraph{Arm-wise doubly robust pseudo-outcome.}
For each \(k\in\{1,\dots,K-1\}\), an unbiased DR pseudo-outcome for \(\tau_k(x)\) is
\begin{equation}
\phi_{k}(X,T,Y)
=
\big(\mu_k(X)-\mu_0(X)\big)
+\frac{\mathbf{1}\{T=k\}}{e_k(X)}\big(Y-\mu_k(X)\big)
-\frac{\mathbf{1}\{T=0\}}{e_0(X)}\big(Y-\mu_0(X)\big).
\label{eq:multi_dr_phi}
\end{equation}
With consistent nuisance estimates \(\hat\mu_t,\hat e_t\), we compute \(\hat\phi_{i,k}\) on held-out folds.

\paragraph{Proxy error and denoising.}
Analogous to the binary case, we construct a squared proxy error for each arm:
\[
\widetilde A_{i,k}^2 := \big(\hat\tau_k(X_i)-\hat\phi_{i,k}\big)^2.
\]
The key observation is that \(\widetilde A_{i,k}^2\) is contaminated by arm-dependent conditional variance:
under mild regularity,
\[
\mathbb{E}\!\left[\widetilde A_{k}^2 \mid X\right]
=
A_k^2(X) + \operatorname{Var}(\phi_k\mid X) + o(1),
\]
so heteroskedasticity (and, crucially, near-violations of overlap) can dominate ranking.
We therefore apply variance-aware denoising arm-wise:
\begin{equation}
\check A_{i,k}^2(\rho)
:=
\Big(\widetilde A_{i,k}^2 - \rho\,\widehat V_k(X_i)\Big)_+,
\qquad \rho\ge 0.
\label{eq:multi_denoise}
\end{equation}
Here \(\widehat V_k(X)\) estimates \(\operatorname{Var}(\phi_k\mid X)\), and \((\cdot)_+=\max\{\cdot,0\}\).
In the multi-treatment DR structure, a convenient plug-in approximation is
\begin{equation}
\operatorname{Var}(\phi_k\mid X)
\approx
\frac{\operatorname{Var}(Y\mid X,T=k)}{e_k(X)}
+
\frac{\operatorname{Var}(Y\mid X,T=0)}{e_0(X)},
\label{eq:multi_var}
\end{equation}
which matches our implementation that regresses residual-squares within each arm.

\subsubsection{Conformal alignment and pooled BH across arms}
\label{app:multi_conformal}

\paragraph{Arm-wise alignment predictors.}
We train an alignment model \(g_k:\mathcal X\to\mathbb{R}\) (one per arm \(k\))
to predict the denoised proxy \(\check A_{i,k}^2(\rho)\) from covariates \(X_i\).
Smaller scores indicate higher reliability:
\[
s_{k}(x):=g_k(x)\quad\text{(lower is better)}.
\]

\paragraph{Arm-specific conformal \textit{oracle} \(p\)-values.}
On the calibration fold, define the arm-wise null subset
\[
\mathcal I_{0,k}:=\Big\{i\in\mathcal D_{\mathrm{cal}}:\ \check A_{i,k}^2(\rho)\ge c_k\Big\}.
\]
Given a test candidate \(j\), we compute
\begin{equation}
p_{j,k}
=
\frac{1+\#\{i\in\mathcal I_{0,k}: s_k(X_i)\le s_k(X_j)\}}
{|\mathcal D_{\mathrm{cal}}|+1}.
\label{eq:multi_pvalue}
\end{equation}
This is the direct multi-arm analogue of the binary conformal alignment \(p\)-value:
it compares the candidate score to the empirical distribution of scores among ``bad'' calibration points for arm \(k\).

\paragraph{Pooled selection over \((j,k)\).}
To decide deployment across all arms, we treat each pair \((j,k)\) as a hypothesis and apply BH to the pooled set
\(\{p_{j,k}\}_{j\in\mathcal D_{\mathrm{test}},\,k=1,\dots,K-1}\) at target level \(\alpha\).
Let \(\mathcal S(\alpha)\subseteq\mathcal D_{\mathrm{test}}\times\{1,\dots,K-1\}\) be the selected set.
When ground-truth is available, we evaluate realized FDR and power by flattening:
\[
\mathrm{FDP}(\alpha)
=
\frac{|\{(j,k)\in\mathcal S(\alpha): A_{j,k}^2\ge c_k\}|}{|\mathcal S(\alpha)|\vee 1},
\quad
\mathrm{Power}(\alpha)
=
\frac{|\{(j,k)\in\mathcal S(\alpha): A_{j,k}^2< c_k\}|}{|\{(j,k): A_{j,k}^2<c_k\}|\vee 1}.
\]
This ``pooled BH'' matches a deployment scenario where we choose reliable individuals-and-arms jointly.

\subsubsection{Algorithm}
\label{app:multi_algorithm}

\begin{algorithm}[t]
\caption{DCA for multi-treatment reliable selection (pooled BH)}
\label{alg:multi_dca}
\small
\begin{algorithmic}[1]
\REQUIRE Data split \(\mathcal D_{\mathrm{tr1}},\mathcal D_{\mathrm{tr2}},\mathcal D_{\mathrm{cal}},\mathcal D_{\mathrm{test}}\);
arms \(k=1,\dots,K-1\); tolerance \(c_k\); denoising \(\rho\); target FDR \(\alpha\).
\STATE Fit nuisance models \(\hat\mu_t,\hat e_t\) and variance models \(\widehat{\operatorname{Var}}(Y\mid X,T=t)\) on \(\mathcal D_{\mathrm{tr1}}\).
\STATE For each \(i\in\mathcal D_{\mathrm{tr2}}\cup\mathcal D_{\mathrm{cal}}\), compute \(\hat\phi_{i,k}\) via~\eqref{eq:multi_dr_phi}.
\STATE For each arm \(k\), compute \(\widetilde A_{i,k}\) and \(\check A_{i,k}(\rho)\) via~\eqref{eq:multi_denoise}--\eqref{eq:multi_var}.
\STATE Train arm-wise alignment models \(g_k\) on \(\mathcal D_{\mathrm{tr2}}\): regress \(\check A_{i,k}(\rho)\) on \(X_i\).
\STATE On \(\mathcal D_{\mathrm{cal}}\), form \(\mathcal I_{0,k}=\{i:\check A_{i,k}(\rho)\ge c_k\}\) and compute \(p_{j,k}\) for each test candidate \(j\) using~\eqref{eq:multi_pvalue}.
\STATE Pool all \(p_{j,k}\) and apply BH at level \(\alpha\) to obtain \(\mathcal S(\alpha)\).
\ENSURE Selected reliable set \(\mathcal S(\alpha)\subseteq\mathcal D_{\mathrm{test}}\times\{1,\dots,K-1\}\).
\end{algorithmic}
\end{algorithm}

\subsubsection{Experiments: stability at \(K=3\) and failure at \(K=5\)}
\label{app:multi_results}

\paragraph{Synthetic multi-treatment DGP.}
We extend Setting~1 to \(K\in\{3,5\}\) arms with a multinomial propensity (softmax logits) and arm-dependent effects.
Outcomes are heteroskedastic:
\[
Y=\mu_T(X)+\epsilon,\qquad \epsilon\sim \mathcal N\!\big(0,\sigma^2(X)\big),
\quad
\sigma(X)=\sigma_{\mathrm{base}}\big(1+0.5|X_0|\big),
\quad \sigma_{\mathrm{base}}\in\{0.5,1.0,1.5\}.
\]
We use the common protocol in Appendix~\ref{app:common_protocol} with sample splitting.
Following the ground-truth settings, we set each \(c_k\) using only calibration information:
\[
c_k := \mathrm{median}_{i\in\mathcal D_{\mathrm{cal}}}\big(\hat\tau_k(X_i)-\tau_k(X_i)\big)^2.
\]
We compare \textbf{Ours (Denoised)} \(\check A^2(\rho)\) versus \textbf{Naive (Plug-in)} \(\widetilde A^2\) (i.e., \(\rho=0\)),
keeping the same conformal+BH pipeline.

\paragraph{Key findings.}
\emph{(i) Moderate-arm regime \(K=3\):} the multi-treatment extension remains stable.
Across \(\sigma_{\mathrm{base}}\in\{0.5,1.0,1.5\}\), realized FDR tracks the nominal line while DCA achieves higher power than the naive proxy.
In our sweep, \(\rho=0.25\) yields the best trade-off: high power with realized FDR closest to \(y=x\).
Figure~\ref{fig:multi_k3} reports representative curves and sensitivity across \(\rho\).

\emph{(ii) Larger-arm regime \(K=5\):} we observe systematic FDR inflation even for small \(\rho\).
Empirically, realized FDR exceeds the nominal line across a wide range of \(\alpha\), indicating that the pooled procedure can fail under finite-sample multi-arm overlap degradation.
Figure~\ref{fig:multi_k5_fail} summarizes this failure pattern.

\begin{figure}[t]
    \centering
    \includegraphics[width=0.8\textwidth]{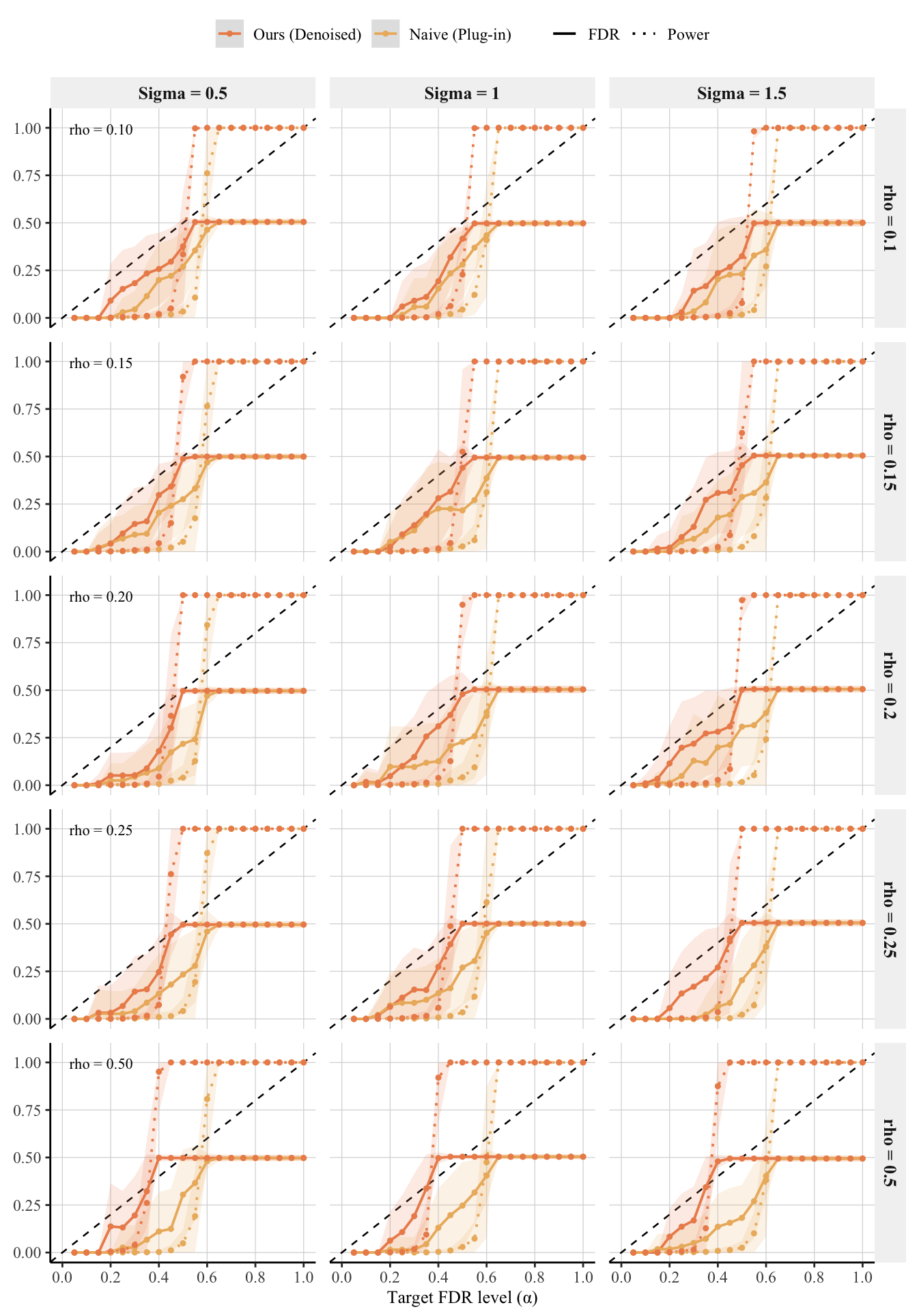}
    \caption{\textbf{Multi-treatment \(K=3\).} Realized FDR and power versus target \(\alpha\).
    DCA remains well-calibrated and more powerful than the naive proxy baseline; \(\rho=0.25\) provides the best power--calibration trade-off in this setting.}
    \label{fig:multi_k3}
\end{figure}

\begin{figure}[t]
    \centering
    \includegraphics[width=0.8\textwidth]{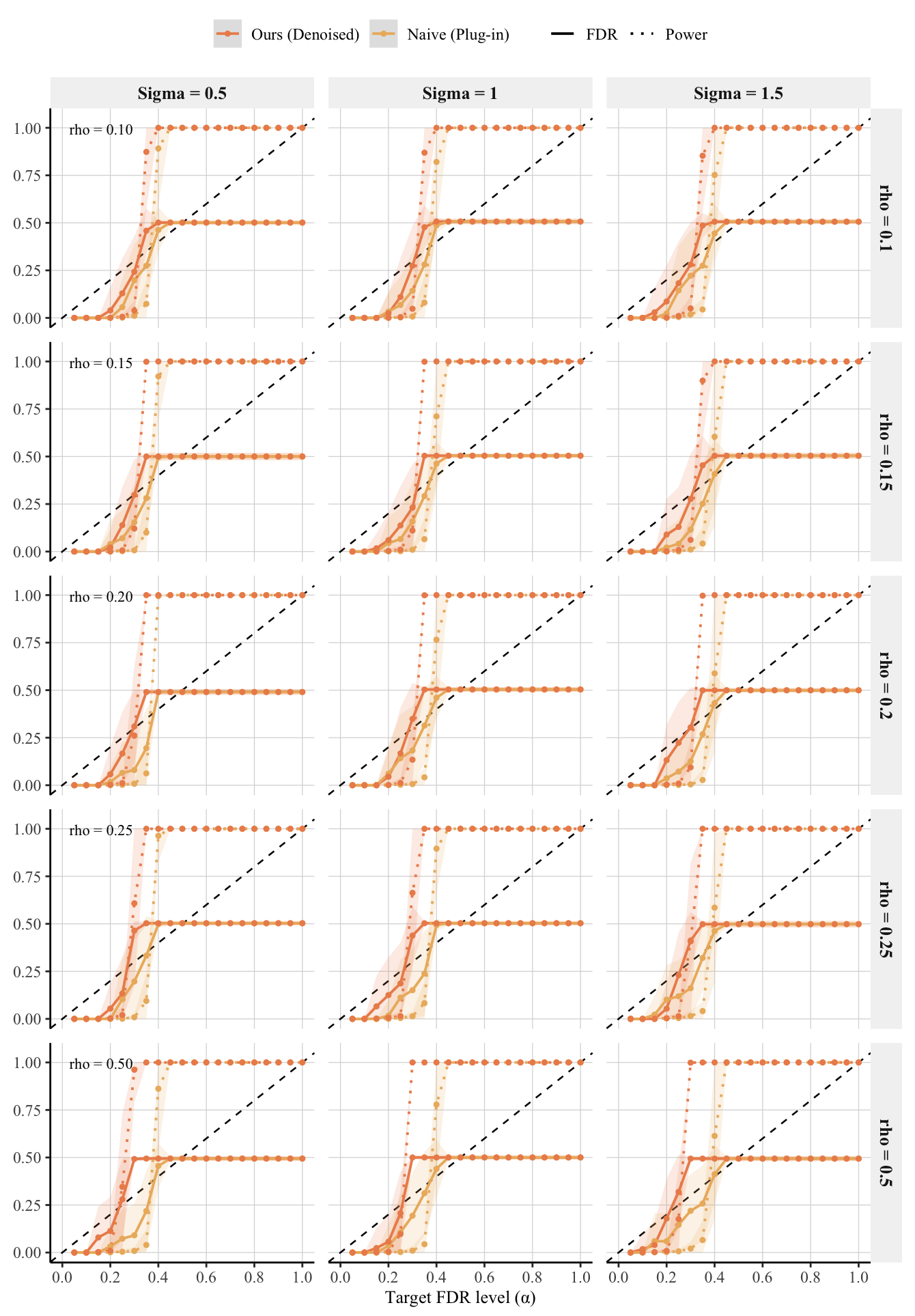}
    \caption{\textbf{Multi-treatment \(K=5\) failure under overlap stress.}
    Realized FDR exceeds the nominal line across \(\alpha\) even for small denoising strengths, indicating breakdown of calibration/selection under finite-sample overlap degradation in multi-arm settings.}
    \label{fig:multi_k5_fail}
\end{figure}

\subsubsection{Why \(K=5\) can break: overlap degradation amplifies proxy noise and destabilizes calibration}
\label{app:multi_failure_analysis}

The observed breakdown at \(K=5\) is consistent with a multi-arm overlap mechanism.
Even when each arm propensity is clipped away from zero, increasing \(K\) pushes the \emph{effective} overlap to deteriorate:
for many \(x\), at least one arm has relatively small \(e_k(x)\), and the DR correction in~\eqref{eq:multi_dr_phi} contains factors \(1/e_k(X)\) and \(1/e_0(X)\).
Consequently, the conditional variance in~\eqref{eq:multi_var} can become large and highly heterogeneous:
\[
\operatorname{Var}(\phi_k\mid X)\;\approx\; \operatorname{Var}(Y\mid X,T=k)/e_k(X)\;+\;\operatorname{Var}(Y\mid X,T=0)/e_0(X),
\]
so the proxy distribution for ``bad'' calibration points \(\mathcal I_{0,k}\) is dominated by a heavy right tail.
This effect is compounded when we \emph{pool} hypotheses across arms: a small fraction of unstable arms can dominate BH discoveries and inflate the realized FDR.

In finite samples, this can manifest as:
(i) unstable null sets \(\mathcal I_{0,k}\) (proxy labels flip due to high variance);
(ii) noisy learned alignment \(g_k\) (trained on contaminated targets);
(iii) non-uniform \(p_{j,k}\) under the null because the calibration comparison set is effectively mis-specified by variance explosions.
Together these effects can yield the observed empirical FDR inflation.

\subsubsection{Future work: overlap-aware multi-treatment DCA}
\label{app:multi_future}

The \(K=5\) failure suggests that a principled multi-treatment extension should explicitly account for overlap.
Promising directions include:

\paragraph{Overlap-aware stabilization of DR pseudo-outcomes.}
Replace the raw inverse-propensity factors in~\eqref{eq:multi_dr_phi} by stabilized/truncated weights,
e.g., \(\mathbf{1}\{T=k\}/\max\{e_k(X),\eta\}\) with an overlap threshold \(\eta\) chosen by a pre-registered rule,
or use self-normalized / Hajek-style corrections to reduce variance blow-up.

\paragraph{Arm-adaptive denoising and calibration.}
Instead of a shared \(\rho\), choose \(\rho_k\) per arm using calibration-only criteria,
or incorporate \(\widehat V_k(X)\) into the conformal comparison (variance-aware conformalization),
so that candidates are compared within similar variance strata.

\paragraph{Structured multiple testing across arms.}
Rather than pooled BH over \((j,k)\), use hierarchical FDR control:
first select reliable individuals, then select reliable arms within selected individuals, or apply group BH with groups indexed by arms.
This can prevent a single unstable arm from dominating global discoveries.

\paragraph{Sample-size scaling and cross-fitting.}
Multi-arm settings require more data per arm to learn \(\hat\mu_t\), \(\widehat{\operatorname{Var}}(Y\mid X,T=t)\), and \(g_k\) reliably.
Cross-fitting across multiple splits and sharing representations across arms (multi-task \(g\)) may reduce variance without leakage.

Overall, these directions tie multi-treatment validity tightly to overlap management.
In deployment, this is desirable: when overlap deteriorates, the correct response is not only to tune \(\rho\),
but to incorporate explicit overlap-aware mechanisms into the proxy construction and the selection stage.

\end{document}